\documentclass[letterpaper, 10 pt, journal, twoside]{support/IEEEtran}
\usepackage{times}
\usepackage[pdftex]{graphicx}
\usepackage{subfigure}
\usepackage{amsmath,amssymb,amsopn,amstext,amsfonts}
\usepackage{cancel}
\usepackage[space]{cite}
\usepackage{soul}
\usepackage{cuted}
\usepackage{lipsum,multicol}
\usepackage{stfloats}
\everymath{\displaystyle}
\usepackage{booktabs}
\usepackage{amsthm}
\renewcommand{\qedsymbol}{$\blacksquare$}

\usepackage{balance}
\usepackage{color}
\usepackage{mathtools}
\usepackage{algorithm}
\usepackage{algorithmic}
\usepackage{bm}
\newtheorem{theorem}{Theorem}
\newtheorem{lemma}[theorem]{Lemma}

\newtheorem{proposition}{Proposition}
\newtheorem{remark}{Remark}
\theoremstyle{definition}
\newtheorem{definition}{Definition}
\usepackage{diagbox}
\usepackage{float}
\usepackage{epstopdf}
\usepackage{url}
\usepackage{multirow}
\usepackage{tikz}
\usepackage[linkcolor=black,citecolor=black,urlcolor=black,colorlinks=true]{hyperref}
\usepackage{enumitem}

\soulregister\cite7
\soulregister\citep7
\soulregister\citet7
\soulregister\ref7
\soulregister\it7
\soulregister\pageref7

\newcommand{\tabincell}[2]{\begin{tabular}{@{}#1@{}}#2\end{tabular}}

\usepackage{arydshln}

\graphicspath{{figures/}}
\DeclareGraphicsExtensions{.pdf,.png,.jpg,.eps}

\IEEEoverridecommandlockouts

\title{\LARGE \bf Robots State Estimation and Observability Analysis Based on Statistical Motion Models}

\author{Wei Xu$^{1}$, Dongjiao He$^{1}$,Yixi Cai$^{1}$, Fu Zhang$^{1}$
	\thanks{$^{1}$All authors are with Department of Mechanical Engineering, University of Hong Kong. {\tt\small \{xuweii, hdj65822, YixiCai, fuzhang\}@hku.hk}}
}%

\begin{document}	
	\maketitle
	\begin{abstract}
		This paper presents a generic motion model to capture mobile robots' dynamic behaviors (translation and rotation). The model is based on statistical models driven by white random processes and is formulated into a full state estimation algorithm based on the error-state extended Kalman filtering framework (ESEKF). Major benefits of this method are its versatility, being applicable to different robotic systems without accurately modeling the robots' specific dynamics, and ability to estimate the robot's (angular) acceleration, jerk, or higher-order dynamic states with low delay.  Mathematical analysis with numerical simulations are presented to show the properties of the statistical model-based estimation framework and to reveal its connection to existing low-pass filters. Furthermore, a new paradigm is developed for robots observability analysis by developing Lie derivatives and associated partial differentiation directly on manifolds. It is shown that this new paradigm is much simpler and more natural than existing methods based on quaternion parameterizations. It is also scalable to high dimensional systems. A novel \textbf{\textit{thin}} set concept is introduced to characterize the unobservable subset of the system states, providing the theoretical foundation to observability analysis of robotic systems operating on manifolds and in high dimension. Finally, extensive experiments including full state estimation and extrinsic calibration (both POS-IMU and IMU-IMU) on a quadrotor UAV, a handheld platform and a ground vehicle are conducted. Comparisons with existing methods show that the proposed method can effectively estimate all extrinsic parameters, the robot's translation/angular acceleration and other state variables (e.g., position, velocity, attitude) of high accuracy and low delay. 
	\end{abstract}
\section{Introduction}

Online estimation of {\it high-order dynamics states} (e.g., translational and angular acceleration) is fundamentally important for many robots and robotic techniques, such as control~\cite{xu2020learning,smeur2020incremental,smeur2018cascaded,sieberling2010robust}, disturbance estimation~\cite{lyu2018disturbance, yuksel2014nonlinear,jeong2012attitude,wang2016trajectory,lee2016robust}, model identification~\cite{lee2017high,borobia2018flight,hoff1996aircraft}, trajectory generation~\cite{mueller2015computationally}, and extrinsic calibration: in robot's control, it is often necessary to measure or estimate the forces (or accelerations) and torques (or angular accelerations) applied on the robots in order to compute the next control action, estimate the external disturbances, or identify an UAV's aerodynamic parameters; In the calibration of multiple IMUs, angular accelerations are picked up by different accelerometers and need to be accurately estimated for estimating the extrinsic parameters of different IMUs. 

IMUs have been widely used in robotics to measure the robots' acceleration and angular velocity~\cite{shaeffer2013mems}. However, raw IMU measurements typically suffer from various imperfections, such as bias and noise, and cannot be fed to a controller directly. In fact, in robotics community, accelerometers are mainly used to provide attitude (and/or position and velocity) estimation. To name a few, Mahony \emph{et al}.~\cite{mahony2005complementary,mahony2008nonlinear} proposed a nonlinear complementary filter operating on the special orthogonal group ${SO}(3)$ for UAVs' attitude estimation using IMUs. This filter is easy to implement and computationally cheap but is sub-optimal due to the constant feedback gain. Smith and Schmidt~\cite{smith1962application} firstly used the extended Kalman filter (EKF) for spacecraft attitude and velocity estimation based on IMU measurements. Since the standard EKF works in Euclidean space while a robot's attitude is naturally on $SO(3)$, this discrepancy has motivated the error state EKF (ESEKF) \cite{markley2003attitude,trawny2005indirect,markley2014fundamentals}, which workaround the singularity problem caused by minimal parameterization of $SO(3)$ by updating the error attitude~\cite{sola2017quaternion}. Lu \emph{et al}.~\cite{lu2019imu} further considered narrow-band noises in IMU-based attitude estimation.

\begin{figure}[t]
	\centering
	\includegraphics[width=1\columnwidth]{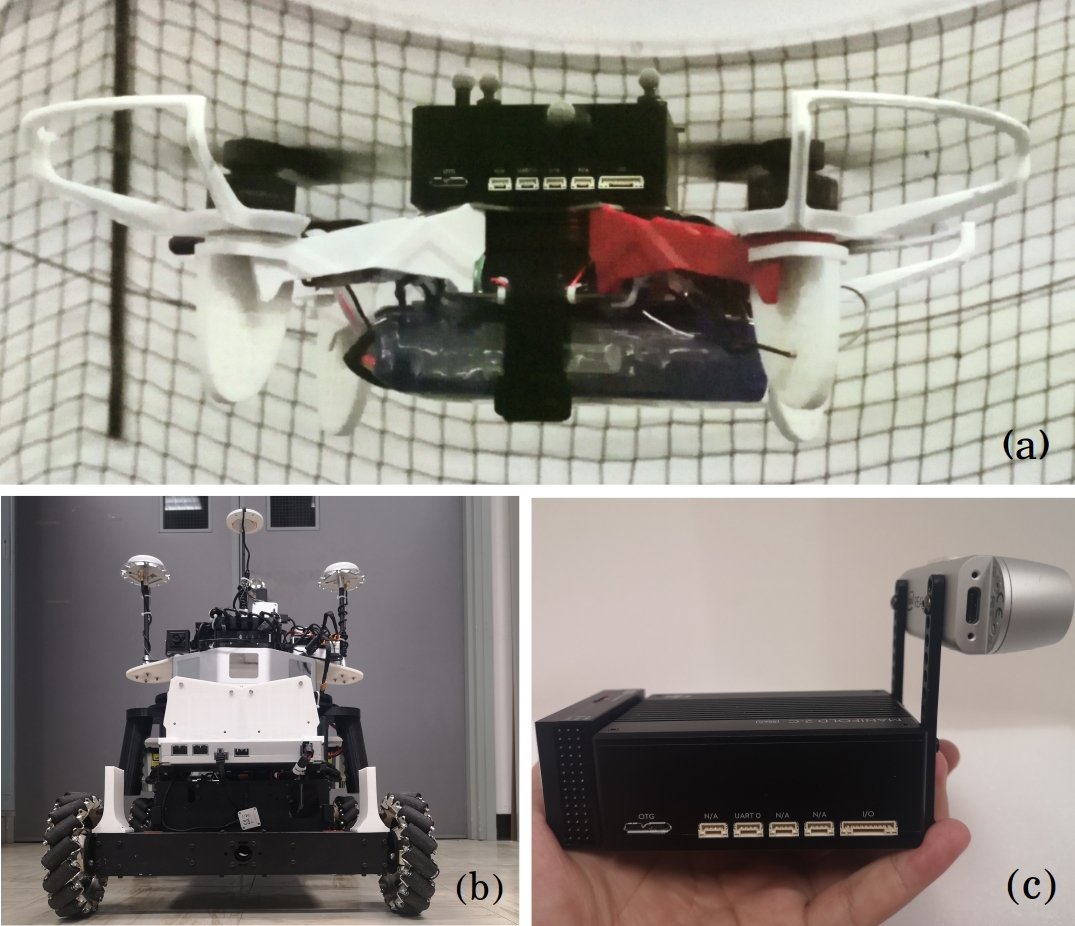}
	\caption{Our proposed method is able to estimate the kinematics state (i.e., position, velocity, and attitude), dynamics state (e.g., (angular) translation, jerk, etc.), and extrinsic parameters simultaneously and online and is applicable to a variety of mobile platforms, including: (a) Unmanned Aerial Vehicles (UAVs); (b)  Unmanned Ground Vehicles (UGVs); (c) Handheld compact IMU-Camera-Computer module.
		\label{fig:uas}}
	\vspace{-0.cm}
\end{figure}

Besides attitude estimation, IMUs are popularly combined with other sensors such as GPS~\cite{caron2006gps}, motion capture systems~\cite{pfister2014comparative}, cameras~\cite{qin2018vins, forster2016manifold, mourikis2007multi}, LiDARs ~\cite{bry2012state, lin2019loam_livox}, optical flow~\cite{horn1981determining}, Pitot tube~\cite{soloviev2009fusion} within an EKF or unscented Kalman filtering (UKF) framework. This closely-coupled estimation method has been proved to be very effective in estimating the robot's {\it kinematic states} including position, velocity, attitude, and IMU bias. The multi-sensor fusion approach has also motivated extrinsic calibration: offline methods include  maximum likelihood (ML)~\cite{burri2018framework}, pose graph optimization~\cite{qin2018vins}, EKF smoother~\cite{hoff1996aircraft}; online methods typically view the extrinsic parameter as a state and use the same EKF/UKF formulation as in state estimation to additionally estimate the extrinsic parameters~\cite{wuest2019online, mirzaei2008kalman, kelly2011visual}.

A major problem of the above works is the absence of {\it high-order dynamics states} (e.g., acceleration) estimation. This paper presents a generic model to capture the dynamic behaviors of a variety of mobile robots. The key idea is to model a robot's translational and angular dynamics by a statistical model with tunable covariance. Based on this generic model, we formulate a new error state EKF framework that can additionally estimate the robot {\it high-order dynamics states} such as acceleration that are not capable by normal EKF~\cite{sola2017quaternion}. Mathematical analysis are conducted to show its stability and reduced phase delay when compared to linear filtering techniques. To investigate the observability of the proposed formulation, we explicitly represent the {\it nonlinear observability rank condition} on the manifold $SO(3)$, leading to simpler, scalable, and more natural observability analysis than existing quaternion-based analysis. The formulated algorithm is finally verified by both simulations and extensive experiments on real robots. Results show that it is able to estimate the kinematic states (i.e., position, attitude, velocity, IMU biases), dynamics states (e.g., acceleration), and extrinsic parameters (e.g., POS-IMU and IMU-IMU) in real time and for a variety of moving platforms, including UAVs, ground vehicles and handheld devices. Moreover, the proposed method delivers estimation performance comparable to normal EKFs~\cite{sola2017quaternion} for kinematics states while in the meantime achieves low latency dynamics states estimation comparable to non-causal zero-phase filters~\cite{gorinevsky2006optimization}.

The rest of the paper is organized as follows: Section II summarizes the related works of states estimation, statistical dynamics model and observability analysis in robotics researches; Section III introduces and analyzes the statistical motion model, based on which the error-state EKF is detailed in Section IV; Observability analysis is presented in Section V; and Section VI and VII presents simulation and experiment results, respectively; Finally, Section VIII draws conclusions.

\section{Related Work}

In existing methods for robotic states estimation, such as GPS/IMU navigation~\cite{sola2017quaternion}, visual IMU navigation and online calibration~\cite{qin2018vins, forster2016manifold, mourikis2007multi, mirzaei2008kalman, bloesch2017iterated, li2013high, kelly2011visual}, and LiDAR/IMU navigation~\cite{bry2012state}, the acceleration and angular velocity are solved from the IMU measurements and are then substituted into kinematics model to form a system ``driven" by the known IMU measurement. This enables to estimate the kinematic states, such as position, velocity, attitude and IMU bias, but the acceleration and angular velocity themselves are viewed as {\it input} to the kinematic model and their values are not filtered. To do so, an ad-hoc linear filter (e.g., Butterworth filter) is usually required to reduce the noises in accelerometer measurements or gyro measurements. This, however, causes considerable output delay due to the Bode's relations~\cite{aastrom2000model}. Another possible approach is to carefully model the robot dynamics and invert it to obtain the forces or torques from the robot outputs, such as the disturbance observers used in~\cite{lyu2018disturbance, yuksel2014nonlinear,jeong2012attitude,wang2016trajectory,lee2016robust,lee2010sensorless}. This dynamic model could also be used with an EKF or UKF to estimate the robot (angular) acceleration such as 
in \cite{wuest2019online,borobia2018flight}. Nevertheless, such a model-based approach needs an accurate model of the robot dynamics which are often subject to various uncertainties, such as motor delay, unmodeled internal dynamics, etc. Instead of using a first principle dynamic model, we propose to use a statistical model to capture the robot dynamics. This statistical model is augmented with the kinematic model where the acceleration appears as a state (instead of input) to the system, enabling itself to be estimated by EKFs. The proposed method leads to lower estimation delay than ad-hod filters and is more generic and robust than model-based methods \cite{wuest2019online,borobia2018flight}. 

Modeling the unknown inputs as a statistical model has been exploited in some prior work and proved effective in practice, such as the Brownian motion for modeling the angular and translational acceleration in both visual-inertial navigation ~\cite{jones2011visual} and vision-only navigation ~\cite{chiuso2002structure}. A even simpler constant velocity model driven by a random noise were also used in ~\cite{davison2003real,guo2014efficient}. Besides the state estimation in robotics,  Hoff {\it et al.}~\cite{hoff1996aircraft} used a quadratic integration stochastic process to model an aircraft's acceleration which enables to identify its aerodynamics. These works usually considered a specific (and usually low order) statistical model and ignored their effect to the system observability. This paper proposes a generic statistical model characterized by its power spectral density and proves that it does not affect the overall observability. Moreover, a new concept, {\it thin} subset, is introduced to characterize the unobservable space of the augmented states. 

Observability plays an important role in robotic states estimation~\cite{lee1982observability}. A well-known criterion for nonlinear system observability is the {\it observability rank condition} based on Lie derivatives~\cite{hermann1977nonlinear}. This criterion has been widely used in robotics to determine the observability of a robotic system, such as range sensor-based 2D localization \cite{lee2006observability, martinelli2011state}, visual-inertial navigation~\cite{weiss2012real, hesch2014camera}, online camera-IMU extrinsic calibration \cite{mirzaei2008kalman, kelly2011visual}, and online parameter identification~\cite{wuest2019online, martinelli2006automatic}, etc. These works usually represent the attitude as a quaternion $\mathbf q$ and view it as a flat vector in $\mathbb{R}^4$ to calculate the Lie derivatives required in the {\it observability rank condition}. To compensate for the over-parameterization when using quaternion, an additional virtual measurement $\mathbf{q}^T \mathbf{q}=1$ is added. Although this method correctly determines the system observability, it undermines the structure of the manifold $SO(3)$. Moreover, the observability analysis is considerably complicated due to the over-parameterization. In this paper, we show that this over-parameterization is totally unnecessary and can be completely avoided by developing Lie derivatives directly on $SO(3)$. Besides preserving the manifold structure, the new implementation leads to a much simpler and more scalable observability analysis. 

\section{Statistical Motion Model}
\subsection{Motion model}
The dynamics states (e.g., translational acceleration, angular acceleration) of a robotic system usually depends on the control actions and/or other forces and torques from the environment. In scenarios such as state estimation of human handheld devices or a UAV intruder, the control actions that drive the system are typically not known to the estimator. Even when such information is available (e.g., estimation of the robot itself states), the dynamic model from the control actions to the state usually suffers from significant uncertainties and/or is affected by environmental disturbances.

We propose to model the dynamics state $\mathbf u$ as a colored stochastic process shown below.
\begin{equation}~\label{e:gamma_model}
\begin{aligned}
\dot {\boldsymbol{\gamma}}  &= \mathbf A_{\bm \gamma}\bm \gamma + \mathbf B_{\bm \gamma} \mathbf w_{\bm \gamma} \\
\mathbf u &=  \mathbf C_{\bm \gamma} \bm \gamma 
\end{aligned}
\end{equation}
where $\left( \mathbf C_{\bm \gamma}, \mathbf A_{\bm \gamma} \right)$ is observable and $\left( \mathbf A_{\bm \gamma}, \mathbf B_{\bm \gamma} \right)$ is controllable, $\mathbf w_{\bm \gamma} \sim \mathcal{N}(0, \mathcal Q_{\bm \gamma})$ are zero mean Gaussian white noises. The resultant output $ \mathbf u $ from (\ref{e:gamma_model}) has a power spectral density (PSD) function $ \mathbf S_{\mathbf u \mathbf u}(\omega) = \mathbf G_{\bm \gamma} (j \omega) \mathcal Q_{\bm \gamma} \mathbf G_{\bm \gamma} (- j \omega)^T$, where $\mathbf G_{\bm \gamma} (j \omega) = \mathbf C_{\bm \gamma} \left( j \omega \mathbf I - \mathbf A_{\bm \gamma} \right)^{-1} \mathbf B_{\bm \gamma} $. 

We posit that the dynamic behaviors of a system (possibly with internal feedback controllers) can be described, statistically, by a certain random process. While it sounds less intuitive as a system is usually designed to accomplish certain tasks and its behaviors are mostly deterministic instead of random, deterministic behavior can always be viewed as a sample of a collection or ensemble of signals (i.e., random process)~\cite{asp}. With this hypothesis, the model parameters (i.e., $\mathbf A_{\bm \gamma}, \mathbf B_{\bm \gamma}, \mathbf C_{\bm \gamma}$ and $\mathcal Q_{\bm \gamma}$) should be chosen such that the resultant PSD $\mathbf S_{ \mathbf u \mathbf u}(\omega)$ agrees with the actual state PSD. Then, the state distribution can be effectively captured by (\ref{e:gamma_model}) while their exact value can be inferred from the measured data.

Since the motion of robotic systems usually possesses certain smoothness (e.g., due to actuator delay), quick changes in accelerations are relatively unlikely and a low-pass type $\mathbf G_{\bm \gamma} (j \omega)$ usually suffices. More specifically, we propose to use an $n$-th order integrator as shown in (\ref{e:n_int_model}). A benefit of the integrator is its sparse structure, which considerably saves the computation.
\begin{equation}~\label{e:n_int_model}
\begin{aligned}
\underbrace{\begin{bmatrix}
\dot{\gamma}_1 \\ \vdots \\ \dot{\gamma}_{N}
\end{bmatrix}}_{\dot{\bm \gamma}} 
&= \underbrace{\begin{bmatrix} 0 & 1 & &  \\ & \ddots& \ddots & \\ & &0&1\\ & & &0 \end{bmatrix}}_{\mathbf A_{\bm \gamma} } 
\underbrace{\begin{bmatrix}
\gamma_1 \\ \vdots \\ \gamma_{N}
\end{bmatrix}}_{\bm \gamma}
 + \underbrace{\begin{bmatrix}
w_{{\gamma}_1} \\ \vdots  \\ w_{{\gamma}_{N}}
\end{bmatrix}}_{\mathbf w_{\gamma}} \\
u &= \underbrace{\begin{bmatrix} 1 & 0 & \hdots & 0 \end{bmatrix}}_{\mathbf C_{\bm \gamma}} \bm \gamma
\end{aligned}
\end{equation}
where $w_{\gamma_i} \sim \mathcal{N}(0, q_{\gamma_i})$. The resultant PSD is:

\begin{equation}
\begin{aligned}\nonumber
S_{uu} \left(\omega\right)= \sum_{i=1}^{N}\frac{q_{\gamma_i}}{\omega^{2i}} 
\end{aligned}
\end{equation}
where $q_{\gamma_i}$ are tunable parameters. Fig. \ref{fig:psdofquad} shows the case when $N=4$ and $q_{\gamma_i} = 1$.
\begin{figure}[t]
	\centering
	\includegraphics[width=1\columnwidth]{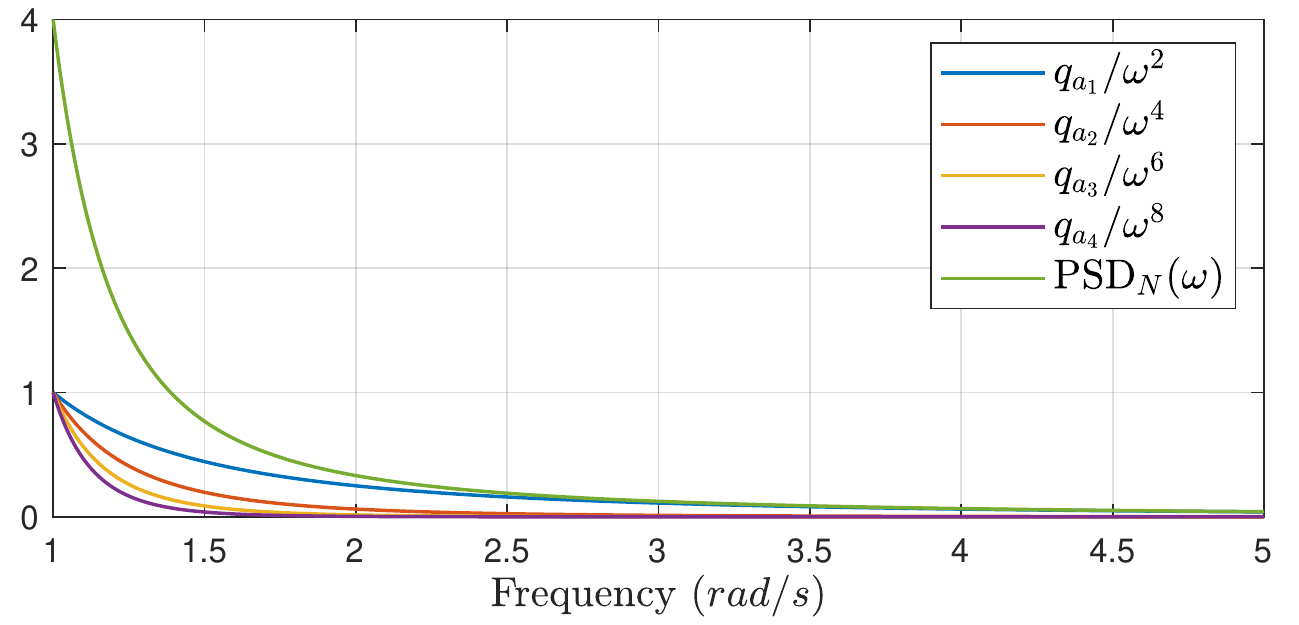}
	\caption{The output PSD of colored stochastic when $N=4$ and $q_{\gamma_i} = 1$.
		\label{fig:psdofquad}}
	\vspace{-0.cm}
\end{figure}

{\it Remark:} The integrator model in (\ref{e:n_int_model}) has been used in some applications, such as the constant velocity model (i.e., $N=1$) used in SLAM~\cite{davison2003real} and the quadratic stochastic process (i.e., $N=4$) used in the aircraft model identification ~\cite{stalford1981high,ramachandran1977identification}.

\subsection{System observability}\label{obsv_gamma}

We analyze the property of the proposed motion model when it is combined with the system kinematic model. To simplify the analysis, we use a linear kinematic model while defer the analysis for real nonlinear robotic systems to Section V. Assume the kinematic model is

\begin{equation}~\label{e:system_equation}
\begin{aligned}
\dot{ \mathbf s} &= \mathbf A_{\mathbf s} \mathbf s + \mathbf B_{\mathbf s} \mathbf u \\
\mathbf y_{\mathbf s} &= \mathbf C_{\mathbf s} \mathbf s + \mathbf n_{\mathbf s} 
\end{aligned}
\end{equation}
where $\mathbf u$ (i.e., acceleration) is the dynamics state from the motion model (\ref{e:gamma_model}). This model is assumed to be a minimal realization and the corresponding transfer function is $\mathbf G_{\mathbf s} (j \omega) = \mathbf C_{\mathbf s} (j \omega \mathbf I - \mathbf A_{\mathbf s})^{-1} \mathbf B_{\mathbf s}$.

Concatenating the motion model (\ref{e:gamma_model}) with  (\ref{e:system_equation}) in series leads to a system in Fig.~\ref{fig:analys_diag}, where the output is the original system output $\mathbf y_s$. The output could also contain dynamics state measurements $\mathbf u_m$. For example, when estimating the state of the robot itself, the acceleration is usually measured by an onboard accelerometer. In this case, $\mathbf u_m = \mathbf u + \mathbf n_u$ where $\mathbf n_u$ is the measurement noise. 

\begin{figure}[h]
	\centering
	\includegraphics[width=0.8\columnwidth]{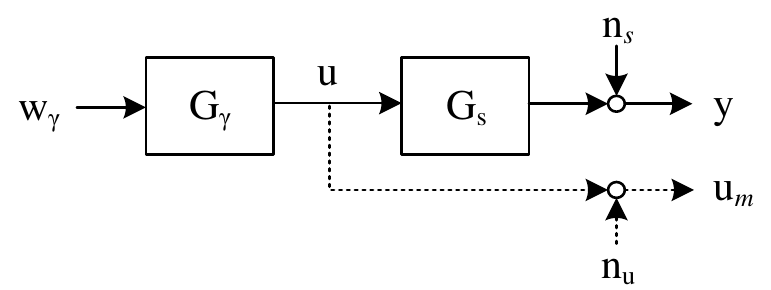}
	\caption{The concatenated system consists of a colored stochastic process $\mathbf G_\gamma$ and kinematic model $\mathbf G_{\mathbf s}$.
		\label{fig:analys_diag}}
	\vspace{-0.cm}
\end{figure}

On the one hand, When there is no measurement $\mathbf u_m$, the concatenated system $ {\mathbf G} = \mathbf G_{\mathbf s} (j \omega) \mathbf G_{\bm \gamma}(j \omega)$ is observable if there is no pole zero cancellation between $\mathbf G_{\mathbf s}(j \omega) $ and $\mathbf G_{\bm \gamma}(j \omega)$. On the other hand, when $\mathbf u_m$ is present, the concatenated system is always observable (see Lemma \ref{obs_lin} below). In either case, once the system is observable, a Kalman filter could be employed to estimate the dynamics state $\mathbf u$ from the system output $\mathbf y_s$ (and state measurement $\mathbf u_m$).

\begin{lemma}~\label{obs_lin}
For any two linear observable systems $(\mathbf C_{\mathbf s}, \mathbf A_{\mathbf s}, \mathbf B_{\mathbf s})$ and $(\mathbf C_{\bm \gamma}, \mathbf A_{\bm \gamma}, \mathbf B_{\bm \gamma})$, their series concatenation
\begin{equation}\nonumber
\begin{aligned}
\underbrace{\begin{bmatrix}
\dot{\mathbf s } \\ \dot{\bm \gamma} 
\end{bmatrix}}_{\dot{\mathbf x }} 
&= \underbrace{\begin{bmatrix} \mathbf A_{\mathbf s} & \mathbf B_{\mathbf s} \mathbf C_{\bm \gamma} \\ \mathbf 0 & \mathbf A_{\bm \gamma}  \end{bmatrix}}_{\mathbf A } 
\underbrace{\begin{bmatrix}
{\mathbf s } \\ {\bm \gamma} 
\end{bmatrix}}_{\mathbf x} + \underbrace{\begin{bmatrix} \mathbf 0  \\ \mathbf B_{\bm \gamma}  \end{bmatrix}}_{\mathbf B_{w} } {\mathbf w_{\bm \gamma}}
  \\
\underbrace{\begin{bmatrix} \mathbf y_s \\ \mathbf u_m \end{bmatrix}}_{\mathbf y} &= \underbrace{\begin{bmatrix} \mathbf C_{\mathbf s} & \mathbf 0 \\
\mathbf 0 & \mathbf C_{\bm \gamma} \end{bmatrix}}_{\mathbf C} \mathbf x 
+  \underbrace{\begin{bmatrix} \mathbf n_y  \\ \mathbf n_u \end{bmatrix}}_{\mathbf n }
\end{aligned}
\end{equation}
is also observable.
\end{lemma}

\begin{proof}
The observability matrix of the concatenated system $(\mathbf C, \mathbf A)$ is 
\begin{equation}~\label{e:obsv}
    \mathcal{O} = \begin{bmatrix} \mathbf C \\ \mathbf C \mathbf A \\ \vdots \\ \mathbf C \mathbf A^{n_{ s} + n_{ \gamma} - 1} \end{bmatrix} = \begin{bmatrix}  \mathcal{O}_{\mathbf s} & \mathcal{O}_{12} \\ \mathbf 0 & \mathcal{O}_{\bm \gamma}  \\ \mathcal{O}_{31} & \mathcal{O}_{32}  \end{bmatrix}
\end{equation}
where $n_s$ and $n_{\gamma}$ are respectively the dimension of the system $(\mathbf C_{\mathbf s}, \mathbf A_{\mathbf s})$ and $(\mathbf C_{\bm \gamma}, \mathbf A_{\bm \gamma})$, $\mathcal{O}_{\mathbf s}$ and $\mathcal{O}_{\bm \gamma}$ are the observability matrix of the two systems. Since $(\mathbf C_{\mathbf s}, \mathbf A_{\mathbf s})$ and $(\mathbf C_{\bm \gamma}, \mathbf A_{\bm \gamma})$ are both observable, $\mathcal{O}_{\mathbf s}$ and $\mathcal{O}_{\bm \gamma}$ both has full column rank, which implies $\mathcal O$ has full column rank and the concatenated system $(\mathbf C, \mathbf A)$ is observable.
\end{proof}

\subsection{Case study}~\label{case_study}

Besides positing the dynamic behavior of a robotic system as a random process, we present another interpretation, from the perspective of signal processing, to our proposed motion model when it is combined with other techniques such as the Kalman filter. As proved in the previous section, the motion model, when concatenated to the kinematic model (\ref{e:system_equation}), remains observable, 
therefore a Kalman filter could be employed to estimate the full state vector $\mathbf x = \begin{bmatrix} \mathbf s^T & \bm \gamma^T \end{bmatrix}^T$ and hence $\mathbf u = \mathbf C_{\bm \gamma} \bm \gamma$ from the measurements $\mathbf u_m$ and $\mathbf y$.  Alternatively, the ground truth state $\mathbf u$ can be estimated by directly filtering $\mathbf u_m$. In this section, we present a simple case to illustrate the difference and connections between these two filters. In this simple case study, we assume that the dynamics state $u$ is the one-dimension linear acceleration $a$, the system kinematic model (\ref{e:system_equation}) is a velocity model, and the measurements are the velocity $ v_m$ and acceleration $ a_m$. We choose $N=2$ in the colored stochastic. The concatenated system is hence:

\begin{equation}~\label{e:ABmode2}
\begin{aligned}
\underbrace{\begin{bmatrix}
\dot{{v}} \\ \dot{ a} \\ \dot{ a}_1 
\end{bmatrix}}_{\dot{\mathbf x }} 
&= \underbrace{\begin{bmatrix} 0 & 1 & 0 \\ 0 & 0 & 1 \\ 0 & 0 &0 \end{bmatrix}}_{\mathbf A } 
\underbrace{\begin{bmatrix}
 v \\  a \\  a_1
\end{bmatrix}}_{\mathbf x}
 + \underbrace{\begin{bmatrix} 0 & 0 \\ 1 & 0 \\ 0 & 1 \end{bmatrix}}_{\mathbf B_w } 
\underbrace{\begin{bmatrix}
 w_{a_1}  \\  w_{a_2}
\end{bmatrix}}_{\mathbf w} \\
\underbrace{\begin{bmatrix} v_m \\  a_m \end{bmatrix}}_{\mathbf y} &= \underbrace{\begin{bmatrix} 1 & 0 & 0 \\
0 & 1 & 0 \end{bmatrix}}_{\mathbf C} \mathbf x + \underbrace{\begin{bmatrix}  n_v  \\  n_a \end{bmatrix}}_{\mathbf n}
\end{aligned}
\end{equation}
where $ w_{a_1}\sim\mathcal N\left( 0, q_{a_1}\right),  w_{a_2} \sim\mathcal N\left(0, q_{a_2}\right)$ are the process noises and $ n_{v}\sim\mathcal N\left(0, r_v\right),\  n_{a}\sim \mathcal N\left(0, r_a\right)$ are the measurement noises.

As shown in the previous section, the concatenated system (\ref{e:ABmode2}) is observable, hence the Kalman filter will converge to a stationary gain $\mathbf L_{\infty}$. At convergence, the transfer function from the combined output $\mathbf y$ to the acceleration estimation $\widehat{a}$ is
\begin{equation}~\label{e:transf_cal_sskf}
\begin{aligned}
\widehat{a}\left(s\right) &= \mathbf e_2^T \left(s\mathbf I -\mathbf A + \mathbf L_\infty \mathbf C \right) ^{-1} \mathbf L_\infty\; \mathbf y\left(s\right)\\
&= G_v(s) v_m(s) + G_a(s) a_m(s) 
\end{aligned}
\end{equation}
where $\mathbf e_2 = \begin{bmatrix}0 & 1 & 0 \end{bmatrix}^T$ and $G_v(s), G_a(s)$ are the two components of $\mathbf e_2^T \left(s\mathbf I -\mathbf A + \mathbf L_\infty \mathbf C \right) ^{-1} \mathbf L_\infty$.

Substituting the measurement model from (\ref{e:ABmode2}) into (\ref{e:transf_cal_sskf}) and recalling $v(s) = \frac{1}{s} a(s)$, we obtain
\begin{equation}\nonumber
\begin{aligned}
\widehat{a}\left(s\right)= \left[\frac{G_{v}(s)}{s}+G_{a}(s)\right]a(s)+G_{v}(s)n_{v}(s)+G_{a}(s)n_{a}(s)
\end{aligned}
\end{equation}

Since the measurement noise $n_{v}$ and $n_{a}$ are independent, the standard deviation of $\widehat{a}$ caused by $n_v$ and $n_a$ is hence
\begin{equation}\nonumber
\begin{aligned}
\sigma_{\widehat{a}}  &= \sqrt{r_v \| G_{v}( j \omega) \|_2^2 + r_a \| G_{a}( j \omega)\|_2^2}
\end{aligned}
\end{equation}
And the transfer function from actual acceleration $a(s)$ to the estimated one $\widehat{a}(s)$ is:
\begin{equation}\nonumber
G_{a\widehat{a}}(s)=\frac{G_{v}(s)}{s}+G_{a}(s)
\end{equation}
Its bode diagram is shown in Fig.~\ref{fig:bode_quadra}. It can be seen that $G_{a\widehat{a}}$ behaves like a first-order low-pass filter in high frequencies. 

\begin{figure}[t]
	\centering
	\includegraphics[width=1\columnwidth]{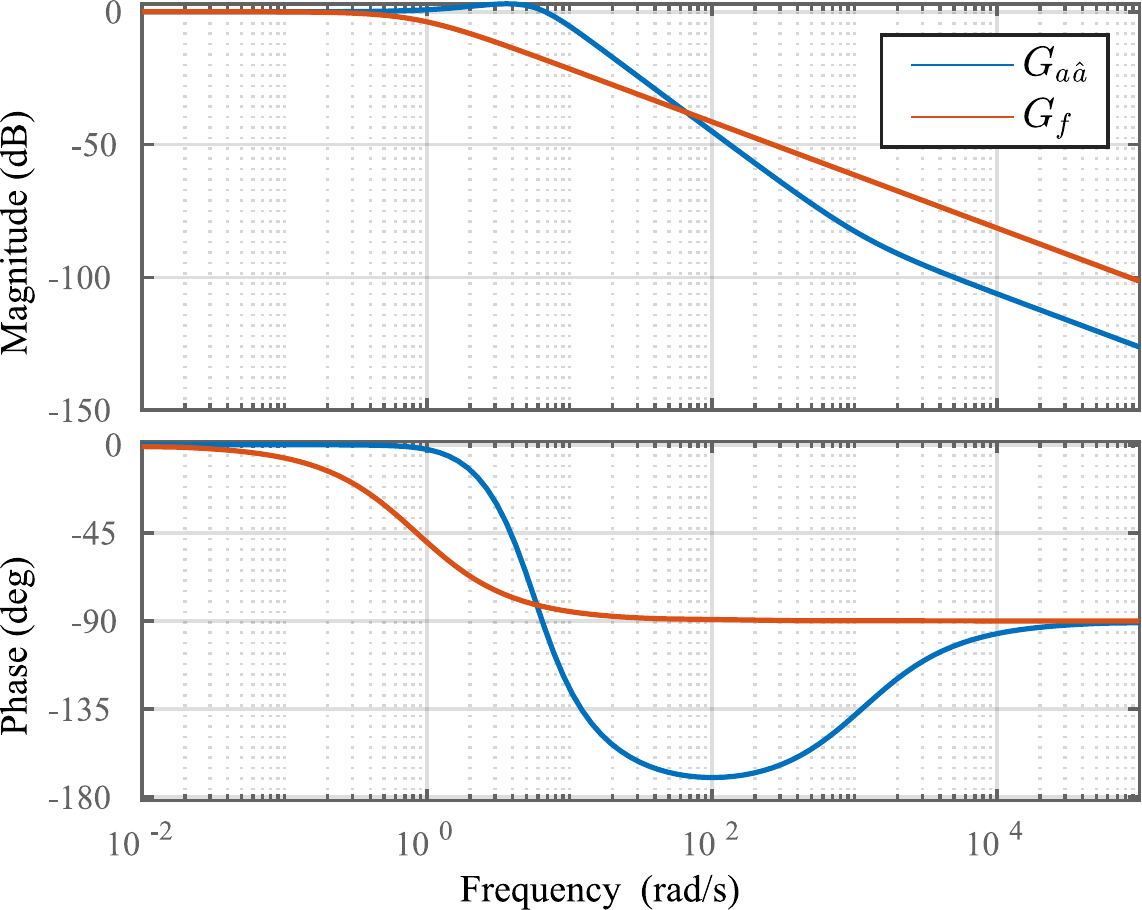}
	\caption{The bode diagram of $G_f(s)$ and $G_{a\widehat{a}}(s)$ when $\begin{bmatrix}
	r_v & r_a
	\end{bmatrix}=\begin{bmatrix}
	0.1& 1
	\end{bmatrix}$ and $\begin{bmatrix}
	q_{a_1}& q_{a_2}
	\end{bmatrix}=\begin{bmatrix}
	0.1& 1
	\end{bmatrix}$. 
		\label{fig:bode_quadra}}
	\vspace{-0.cm}
\end{figure}

Moreover, if directly filtering the acceleration measurement $a_m$ by a first order Butterworth low-pass filter:
\begin{equation}\nonumber
\begin{aligned}
G_f(s)=\frac{1}{1+ks}
\end{aligned}
\end{equation}
where the parameter $k$ is determined by ensuring the same level of noise attenuation (i.e., $\| G_f(j \omega) \|_2 = \sqrt{r_v \| G_{v}( j \omega) \|_2^2 + r_a \| G_{a}( j \omega)\|_2^2}$). The transfer function $G_f(s)$ is shown in Fig.~\ref{fig:bode_quadra}. It can be seen that the Kalman filter (i.e., $G_{a\widehat{a}}(s)$) has much higher bandwidth than the direct filtering method (i.e., $G_f(s)$), allowing more signal to pass through with low delay. This is no surprising as the Kalman filter infers this acceleration from both the acceleration and velocity measurements instead of one only. Finally, it is noticed that the Kalman filter has a slight amplification in the middle frequency range, which could lead to a small overshoot in the acceleration estimate.

\section{State Estimation}\label{sim}
In this section, we leverage the proposed motion model (\ref{e:gamma_model}) to state estimations of robotic systems. We take an UAV shown in Fig.~\ref{fig:example_sim} as an illustrative example. This system is configured with an IMU sensor containing a gyroscope and an accelerometer and a position sensor (Pos in Fig.\ref{fig:example_sim}) with orientation measurements, such as GPS, motion capture feedback, or visual-inertial odometry (VIO). The objective is to estimate the kinematics state (i.e., position, velocity, attitude, and IMU biases), dynamics states (i.e., angular velocity and linear acceleration), and the extrinsic parameters (i.e., the constant offset $\mathbf c$ between the position sensor and IMU). 

\begin{figure}[b]
	\centering
	\includegraphics[width=1\columnwidth]{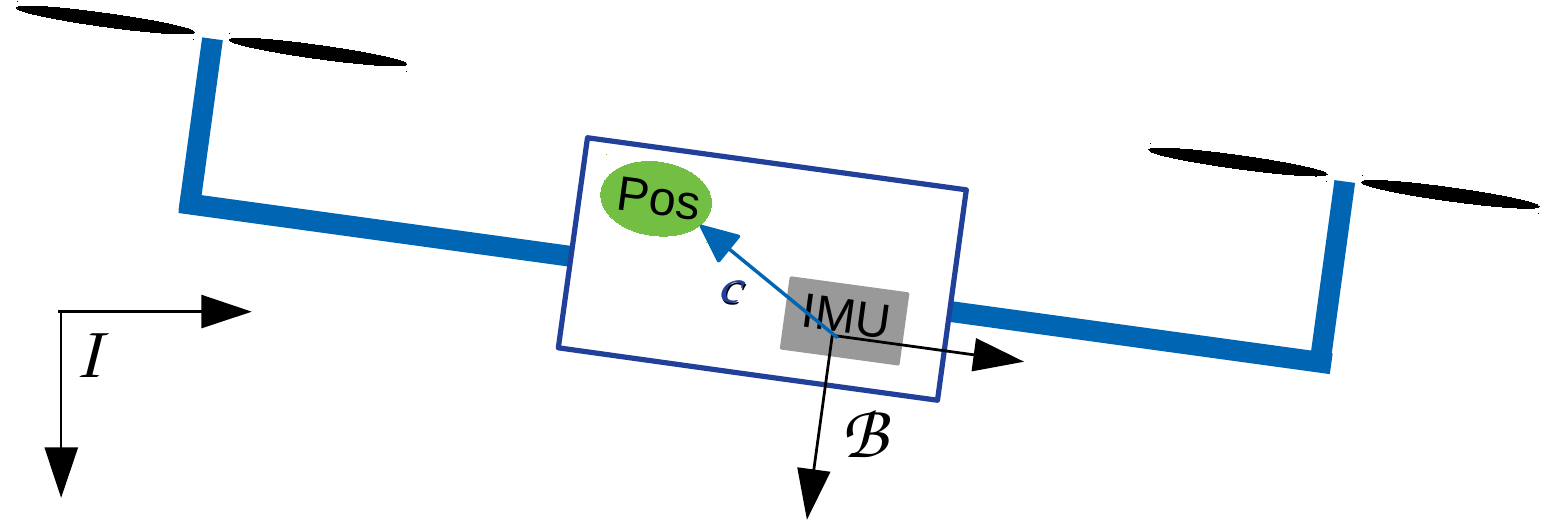}
	\caption{A typical UAV configured with a position sensor and IMU.
		\label{fig:example_sim}}
	\vspace{-0cm}
\end{figure}

\subsection {State Space Model}~\label{sim_model}
To model the UAV system, we choose the IMU as the reference point and denote its position, specific acceleration, and body angular velocity as $\mathbf p$, $\mathbf a$, and $\bm \omega$, respectively. It should be noted that the specific acceleration is the total acceleration removing the gravity and is directly measured by the accelerometer \cite{martin2010true}. The rotation from inertial frame to IMU's body frame is denoted as a rotation matrix $\mathbf R\in {SO}(3)$. The measurement of position sensor, accelerometer, and gyroscope are $\mathbf p_m, \mathbf a_m$, and $\bm \omega_m$, respectively. The orientation measurement comes from magnetometer or other sensors like vision sensors and is denoted as $\mathbf m_m$. The biases of IMU accelerometer and gyroscope are denoted as $\mathbf b_{\mathbf a}$ and $\mathbf b_{\bm \omega}$, which are typically modeled as random walks~\cite{lu2019imu}. With these notations, the resultant kinematic model is: 
\begin{equation}~\label{e:kine_model_s}
\begin{aligned}
\dot{\mathbf p}&=\mathbf v,\quad \ \dot{\mathbf v}=\mathbf R \mathbf a + \mathbf g,\quad \dot{\mathbf b}_{\mathbf a}=\mathbf w_{\mathbf b_{\mathbf a}}, \\
\dot {\mathbf c}&=\mathbf 0,\quad \dot{\mathbf R}=\mathbf R \lfloor \bm \omega\rfloor,\quad \quad \dot{\mathbf b}_{\bm \omega}=\mathbf w_{\mathbf b_{\bm \omega}}
\end{aligned}
\end{equation}
where the notation $\lfloor \bm \cdot \rfloor$ denotes the skew-symmetric matrix of a vector $\bm \cdot$ in $\mathbb{R}^3$. The system outputs are
\begin{subequations}~\label{e:measure}
\begin{alignat}{8}~\label{e:measure_aw}
{\mathbf a}_{m}&=\mathbf a + \mathbf b_{\mathbf a}+\mathbf n_{\mathbf a},\quad\ {\bm \omega}_{m}=\bm \omega + \mathbf b_{\bm \omega}+\mathbf n_{\bm \omega}\\
{\mathbf p}_{m}&=\mathbf p + \mathbf R\mathbf c+\mathbf n_{\mathbf p},\quad {\mathbf m}_{m}=\mathbf R^T \mathbf e_1+\mathbf n_m \label{e:measure_pm}
\end{alignat}
\end{subequations}

\subsection {Input formulation}~\label{input_formula}
The acceleration $\mathbf a$ and angular velocity $\boldsymbol{\omega}$ are unknown in (\ref{e:kine_model_s}). To form a valid state space model with known inputs, a common practice (e.g., \cite{sola2017quaternion, qin2018vins,forster2016manifold, mourikis2007multi, mirzaei2008kalman, kelly2011visual, bry2012state}) is to solve them from the measurements equation (\ref{e:measure_aw}), forming a system ``driven" by the known measurement $\mathbf a_m$ and $\boldsymbol{\omega}_m$. 
\begin{equation}~\label{e:kine_model_input}
\begin{aligned}
\dot{\mathbf p}&=\mathbf v,\quad \ \dot{\mathbf v}=\mathbf R \mathbf (\mathbf a_m - \mathbf b_{\mathbf a} - \mathbf n_{\mathbf a}) + \mathbf g,\quad \dot{\mathbf b}_{\mathbf a}=\mathbf w_{\mathbf b_{\mathbf a}}, \\
\dot {\mathbf c}&=\mathbf 0,\quad\; \dot{\mathbf R}=\mathbf R \lfloor \bm \omega_m - \mathbf b_{\bm \omega} - \mathbf n_{\bm \omega}\rfloor, \quad \quad\ \dot{\mathbf b}_{\bm \omega}=\mathbf w_{{\mathbf b}_{\bm \omega}}
\end{aligned}
\end{equation}

We term this formulation as ``input formulation" as the acceleration and angular velocity measurements appear as inputs to the formulated system. The system output vector is then $\mathbf z_m^T = \begin{bmatrix} {\mathbf p}_{m}^T & {\mathbf m}_{m}^T \end{bmatrix}$ in (\ref{e:measure_pm}). It should be noticed that a state estimator (e.g., EKF) based on the input formulation (\ref{e:kine_model_input}) does not filter the input signals $\mathbf a_m$ and $\boldsymbol{\omega}_m$.

\subsection {State formulation}~\label{state_formula}
With the proposed statistical model (\ref{e:gamma_model}), the unknown inputs in (\ref{e:kine_model_s}), $\mathbf a$ and $\boldsymbol{\omega}$, are now the outputs of (\ref{e:gamma_model}), leading to a concatenated system shown below:
\begin{equation}
\begin{split}~\label{e:kine_model_state}
    \dot{\bm \gamma}_{\mathbf a}&=\mathbf A_{\bm \gamma_{\mathbf a}}{\bm \gamma}_{\mathbf a}\!+\!\mathbf B_{\bm \gamma_{\mathbf a}}\mathbf w_{\bm \gamma_{\mathbf a}}, \
    \dot{\bm \gamma}_{\bm \omega}=\mathbf A_{\bm \gamma_{\bm \omega}}{\bm \gamma}_{\bm \omega}\!+\!\mathbf B_{\bm \gamma_{\bm \omega}}\mathbf w_{\bm \gamma_{\bm \omega}} \\
    \dot{\mathbf p} &=\mathbf v,\quad \ \dot{\mathbf v}=\mathbf R \left( \mathbf C_{\bm \gamma_{\mathbf a}}{\bm \gamma}_{\mathbf a} \right) + \mathbf g,\quad \dot{\mathbf b}_{\mathbf a}=\mathbf w_{\mathbf b_{\mathbf a}},  \\
\dot {\mathbf c} &=\mathbf 0,\quad\; \dot{\mathbf R}=\mathbf R \lfloor \mathbf C_{\bm \gamma_{\bm \omega}}{\bm \gamma}_{\boldsymbol{\omega}} \rfloor, \quad \quad \ \dot{\mathbf b}_{\bm \omega}=\mathbf w_{{\mathbf b}_{\bm \omega}}
\end{split}
\end{equation}

The system output vector is $\mathbf z_m^T = \begin{bmatrix}
{\mathbf p}_{m}^T& {\mathbf m}_{m}^T& \mathbf a_m^T& \boldsymbol{\omega}_m^T
\end{bmatrix}$ as in (\ref{e:measure}). Unlike the input formulation (\ref{e:kine_model_input}), the dynamics states $\mathbf a=\mathbf C_{\bm \gamma_{\mathbf a}}\bm \gamma_{\mathbf a}$ and $\bm \omega= \mathbf C_{\bm \gamma_{\bm \omega}}\bm \gamma_{\bm \omega}$ now appear as functions of system states, thus could be filtered by a state estimator. We term this new formulation as ``state formulation". 

\subsection {Estimation}~\label{sec:estimation}
To estimate the dynamics states $\mathbf a$ and $\bm \omega$, it is sufficient to estimate the state vector of the state formulation (\ref{e:kine_model_state}). This is a standard state estimation problem and various nonlinear state estimator could be used such as EKF and UKF. In this paper, we employ an error-state EKF as detailed in \cite{sola2017quaternion}. The filter error-state vector is defined as
\begin{equation}~\label{e:error_state}
    \widetilde{\mathbf x} = \begin{bmatrix}
    \widetilde{\mathbf p}^T & \widetilde{\mathbf v}^T & \delta \bm \theta^T & \widetilde{\mathbf c}^T & \widetilde{\mathbf b}_a^T &\widetilde{\mathbf b}_\omega^T  &\widetilde{\bm \gamma}_{\mathbf a}^T &\widetilde{\bm \gamma}_{\boldsymbol{\omega}}^T
    \end{bmatrix}^T
\end{equation}
where  $\delta \bm \theta=\text{Log}(\widehat{\mathbf R}^T\mathbf R )$ is the attitude error and the rests are standard additive errors (i.e., the error in the estimate $\widehat{\mathbf x} $
of a quantity $\mathbf x$ is $\widetilde{\mathbf x} = \mathbf x - \widehat{\mathbf x}$). Intuitively, the attitude error $\delta \bm \theta $ describes the (small) deviation between the true and the estimated attitude. The main advantage of this error definition is that it allows us to represent the attitude uncertainty by the $3\times3$ covariance matrix $ \mathbb{E} \left\{\delta\bm\theta\delta\bm\theta^T\right\}$. Since the attitude has 3 degree of freedom (DOF), this is a minimal representation.

\subsubsection{Filter Propagation}\
\par

The propagation step breaks into state propagation and covariance propagation. The former one is achieved by applying expectation operator on both side of (\ref{e:kine_model_state}). The propagation on covariance is applied on the error state vector (\ref{e:error_state}) by linearizing the state equation (\ref{e:kine_model_state}):

\begin{equation}\nonumber
\begin{aligned}
\dot{\widetilde{\mathbf x}} &= \mathbf F_{{\mathbf x}}\widetilde{\mathbf x} + \mathbf F_{\mathbf w} \mathbf w
\end{aligned}
\end{equation}
where $\mathbf F_{{\mathbf x}}$, $\mathbf F_{\mathbf w}$ are the Jacobin matrices w.r.t. $\widetilde{\mathbf x}$, and $\mathbf w \sim \mathcal{N}(\mathbf 0, \mathcal{Q})$ is the process noise.

The covariance matrix is then propagated as 
\begin{equation}\nonumber
    \mathbf P_{k+1|k} = \mathbf \Phi_{\mathbf x_k} \mathbf P_{k|k} \mathbf \Phi_{\mathbf x_k}^T + \mathbf \Phi_{\mathbf w_k} \mathcal Q \mathbf \Phi_{\mathbf w_k}^T
\end{equation}
where $\mathbf \Phi_{\mathbf x_k} = \mathbf F_{{\mathbf x}_k}\cdot\Delta t + \mathbf I$ and $\mathbf \Phi_{\mathbf w_k} = \mathbf F_{\mathbf w_k} \cdot\Delta t$.

\subsubsection{Filter Update}\
\par

The measurement residual is
\begin{equation}\nonumber
    \mathbf r = \mathbf z_m - \hat{\mathbf z} \simeq \mathbf H \widetilde{\mathbf x} + \mathbf n
\end{equation}
where $\mathbf H$ is the Jacobin matrices of (\ref{e:measure}) w.r.t. $\widetilde{\mathbf x}$, and $\mathbf n \sim \mathcal{N}(\mathbf 0, \mathcal{R})$ is the measurement noise vector in (\ref{e:measure}). Every time there comes a new measurement, the state variables are updated through the following steps:

a). Compute the Kalman gain:
\begin{equation}\nonumber
\begin{aligned}
\mathbf K_{k+1} &= \mathbf P_{k+1|k}\mathbf H_{k+1}^T \left( \mathbf H_{k+1}\mathbf P_{k+1|k}\mathbf H_{k+1}^T + \mathcal R \right)^{-1}
\end{aligned}
\end{equation}

b). Update the Error states:
\begin{equation}\nonumber
\widetilde{\mathbf x}_{k+1|k+1} = \mathbf K_{k+1}\left(\mathbf z_{m_{k+1}} - \hat{\mathbf z}_{k+1}  \right)
\end{equation}

c). Update the original state variables based on the definition of the error state in (\ref{e:error_state}), and compute the covariance as:
\begin{equation}\nonumber
\mathbf P_{k+1|k+1} = \mathbf P_{k+1|k} - \mathbf K_{k+1}\mathbf S_{k+1}\mathbf K_{k+1}^T
\end{equation}

\section{Observability Analysis}

For a filter (e.g., EKF) to converge, the system must be observable~\cite{lee1982observability}. For linear time invariant (LTI) systems, the observability is easily determined by the well-known linear observability matrix rank test~\cite{hespanha2018linear}. For nonlinear systems, the observability is much more complicated~\cite{hermann1977nonlinear}. Among the various variants of observability outlined in~\cite{hermann1977nonlinear}, the {\it locally weakly observability} has been widely used in robotics for state estimation~\cite{hesch2014camera, weiss2011real,martinelli2011state, weiss2012real}, online extrinsic calibration \cite{mirzaei2008kalman, kelly2011visual}, and parameter identification~\cite{wuest2019online, martinelli2006automatic}, etc.

A system is said to be {\it locally weakly observable} if its state can be distinguished {\it instantaneously} (i.e., requiring arbitrarily small time duration) from any other states in a certain neighbor. Intuitively, this means a state estimator (e.g., EKF) will converge to an unbiased estimate of the state if the model truthfully represents the reality and the system inputs have sufficient excitation. An advantage of the locally weakly observability is that it lends itself to a simple algebraic test known as the {\it  observability rank condition} ~\cite{hermann1977nonlinear}.

\subsection{Observability Rank Condition}

Consider an input-linear system~\cite{nijmeijer1990nonlinear} on smooth manifold  $\mathcal{M}$:
\begin{equation}~\label{e:input_linear_sys}
\begin{aligned}
\dot{\mathbf{x}} &= \mathbf{f}_0\left(\mathbf x\right)+\sum_{i=1}^{m}\mathbf{f}_i\left(\mathbf x\right)u_i\\
\mathbf{y} &= \mathbf{h}\left(\mathbf x\right)
\end{aligned}
\end{equation}
where $\mathbf x \in \mathcal{M}$, $\mathbf{f}_0(\mathbf{x})$ is the drift (i.e., zero-input) vector field, and $\mathbf{f}_i(\mathbf{x})$, $i = 1,\cdots,m$, defines a vector field on $\mathcal{M}$ that is linear to each control input. Then the observability rank condition implies that if the observability matrix $\mathcal{O}(\bar{\mathbf x})$ defined below has full column rank at $\bar{\mathbf x}$, system (\ref{e:input_linear_sys}) is locally weakly observable at $\bar{\mathbf x}$~\cite{hermann1977nonlinear}. Furthermore, if $\mathcal{O}(\bar{\mathbf x})$ has full column rank for all $\bar{\mathbf x} \in \mathcal{M}$, then the system is locally weakly observable \cite{isidori2013nonlinear}. 

\begin{equation}~\label{e:obsv_lie}
    \mathcal{O}(\bar{\mathbf x}) = \begin{bmatrix}  \vdots \\ \nabla_{\mathbf x} \left( \left(\mathcal L_{\mathbf f^k \dots \mathbf f^1}^k\mathbf h \right) (\mathbf x) \right) (\bar{\mathbf x})  \\\vdots \end{bmatrix}
\end{equation}
where $k = 0, 1, \dots, $ and $\mathbf f^k$ is any of $\{\mathbf f_0, \mathbf f_1, \dots, \mathbf f_m\}$. When $k=1$, $\left(\mathcal L_{\mathbf f}^1\mathbf h \right) ({\mathbf x})$ is the Lie derivative of function $\mathbf h$ along the vector field $\mathbf f $ at point $\mathbf x \in \mathcal{M}$ and is defined below \cite{AKurusa2014}
\begin{equation}~\label{e:obsv_lie_0_order}
\begin{aligned}
    & \left( \mathcal L_{\mathbf f}^1\mathbf h \right) ({\mathbf x}) = \left. \frac{d {\mathbf{h}}(\mathbf{r} (t))}{dt} \right|_{t=0}, \\
    & \quad \quad \quad \quad \text{subject to} \ \frac{ d \mathbf{r}(t)}{dt}= \mathbf f(\mathbf{r}(t)), \mathbf{r}(0) = {\mathbf x} 
\end{aligned}
\end{equation}

Since the time derivative of $\mathbf h (\mathbf x (t))$ is evaluated at time zero, it is sufficient to only consider $\mathbf r(t)$ at $t = 0$, i.e., $\dot{\mathbf{r}}(0)= \mathbf f(\mathbf{r}(0)), \mathbf{r}(0) = {\mathbf x}$. Then (\ref{e:obsv_lie_0_order}) is equivalent to the following form for the sake of notation simplicity
\begin{equation}~\label{e:1st_lie_def}
\begin{aligned}
    & \left( \mathcal L_{\mathbf f}^1\mathbf h \right) ({\mathbf x}) =  \dot{\mathbf{h}}(\mathbf{x}), \text{subject to} \ \dot{\mathbf{x}}= \mathbf f(\mathbf{x}) 
\end{aligned}
\end{equation}

In particular, when $\mathcal{M} = \mathbb{R}^n$, the Lie derivative in (\ref{e:1st_lie_def}) can be simplified by the chain rule and becomes the commonly seen form.

\begin{equation}~\label{e:1st_lie_def_Rn}
\begin{array}{c}
    \left( \mathcal L_{\mathbf f}^1\mathbf h \right) ({\mathbf x}) = \frac{\partial \mathbf{h} ({\mathbf{x}}) }{\partial \mathbf x } \mathbf f ({\mathbf{x}})
\end{array}
\end{equation}

Based on the first order Lie derivative defined in (\ref{e:1st_lie_def}), the higher order Lie derivatives $\left( \mathcal L_{\mathbf f^k \dots \mathbf f^1}^k\mathbf h \right) (\mathbf x)$ in (\ref{e:obsv_lie}) is recursively defined as follows. 
\begin{subequations}~\label{e:lie_def}
\begin{alignat}{2}
    \left( \mathcal L^0 \mathbf h \right) ({\mathbf x})  &= \mathbf h ({\mathbf x}) \label{e:lie_def_0} \\
    \left( \mathcal L_{\mathbf f^k \dots \mathbf f^1}^k\mathbf h \right) ({\mathbf x})  &= \left( \mathcal L_{\mathbf f^k}\left(\mathcal L^{k-1}_{\mathbf f^{k-1} \dots \mathbf f^1}\mathbf h\right) \right) ({\mathbf x}) \label{e:lie_def_k} 
\end{alignat}
\end{subequations}

As can be seen, $\left( \mathcal L_{\mathbf f^k \dots \mathbf f^1}^k\mathbf h \right) (\mathbf x) $ is a function of $\mathbf x \in \mathcal{M} $. To simplify the notation, we denote it as $\mathbf g (\mathbf x)$. Then, the notation $\nabla_{\mathbf{x}} \left( \mathbf g (\mathbf{x}) \right) (\bar{\mathbf{x}}) $ in (\ref{e:obsv_lie}) denotes the gradient of $\mathbf{g}(\mathbf{x}): \mathcal{M} \mapsto \mathbb{R}^m$ w.r.t. $\mathbf{x}$ evaluated at $\bar{\mathbf{x}}$. If $\mathcal{M} = \mathbb{R}^n$, this becomes the well known partial derivation.

\begin{equation}~\label{e:partial_derivation_Rn}
    \nabla_{\mathbf{x}} \left( \mathbf g (\mathbf{x}) \right) (\bar{\mathbf{x}}) = \left. \dfrac{\partial \mathbf g ({\mathbf{x}}) }{\partial \mathbf x } \right|_{\mathbf x = \bar{\mathbf{x}}}\in \mathbb{R}^{m \times n}
\end{equation}

For a generic smooth manifold $\mathcal{M}$ of dimension $n$, a local coordinate chart around $\bar{\mathbf x}$ on $\mathcal{M}$ needs to be found (see Fig.~\ref{fig:coordinate_chart}). Assume the local coordinate chart is $(U, \boldsymbol{\varphi})$, where $U$ is an open subset of $\mathcal{M}$ containing the point $\bar{\mathbf x}$ and $\boldsymbol{\varphi}: U \mapsto \widehat{U}$ is a homeomorphism from $U$ to an open subset $\widehat{U} = \boldsymbol{\varphi}(U) \subset \mathbb{R}^n$. Then the gradient of $\mathbf g (\mathbf{x}) $ is computed as~\cite{hitchin2012differentiable}:
\begin{figure}[h]
	\centering
	\includegraphics[width=0.8\columnwidth]{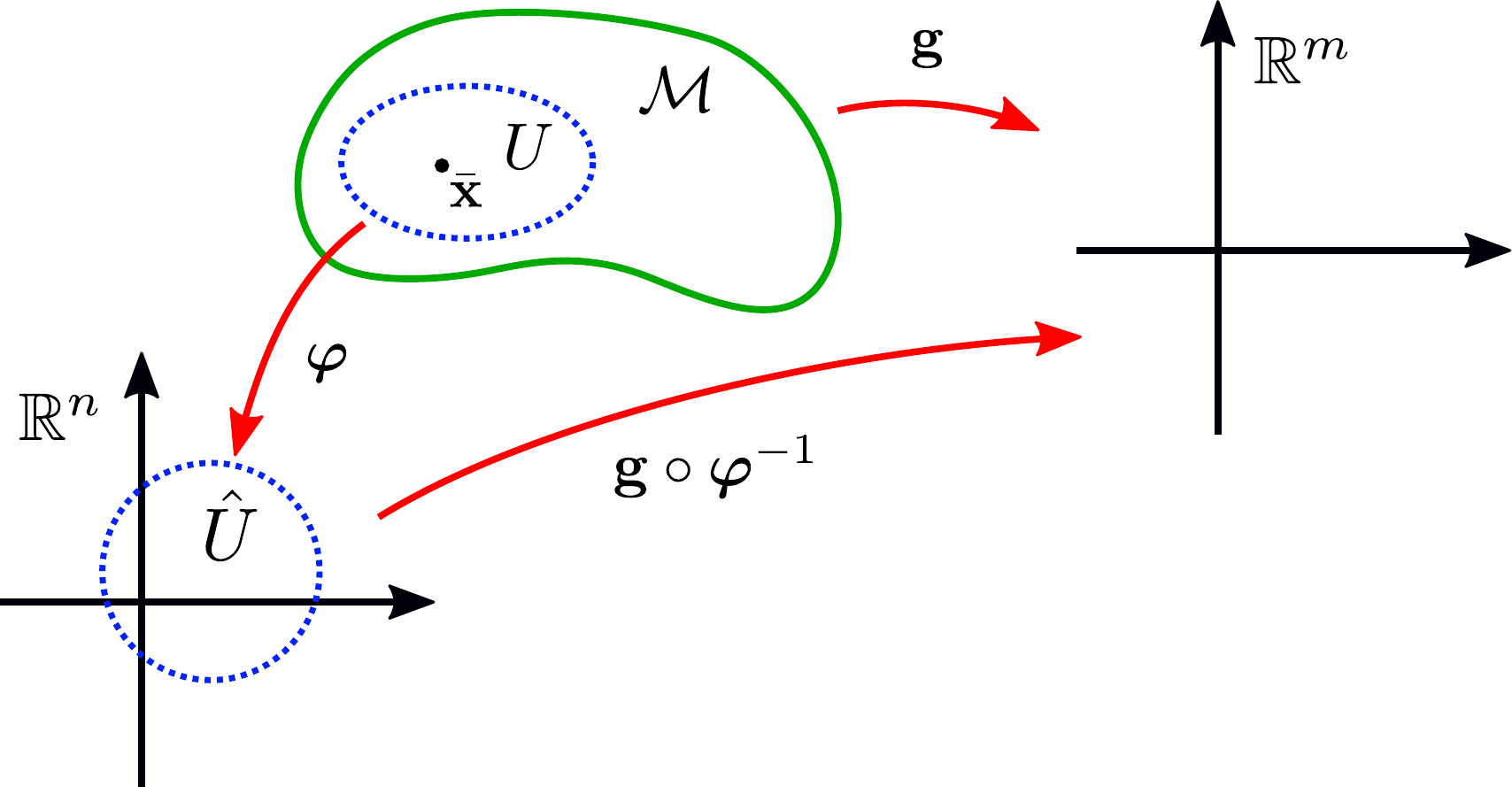}
	\caption{Illustration of a coordinate chart $(U, \boldsymbol{\varphi})$ on manifold $\mathcal M$.
		\label{fig:coordinate_chart}}
	\vspace{-0cm}
\end{figure}
\begin{equation}~\label{e:partial_der_manifold}
    \nabla_{\mathbf{x}} \left( \mathbf g (\mathbf{x}) \right) (\bar{\mathbf{x}}) = \left. \dfrac{\partial \mathbf g \left( \boldsymbol{\varphi}^{-1} \left({\boldsymbol{\theta}} \right) \right) }{\partial \boldsymbol{\theta} } \right|_{\boldsymbol{\theta} = \boldsymbol{\varphi}(\bar{\mathbf x})}
\end{equation}

In particular, if $\mathcal{M} = SO(3)$, one such local coordinate chart is $(U, \boldsymbol{\varphi})$, where $U = \left\{ \mathbf x \in SO(3) | \| \text{Log} (\bar{\mathbf{x}}^{-1} \mathbf x )\|_2 < \pi \right\}$ and $\boldsymbol{\varphi}(\mathbf x) = \text{Log}(\bar{\mathbf x}^{-1}\mathbf{x}) \in \mathbb{R}^3 $. ``Log" is the logarithmic function that maps a rotation matrix to its rotation vector~\cite{murray1994mathematical, lynch2017modern} and its inverse map is denoted as ``Exp". As a result, the gradient of $\mathbf g (\mathbf x): SO(3) \mapsto \mathbb{R}^m $ is

\begin{equation}~\label{e:partial_der_SO3}
    \nabla_{\mathbf{x}} \left( \mathbf g (\mathbf{x}) \right) (\bar{\mathbf{x}}) = \left. \dfrac{\partial \mathbf g \left( \bar{\mathbf x} \cdot \text{Exp} ( {\boldsymbol{\theta}}) \right) }{\partial \boldsymbol{\theta} } \right|_{\boldsymbol{\theta} = \mathbf 0} 
\end{equation}

It should be noted that the coordinate charts are not unique and a variety of alternatives could also be used. Two apparent alternatives are 1) using different (minimal) rotation parameterizations (e.g., Euler angles) instead of rotation vector as the local coordinates $\boldsymbol{\theta}$; and 2) off-centering the map $\boldsymbol{\varphi}(\mathbf x)$ from $\bar{\mathbf x}$. The first alternative will lead to the same partial differentiation as (\ref{e:partial_der_SO3}) while the second one produces an additional Jacobian matrix. Nevertheless, the rank of the observability matrix (\ref{e:obsv_lie}) are identical in all cases, hence different coordinate chart does not affect the system observability. In fact, the ``Exp/Log" map used in (\ref{e:partial_der_SO3}) lends itself a simple observability matrix for rank analysis, as will be seen in the following subsections. This parameterization is also widely used in $SO(3)$ optimizations~\cite{kummerle2011g, hertzberg2013integrating, ceres-solver}. 

\subsection{Observability of Input formulation}
In this section, we analyze the observability of the input formulation (\ref{e:kine_model_input}) based on the observability rank condition outlined above. Similar (or even more complicated) observability analysis have been explored before in visual-inertial navigation \cite{weiss2012real}, camera-IMU extrinsic calibration \cite{mirzaei2008kalman, kelly2011visual}, and parameter identification \cite{wuest2019online}. In all these work, the quaternion $\mathbf q$ is used to represent robots' attitude and is treated as a flat vector in $\mathbb{R}^4$ to compute the Lie derivatives as in (\ref{e:1st_lie_def_Rn}) and partial derivations as in (\ref{e:partial_derivation_Rn}). This method suffers from several drawbacks: 1) In order to compensate for the over-parameterization when using quaternion, an additional virtual measurement $\mathbf q^T \mathbf q = 1$ must be used to prevent rank deficiency; 2) The over-parameterization in state vector and additional virtual measurement lead to a rather complicated observability matrix, the resultant rank analysis is also rather difficult. Applying such analysis to higher order systems (e.g., state formulation (\ref{e:kine_model_state})) is cumbersome; 3) Viewing quaternion $\mathbf q$ as a vector in $\mathbb{R}^4$ undermines the structure of the unit quaternion. More concretely, the locally weakly observability based on $\mathbf q \in \mathbb{R}^4$ implies a neighborhood where the distance is defined as $d(\mathbf q_1, \mathbf q_2) = \| \mathbf q_1 - \mathbf q_2 \|_2$, which is not a natural metric for $SO(3)$ (i.e., it does not form $SO(3)$ into a Riemannian manifold). 

We show that all these drawbacks can be avoided if we develop the Lie derivatives directly on manifold as in (\ref{e:1st_lie_def}) and (\ref{e:lie_def}) and the associated partial derivations as in (\ref{e:partial_der_manifold}). The resultant observaiblity analysis is 1) Simple, it does not need to over-parameterize the state vector hence requiring no virtual measurement. The resultant observability matrix is easy for rank analysis; 2) Scalable, the simple and compact observability matrix allows the analysis to be conducted for more complicated and higher order systems such as the state formulation (\ref{e:kine_model_state}); and 3) Natural, since the observability analysis is conducted directly on manifold, the neighborhood it implies is based on the natural metric that forms $\mathcal M$ into a Riemannian manifold. For $SO(3)$, the metric is $d(\mathbf R_1, \mathbf R_2) = \| \text{Log} \left( \mathbf R_2^T \mathbf R_1 \right) \|_2 $, which is a geodesic on $SO(3)$~\cite{bullo1995proportional}. 

To analyze the observability of the input formulation (\ref{e:kine_model_input}), we ignore the noise $\mathbf w$ and $\mathbf n$ and rewrite it as an input-linear system as below:

\begin{equation}\nonumber
\begin{aligned}
\underbrace{\begin{bmatrix}  \dot{\mathbf p} \\  \dot{\mathbf v} \\ \dot{\mathbf R} \\  \dot{\mathbf c} \\ \dot{\mathbf b}_{\mathbf a}\\ \dot{\mathbf b}_{\bm \omega} \end{bmatrix}}_{\dot{\mathbf x}}
&= \underbrace{\begin{bmatrix} \mathbf v \\ \mathbf g - \mathbf R \mathbf b_{\mathbf a} \\ -\mathbf R \lfloor \mathbf b_{\bm \omega}\rfloor \\ \mathbf 0 \\ \mathbf 0\\ \mathbf 0 \end{bmatrix}}_{\mathbf f_0 (\mathbf x)} 
\! +\! \sum_{i=1}^{3}\underbrace{ \begin{bmatrix} \mathbf 0 \\ \mathbf 0\\ \mathbf R\lfloor \mathbf {e}_i\rfloor \\ \mathbf 0 \\ \mathbf 0\\ \mathbf 0 \end{bmatrix}}_{\mathbf f_i (\mathbf x)} \omega_{m_i} 
\! +\! \sum_{i=1}^{3} \underbrace{\begin{bmatrix} \mathbf 0 \\ \mathbf R \mathbf e_i\\ \mathbf 0 \\ \mathbf 0 \\ \mathbf 0\\ \mathbf 0 \end{bmatrix}}_{\mathbf f_{i+3}(\mathbf x)} {a}_{m_i} \\
\mathbf h_{1}\left(\mathbf x\right) &=\mathbf p + \mathbf R \mathbf c, \quad \mathbf h_{2}\left(\mathbf x\right) = \mathbf R^T \mathbf e_1 
\end{aligned}
\end{equation}

1) \textit{Zero-order Lie derivatives ($\mathcal{L}^0 \mathbf h_{1}, \mathcal{L}^0 \mathbf h_{2}$)}: By the definition in (\ref{e:lie_def_0}), we have
\begin{equation}\nonumber
\begin{aligned}
\left( \mathcal{L}^0 \mathbf h_{1} \right) (\mathbf x) &= \mathbf p + \mathbf R \mathbf c \\
\left( \mathcal{L}^0 \mathbf h_{2} \right) (\mathbf x) &=\mathbf R^T \mathbf e_1 
\end{aligned}
\end{equation}
and the partial derivation based on (\ref{e:partial_der_SO3})
\begin{equation}\nonumber
\begin{aligned}
\nabla_{\mathbf x} \left( (\mathcal{L}^0 \mathbf h_{1}) (\mathbf x) \right) (\bar{\mathbf x})  &=\begin{bmatrix} \mathbf I & \mathbf 0 & -\bar{\mathbf R} \lfloor \bar{\mathbf c} \rfloor & \bar{\mathbf R} & \mathbf 0 & \mathbf 0 \end{bmatrix}  \\
\nabla_{\mathbf x} \left( (\mathcal{L}^0 \mathbf h_{2}) (\mathbf x) \right) (\bar{\mathbf x}) &=\begin{bmatrix} \mathbf 0 & \mathbf 0 & \lfloor \bar{\mathbf R}^T \mathbf e_1\rfloor & \mathbf 0 & \mathbf 0 & \mathbf 0 \end{bmatrix}
\end{aligned}
\end{equation}

2) \textit{First-order Lie derivatives ($\mathcal{L}^1_{\mathbf f_0} \mathbf h_{1}, \mathcal{L}^1_{\mathbf f_1} \mathbf h_{1}, \mathcal{L}^1_{\mathbf f_2} \mathbf h_{1}, \mathcal{L}^1_{\mathbf f_0} \mathbf h_{2}$)}: From (\ref{e:obsv_lie_0_order}), we have
\begin{equation}\nonumber
\begin{aligned}
& \left( \mathcal{L}^1_{\mathbf f_0} \mathbf h_{1} \right) (\mathbf x) = \dot{\mathbf h}_1(\mathbf x), \quad \text{subject to} \ \dot{\mathbf x} = \mathbf f_0(\mathbf x) \\
&=  \mathbf v + \dot{\mathbf R} \mathbf c , \quad  \text{subject to}\  \dot{\mathbf R} = -\mathbf R \lfloor \mathbf b_{\bm \omega} \rfloor \\
&= \mathbf v - \mathbf{R} \lfloor \mathbf b_{\bm \omega} \rfloor \mathbf c 
\end{aligned} 
\end{equation}

Hence, the first order Lie derivatives are 
\begin{subequations}\nonumber
\begin{alignat}{2}
\left( \mathcal{L}^1_{\mathbf f_0} \mathbf h_{1} \right) (\mathbf x) &= \mathbf v - \mathbf{R} \lfloor \mathbf b_{\bm \omega} \rfloor \mathbf c \\
\left( \mathcal{L}^1_{\mathbf f_i} \mathbf h_{1} \right) (\mathbf x) &=  \mathbf{R} \lfloor \mathbf e_i\rfloor \mathbf c, \ i = 1,2,3 \\
\left( \mathcal{L}^1_{\mathbf f_0} \mathbf h_{2} \right) (\mathbf x) &= \lfloor \mathbf b_{\bm \omega}\rfloor \mathbf{R}^T \mathbf e_1
\end{alignat} 
\end{subequations}
and the associated partial derivations are
\begin{subequations}\nonumber
\begin{alignat}{2}
\nabla_{\mathbf x} \left( (\mathcal{L}^1_{\mathbf f_0} \mathbf h_{1}) (\mathbf x) \right) (\bar{\mathbf x})  &=\begin{bmatrix} \mathbf 0 & \mathbf I & \bullet & \bullet & \mathbf 0 & \bullet \end{bmatrix}  \\
\nabla_{\mathbf x} \left( (\mathcal{L}^1_{\mathbf f_i} \mathbf h_{1}) (\mathbf x) \right) (\bar{\mathbf x})  &=\begin{bmatrix} \mathbf 0 & \mathbf 0 & \bullet & \bar{\mathbf R} \lfloor \mathbf e_i\rfloor  & \mathbf 0 & \mathbf 0 \end{bmatrix}  \\
\nabla_{\mathbf x} \left( (\mathcal{L}^1_{\mathbf f_0} \mathbf h_{2}) (\mathbf x) \right) (\bar{\mathbf x}) &=\begin{bmatrix} \mathbf 0 & \mathbf 0 & \bullet & \mathbf 0 & \mathbf 0 & - \lfloor \bar{\boldsymbol{\beta}}\rfloor \end{bmatrix}
\end{alignat}
\end{subequations}
where $\bar{\boldsymbol{\beta}} = \bar{\mathbf{R}}^T \mathbf e_1$ and those $\bullet$ are elements that do not affect the observability analysis but can be calculated easily when required. 

3) \textit{Second-order Lie derivatives ($\mathcal{L}^2_{\mathbf f_0 \mathbf f_0} \mathbf h_{1}$, $\mathcal{L}^2_{\mathbf f_{i+3} \mathbf f_0 } \mathbf h_{1}$, $\mathcal{L}^2_{\mathbf f_i \mathbf f_0 } \mathbf h_{2}$, $i = 1,2,3$)}: Finally, we calculate the second order Lie derivatives based on the recursive definition in (\ref{e:lie_def_k}):
\begin{equation}\nonumber
\begin{aligned}
\left( \mathcal{L}^2_{\mathbf f_0 \mathbf f_0} \mathbf h_{1} \right) (\mathbf x) &= \mathcal{L}^1_{ \mathbf f_0}  \left( \mathcal{L}^1_{\mathbf f_0} \mathbf h_{1} \right) (\mathbf x) = \frac{d \left( \left( \mathcal{L}^1_{ \mathbf f_0} \mathbf h_{1}\right) (\mathbf x) \right)}{dt} \\
&= \dot{\mathbf v}- \dot{\mathbf R} \lfloor \mathbf b_{\bm \omega} \rfloor \mathbf c, \quad \text{subject to} \ \dot{\mathbf x} = \mathbf f_0(\mathbf x) \\
&= \mathbf g - \mathbf R \mathbf b_{\mathbf a} + \mathbf{R} \lfloor \mathbf b_{\bm \omega} \rfloor ^2 \mathbf c
\end{aligned} 
\end{equation}
Hence,
\begin{subequations}\nonumber
\begin{alignat}{2}
\left( \mathcal{L}^2_{\mathbf f_0 \mathbf f_0} \mathbf h_{1} \right) (\mathbf x) &= \mathbf g - \mathbf R \mathbf b_{\mathbf a} + \mathbf{R} \lfloor \mathbf b_{\bm \omega}\rfloor^2 \mathbf c \\
\left( \mathcal{L}^2_{\mathbf f_{i+3} \mathbf f_0} \mathbf h_{1} \right) (\mathbf x) &= \mathbf R \mathbf e_i,\  i = 1,2,3 \\
\left( \mathcal{L}^2_{\mathbf f_{i} \mathbf f_0} \mathbf h_{2} \right) (\mathbf x) &= \lfloor \mathbf b_{\bm \omega}\rfloor \lfloor \mathbf{R}^T \mathbf e_1 \rfloor \mathbf e_i,  \ i = 1,2,3
\end{alignat}
\end{subequations}
and the associated partial derivations can be obtained.

Putting together all the partial derivations and replacing $\bar{\mathbf x}$ with ${\mathbf x}$ for the sake of notation simplicity, we obtain the following observability matrix:
\begin{equation}
\begin{aligned}\notag
\mathcal{O}_I \!\!=\!\! \begin{bmatrix}
\nabla_{\mathbf x} \mathcal{L}^0 \mathbf h_{1} \\
\nabla_{\mathbf x} \mathcal{L}^1_{\mathbf f_{0}} \mathbf h_{1} \\
\nabla_{\mathbf x} \mathcal{L}^1_{\mathbf f_{1}} \mathbf h_{1} \\
\nabla_{\mathbf x} \mathcal{L}^1_{\mathbf f_{2}} \mathbf h_{1} \\
\nabla_{\mathbf x} \mathcal{L}^2_{\mathbf f_{0} \mathbf f_{0}} \mathbf h_{1} \\
\nabla_{\mathbf x} \mathcal{L}^2_{\mathbf f_{4} \mathbf f_{0}} \mathbf h_{1} \\
\nabla_{\mathbf x} \mathcal{L}^2_{\mathbf f_{5} \mathbf f_{0}} \mathbf h_{1} \\
\nabla_{\mathbf x} \mathcal{L}^1_{\mathbf f_{0}} \mathbf h_{2} \\
\nabla_{\mathbf x} \mathcal{L}^2_{\mathbf f_{1} \mathbf f_{0}} \mathbf h_{2} \\
\nabla_{\mathbf x} \mathcal{L}^2_{\mathbf f_{2} \mathbf f_{0}} \mathbf h_{2}
\end{bmatrix} \!\!=\!\! \begin{bmatrix}\setlength{\arraycolsep}{0.2pt}
\mathbf I & \mathbf 0 & \bullet & \bullet &\mathbf 0 &\mathbf 0 \\
\mathbf 0 & \mathbf I\!\! & \bullet & \bullet &\mathbf 0 & \bullet \\
\mathbf 0 & \mathbf 0 & \bullet & \!\!\!{\mathbf R}\lfloor \mathbf e_1\rfloor\!\! & \mathbf 0 & \mathbf 0\\
\mathbf 0 & \mathbf 0 & \bullet & \!\!\!{\mathbf R}\lfloor \mathbf e_2\rfloor\!\! & \mathbf 0 & \mathbf 0\\
\mathbf 0 & \mathbf 0 & \bullet & \bullet & \!\!\!-{\mathbf R}\!\! & \bullet \\
\mathbf 0 & \mathbf 0 & \!\!\!-{\mathbf R}\lfloor \mathbf e_1\rfloor\!\!\! & \mathbf 0 & \mathbf 0 & \mathbf 0\\
\mathbf 0 & \mathbf 0 & \!\!\!-{\mathbf R}\lfloor \mathbf e_2\rfloor\!\!\! & \mathbf 0 & \mathbf 0 & \mathbf 0\\
\mathbf 0 & \mathbf 0 & \bullet & \mathbf 0 &\mathbf 0 &\!\!-\lfloor{\bm \beta}\rfloor\\
{\mathbf 0} & {\mathbf 0} & \bullet & {\mathbf 0} & {\mathbf 0} & \!\!\!-\lfloor\lfloor{\bm \beta}\rfloor\mathbf e_1\rfloor\\
{\mathbf 0} & {\mathbf 0} & \bullet & {\mathbf 0} & {\mathbf 0} & \!\!\!-\lfloor\lfloor{\bm \beta}\rfloor\mathbf e_2\rfloor
\end{bmatrix}
\end{aligned}
\end{equation}
which apparently has full column rank since $\lfloor \mathbf e_1 \rfloor \mathbf v = \mathbf 0$ and $\lfloor \mathbf e_2 \rfloor \mathbf v = \mathbf 0$ cannot hold simultaneously for any nonzero vector $\mathbf v$. Therefore, the input formulation (\ref{e:kine_model_input}) is observable. Moreover, It is obviously that when compared with the observability analysis using quaternion parameterization~\cite{mirzaei2008kalman, kelly2011visual}, the observability matrix $\mathcal{O}_I$ is much simpler and its rank is immediately available.
\begin{figure*}[b]
\hrulefill
\begin{align}~\label{e:obsv_stat_simplify}
\mathcal{O}_{S}&=\left[\begin{array}{c}
    \multirow{3}*{${}^0_{n_1}\mathcal G^{\mathbf h_1}_{\mathbf x}$}\\ \\ \\
   [1mm]\hdashline[1pt/1pt] \\
   [-3mm] \multirow{2}*{${}^0_{n_2}\mathcal G^{\mathbf h_2}_{\mathbf x}$}\\ \\
   [1mm]\hdashline[1pt/1pt] \\
   [-3mm] \multirow{2}*{${}^0_{n_a}\mathcal G^{\mathbf h_3}_{\mathbf x}$}\\ \\
   [1mm]\hdashline[1pt/1pt] \\
   [-3mm] \multirow{2}*{${}^0_{n_{\omega}}\mathcal G^{\mathbf h_4}_{\mathbf x}$}\\ \\
\end{array}\right]
= \underbrace{\left[\begin{array}{cccccc;{1pt/1pt}cc;{1pt/1pt}ccc}
\mathbf I_{3} & \mathbf 0 & \bullet & \bullet &\mathbf 0 &\mathbf 0 & \mathbf 0 & \mathbf 0 & \mathbf 0 & \mathbf 0 \\
\mathbf 0 & \mathbf I_{3} & \bullet & \bullet &\mathbf 0 &\mathbf 0 &\mathbf 0 & \mathbf 0 & \mathbf 0 & \bullet \\
   {\mathbf 0} & {\mathbf 0} & {}^2_{n_1}\mathcal G^{\mathbf h_1}_{\bm \theta} & {}^2_{n_1}\mathcal G^{\mathbf h_1}_{\mathbf c} & \mathbf 0 &\mathbf 0  & {}^2_{n_1}\mathcal G^{\mathbf h_1}_{\mathbf a} & \bullet & {}^2_{n_1}\mathcal G^{\mathbf h_1}_{\bm \omega} & \bullet\\
   [1mm]\hdashline[1pt/1pt]&&&&&&&&& \\
   [-3mm]\mathbf 0 & \mathbf 0 & \lfloor \bm \beta\rfloor & \mathbf 0 &\mathbf 0 &\mathbf 0 &\mathbf 0 &\mathbf 0 &\mathbf 0 &\mathbf 0 \\
   {\mathbf 0} & {\mathbf 0} & {}^1_{n_2}\mathcal G^{\mathbf h_2}_{\bm \theta} & {\mathbf 0} &\mathbf 0 &\mathbf 0 &\mathbf 0 & {\mathbf 0} & {}^1_{n_2}\mathcal G^{\mathbf h_2}_{\bm \omega} & \bullet \\
   [1mm]\hdashline[1pt/1pt]&&&&&&&&& \\
   [-3mm] \mathbf 0 & \mathbf 0 & \mathbf 0 & \mathbf 0 & \mathbf I_3 & \mathbf 0 & \mathbf I_3 & \mathbf 0 & \mathbf 0 & \mathbf 0 \\
      \mathbf 0 & \mathbf 0 & \mathbf 0 & \mathbf 0 & \mathbf 0 & \mathbf 0 & \mathbf 0 & \mathbf I_{3n_a-3} & \mathbf 0 & \mathbf 0\\
   [1mm]\hdashline[1pt/1pt]&&&&&&&&& \\
   [-3mm] \mathbf 0 & \mathbf 0 & \mathbf 0 & \mathbf 0 & \mathbf 0 & \mathbf I_3 & \mathbf 0 & \mathbf 0 & \mathbf I_3 & \mathbf 0 \\
   \mathbf 0 & \mathbf 0 & \mathbf 0 & \mathbf 0 & \mathbf 0 & \mathbf 0   & \mathbf 0 & \mathbf 0 & \mathbf 0 & \mathbf I_{3n_{\omega}-3}\\
\end{array}\right]}_{\bm{\Psi}}
\underbrace{\left[\begin{array}{ccc}
   \mathbf{I}_{18} & \mathbf{0} & \mathbf{0} \\
   [1mm] \hdashline[1pt/1pt]&& \\
   [-3mm] \mathbf{0} & \mathcal{O}_{\bm \gamma_{\mathbf a}} & \mathbf{0} \\
   [1mm] \hdashline[1pt/1pt]&& \\
   [-3mm] \mathbf{0} & \mathbf{0} & \mathcal{O}_{\bm \gamma_{\bm \omega}}
\end{array}\right]}_{\mathcal O_{\gamma}}
\end{align}
\end{figure*}
\subsection{Observability of State formulation}
In this section, we show that the observability analysis previously developed on manifold is scalable to higher order systems such as the state formulation (\ref{e:kine_model_state}). The system can be rewritten as a more compact form as below
\begin{equation}~\label{e:state_formula_model_detail}
    \begin{aligned}
    &\dot{\mathbf x}=\mathbf f_{0}\left(\mathbf x\right)\\
    &\mathbf h_{1}\left(\mathbf x\right)= \mathbf p + \mathbf R \mathbf c,\ \mathbf h_3\left(\mathbf x\right)= \mathbf b_{\mathbf a} + \mathbf C_{\bm \gamma_{\mathbf a}} \bm \gamma_{\mathbf a}  \\
    &\mathbf h_2\left(\mathbf x\right)= \mathbf R^T \mathbf e_1,\quad \mathbf h_4\left(\mathbf x\right)=\mathbf b_{\bm \omega} + \mathbf C_{\bm \gamma_{\bm \omega}} \bm \gamma_{\bm \omega}
    \end{aligned}
\end{equation}
where the function $\mathbf f_{0}\left(\mathbf x\right)$ comes from (\ref{e:kine_model_state}) when ignoring all the white noises. This system has no exogenous inputs, so only Lie derivatives $\mathcal L_{\mathbf f_{0} \cdots \mathbf f_0 }^{k}$ are required in determining the system observability. Further denoting
\begin{equation}\nonumber
    \begin{array}{cc}
{}^i_k\mathcal G^{\mathbf h}_{\mathbf x} (\bar{\mathbf x})=\begin{bmatrix} \nabla_{\mathbf x}\left( \left( \mathcal L^i_{\mathbf f_{0} \cdots \mathbf f_{0} } \mathbf h \right) (\mathbf x) \right) (\bar{\mathbf x}) \\ \vdots \\
\nabla_{\mathbf x} \left( \left( \mathcal L^k_{\mathbf f_{0} \cdots \mathbf f_{0}} \mathbf h \right) (\mathbf x) \right) (\bar{\mathbf x}) \end{bmatrix}
    \end{array}
\end{equation}
and denoting $\bar{\mathbf x}$ as $\mathbf x$, we obtain the observability matrix $\mathcal{O}_{S}$ shown in (\ref{e:obsv_stat_simplify}), where $\mathcal O_{\bm \gamma_{\mathbf a}}$ and $\mathcal O_{\bm \gamma_{\mathbf a}}$ is the observability matrix of system $(\mathbf C_{\bm \gamma_{\mathbf a}}, \mathbf A_{\bm \gamma_{\mathbf a}}, \mathbf B_{\bm \gamma_{\mathbf a}})$ of dimension $n_a$ and $(\mathbf C_{\bm \gamma_{\bm \omega}}, \mathbf A_{\bm \gamma_{\bm \omega}}, \mathbf B_{\bm \gamma_{\bm \omega}})$ of dimension $n_{\omega}$, respectively. Since they are observable, $\mathcal{O}_{\bm \gamma}$ has full rank and rank($\mathcal O_S$) = rank($\boldsymbol{\Psi}$), implying that observability of the system (\ref{e:kine_model_state}) is independent from the statistical model as long as it is observable. This fundamentally enables us to use generic higher order statistical models for $\mathbf a$ and $\boldsymbol{\omega}$ to better characterize the motion, instead of the constant speed model used in~\cite{jones2011visual, chiuso2002structure, davison2003real,guo2014efficient}. Furthermore, we have the following Lemma:

\begin{lemma}~\label{lemma:obs}
The matrix $\mathcal O_S$ in (\ref{e:obsv_stat_simplify}) has full column rank if and only if the matrix
\begin{equation}
    \begin{aligned}\label{e:O_state_E}
        \left[
        \begin{array}{ccc;{1pt/1pt}c}
        \lfloor \bm \beta\rfloor & \mathbf 0 &\mathbf 0 &\mathbf 0\\
        {}^2_{n_1}\mathcal G^{\mathbf h_1}_{\bm \theta} & {}^2_{n_1}\mathcal G^{\mathbf h_1}_{\mathbf c} & {}^2_{n_1}\mathcal G^{\mathbf h_1}_{\mathbf a} & {}^2_{n_1}\mathcal G^{\mathbf h_1}_{\bm \omega}\\
        [1mm]\hdashline[1pt/1pt]&&&\\
        [-3mm] {\mathbf 0} & {\mathbf 0} & \mathbf 0 & {}^1_{n_2}\mathcal G^{\mathbf h_2}_{\bm \omega}
           \end{array}
           \right] = 
        \left[\begin{array}{c;{1pt/1pt}c}
           \mathcal{O}_{S_1} & \bullet \\
           [1mm]\hdashline[1pt/1pt]&\\
           [-3mm] \mathbf 0 & \mathcal{O}_{S_2}
        \end{array}\right]
    \end{aligned}
\end{equation}
has full column rank.
\end{lemma}

\begin{proof}Simplifying $\boldsymbol{\Psi}$ in (\ref{e:obsv_stat_simplify}) by the following 3 steps: 1) Since $\mathbf h_2 = {\mathbf R}^T \mathbf e_1 = {\bm{\beta}}$, it can be shown that $\nabla_{\bm{\theta}} \mathcal L^k_{\mathbf f_{0} \cdots \mathbf f_{0} } \mathbf h_2$ is in the form of $\mathbf K_k \lfloor {\boldsymbol{\beta}} \rfloor $, hence the sub-matrix ${}^1_n\mathcal G^{\mathbf h_2}_{\bm \theta}$ in $\bm \Psi$ can be completely eliminated by the preceding row $\lfloor {\boldsymbol{\beta}}\rfloor $; 2) Eliminate all $\bullet$ elements and the two $\mathbf I_3$ at (6, 7) and (8, 9) using the identity matrix in the preceding row or column; 3) Exchange row 3 and 4. Since matrix (block) row transformations do not change its column rank, the lemma statement holds.
\end{proof}

{Lemma} \ref{lemma:obs} implies that observability of system (\ref{e:kine_model_state}) generally depends on the initial attitude $\mathbf R$, parameter $\mathbf c$, acceleration $\mathbf a$ and its higher derivatives $\dot{\mathbf a}, \ddot{\mathbf a}, \cdots$, and angular velocity $\boldsymbol{\omega}$ and its higher derivatives $\dot{\boldsymbol{\omega}}, \ddot{\boldsymbol{\omega}}, \cdots$. While this looks more conservative than the input formulation where the observability matrix $\mathcal{O}_I$ has full column rank for all initial states ${\mathbf x}$, we should notice that $\mathcal O_I$ is obtained by assuming the input $\mathbf a_m$ and $\bm \omega_m$ can take arbitrary values in its domain \cite{hermann1977nonlinear}. This corresponds to choosing a value for each of $\mathbf a, \boldsymbol{\omega}$ and their higher order derivatives in the state formulation such that the first two block columns of (\ref{e:O_state_E}) have full column rank, which is true with trivial derivations. This implies that the kinematics states in the state formulation could be observed by properly driving the system with $\mathbf a$, $\bm \omega$ and their higher order derivatives.  Furthermore, we could prove that the extended states $\boldsymbol{\gamma}_{\mathbf a}$ and $\boldsymbol{\gamma}_{\boldsymbol \omega}$ are also observable by showing the full state formulation (\ref{e:kine_model_state}) is locally weakly observable (i.e., $\mathcal O_S$ in (\ref{e:obsv_stat_simplify}) has full column rank) if the system has sufficient excitation. More precisely, we introduce the following definition to characterize the set where the system has insufficient excitation.  

\begin{definition}~\label{thin}
A subset $U \subset \mathbb{R}^n$ is said to be {\it thin} if it contains no open subset $V \subset \mathbb{R}^n$. 
\end{definition}

An equivalent statement for Definition \ref{thin} is that for any point $\mathbf x \in U$, any open neighborhood $V \subset \mathbb{R}^n$ of $\mathbf x$ contains points not in $U$. Intuitively, a thin set does not span in every direction of $\mathbb R^n$. Fig.~\ref{fig:thin_examp} shows a few examples.

\begin{figure}[h]
	\centering
	\includegraphics[width=0.9\columnwidth]{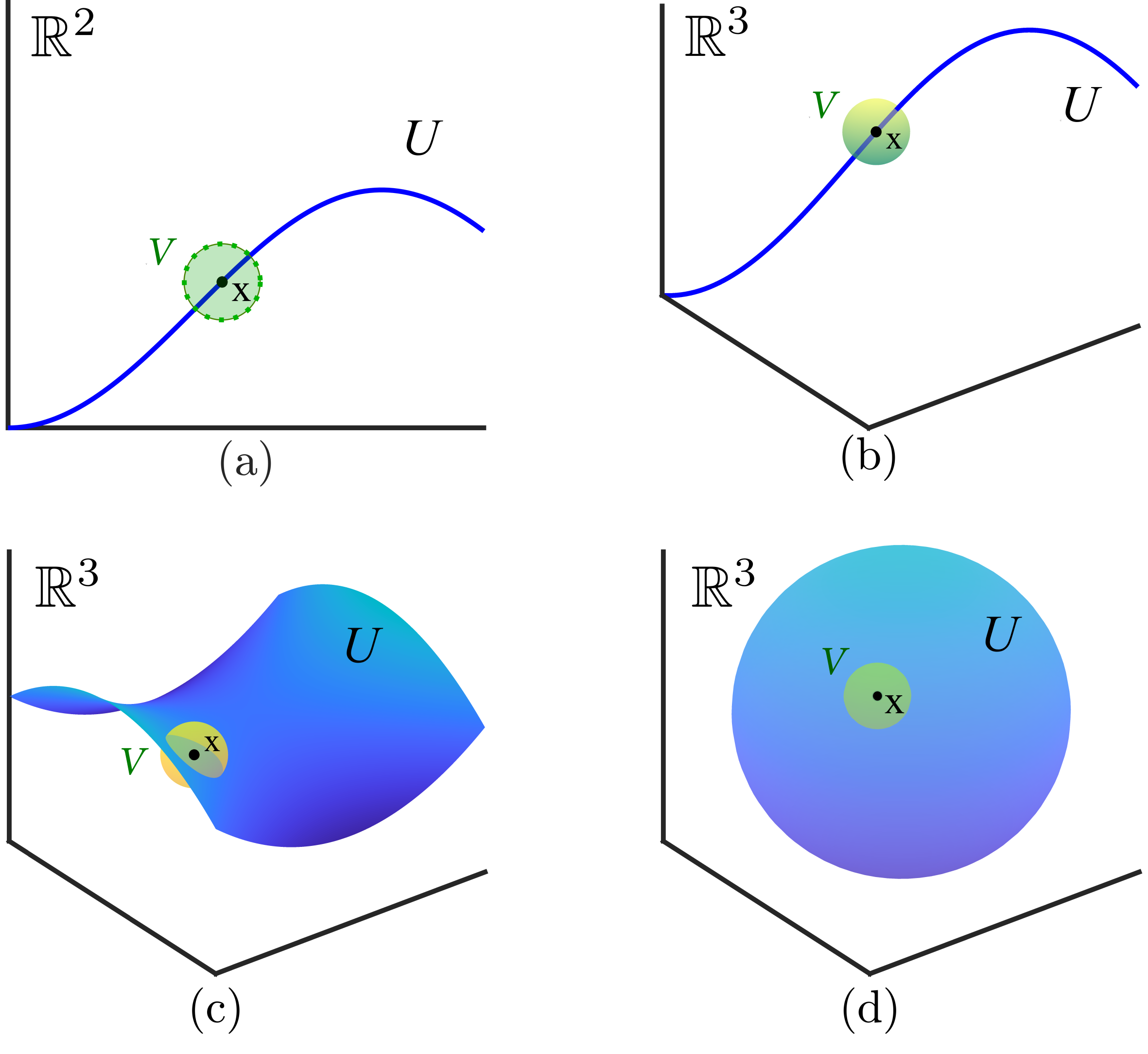}
	\caption{Examples for {\it thin} subsets: A curve in $\mathbb R^2$ (a), a curve in $\mathbb R^3$ (b), and a surface in $\mathbb R^3$ (c) are {\it thin} since they cannot contain any open subset $V$ in the respective $\mathbb R^n$; (d) A ball in $\mathbb R^3$ is {\it not thin} since it contains an open subset (e.g., a smaller ball) $V \subset \mathbb{R}^3$.
		\label{fig:thin_examp}}
	\vspace{-0cm}
\end{figure}

\begin{proposition}\label{proposition_O_S1}
If matrix $\mathcal O_{S_1}$ in (\ref{e:O_state_E}) is rank-deficient, the initial states $\bm \gamma_{\mathbf a}$ and $\bm \gamma_{\bm \omega}$ (or equivalently $\bm{\omega}^{(k)}, \mathbf a^{(k)}$) must lie on a thin subset of $\mathbb{R}^{n_{ a} + n_{ \omega}}$.
\end{proposition}

See Appendix.~\ref{appix:proposition1} for proof.

\begin{proposition}\label{proposition_O_S2}
If matrix $\mathcal O_{S_2}$ in (\ref{e:O_state_E}) is rank-deficient, then $\bm \omega^{(k)}$ is parallel to $\bm \beta $, i.e., $\bm \omega^{(k)}\parallel \bm \beta$, for all $k \geq 0$.
\end{proposition}
See Appendix.~\ref{appix:proposition2} for proof.\\

From (\ref{e:O_state_E}), it is seen that if $\mathcal O_{S_1}$ and $\mathcal O_{S_1}$ both has full column rank, $\mathcal O_S$ has full column rank. Then, the above two propositions imply that the system observability depends on $\boldsymbol{\omega}^{(k)}$ and $\mathbf a^{(k)}$ (or equivalently the extended state $\bm \gamma_{\mathbf a}$ and $\bm \gamma_{\bm \omega}$) only. Moreover, since the subset on which $\mathcal O_{S_1}$ and $\mathcal O_{S_1}$ hence $\mathcal O_{S}$ are rank-deficient is {\it thin}, the probability of the system initial states lying exactly on it is very low (e.g., due to the existence of noises). For the rest majority of $\mathbb{R}^{n_{ a} + n_{ \omega}}$, i.e., $\boldsymbol{\omega}^{(k)}$ and $\mathbf a^{(k)}$ have sufficient excitation, the system is locally weakly observable. This is exactly the same case for the input formulation where sufficient excitation in $\bm{\omega}$ and $\mathbf a$ is required. The only difference is that the $\boldsymbol{\omega}$ and $\mathbf a$ appear as input instead of state. 

\section{Simulation Results}~\label{simulation}
In order to validate the proposed framework, we simulate a trajectory for the UAV in Fig. \ref{fig:example_sim}. In the simulation, the UAV takes off at $2\; s$ and stays at $5\; m$ height, then lands slowly after $12\; s$. The UAV is driven by sinusoidal acceleration and angular acceleration of multiple frequencies, which give sufficient excitation to the system. The IMU is installed at $\mathbf c = \begin{bmatrix}0.5\; m & 0.5\; m & 0.5\; m \end{bmatrix}^T$ offset to the position sensor. The UAV position, acceleration, and angular velocity are added with white Gaussian noises to produce the simulated measurements. Fig.~\ref{fig:path_sim} shows the position and attitude of the IMU in the simulation where Euler angles (Yaw, Pitch and Roll) following the extrinsic Z-Y-X order are used to represent the attitude.
\begin{figure}[b]
	\centering
	\includegraphics[width=1\columnwidth]{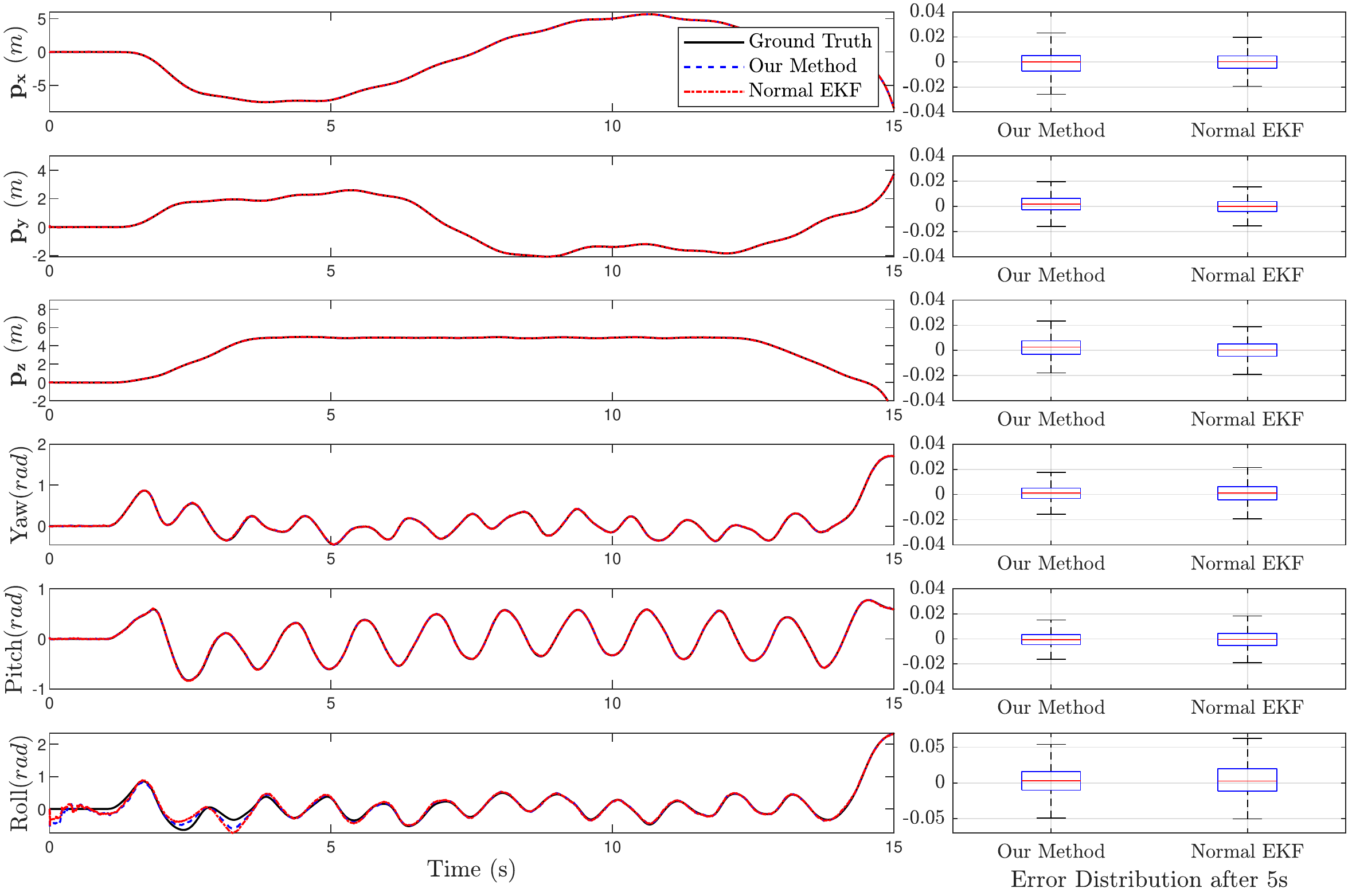}
	\caption{The estimation results and error distribution of position and attitude in the simulation.
		\label{fig:path_sim}}
	\vspace{-0cm}
\end{figure}

We compare the state estimation results of two error-state EKFs, one based on the state formulation (\ref{e:kine_model_state}) and the other based on the commonly used input formulation (\ref{e:kine_model_input}) (i.e., referred to as ``Normal EKF"). For the state formulation, we use a 4-th order integrator as in (\ref{e:n_int_model}) to model the acceleration and angular velocity independently. Fig.~\ref{fig:path_sim} and Fig.~\ref{fig:error_sim} show the comparison results, where the kinematics states (i.e, position $\mathbf p$, velocity $\mathbf v$, attitude $\mathbf R$, and biases $\mathbf b_{\mathbf a}, \mathbf b_{\bm \omega}$) and extrinsic parameter $\mathbf c$ are shown because the input formulation can only estimate these states. From Fig.~\ref{fig:path_sim} and Fig.~\ref{fig:error_sim}, we can see that both EKFs converge to the ground truth values, and the estimation errors after convergence have comparable variances. This verifies that our method can achieve similar estimation performance on kinematics state as that of the input formulation. 

\begin{figure}[t]
	\centering
	\includegraphics[width=1\columnwidth]{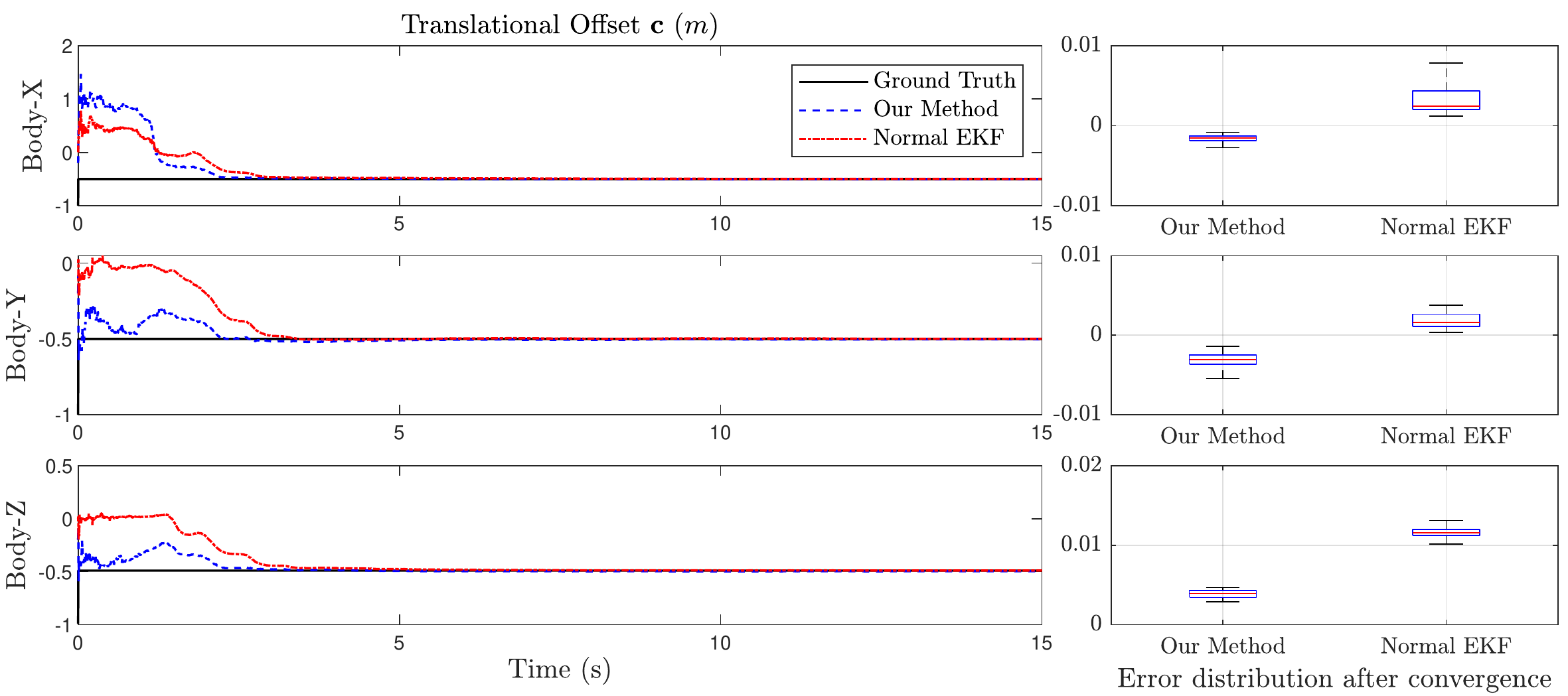}
	\caption{The estimation results of the translational offset $\mathbf c$ in simulation.
		\label{fig:error_sim}}
	\vspace{-0cm}
\end{figure}

\begin{figure}[t]
	\centering
	\includegraphics[width=1\columnwidth]{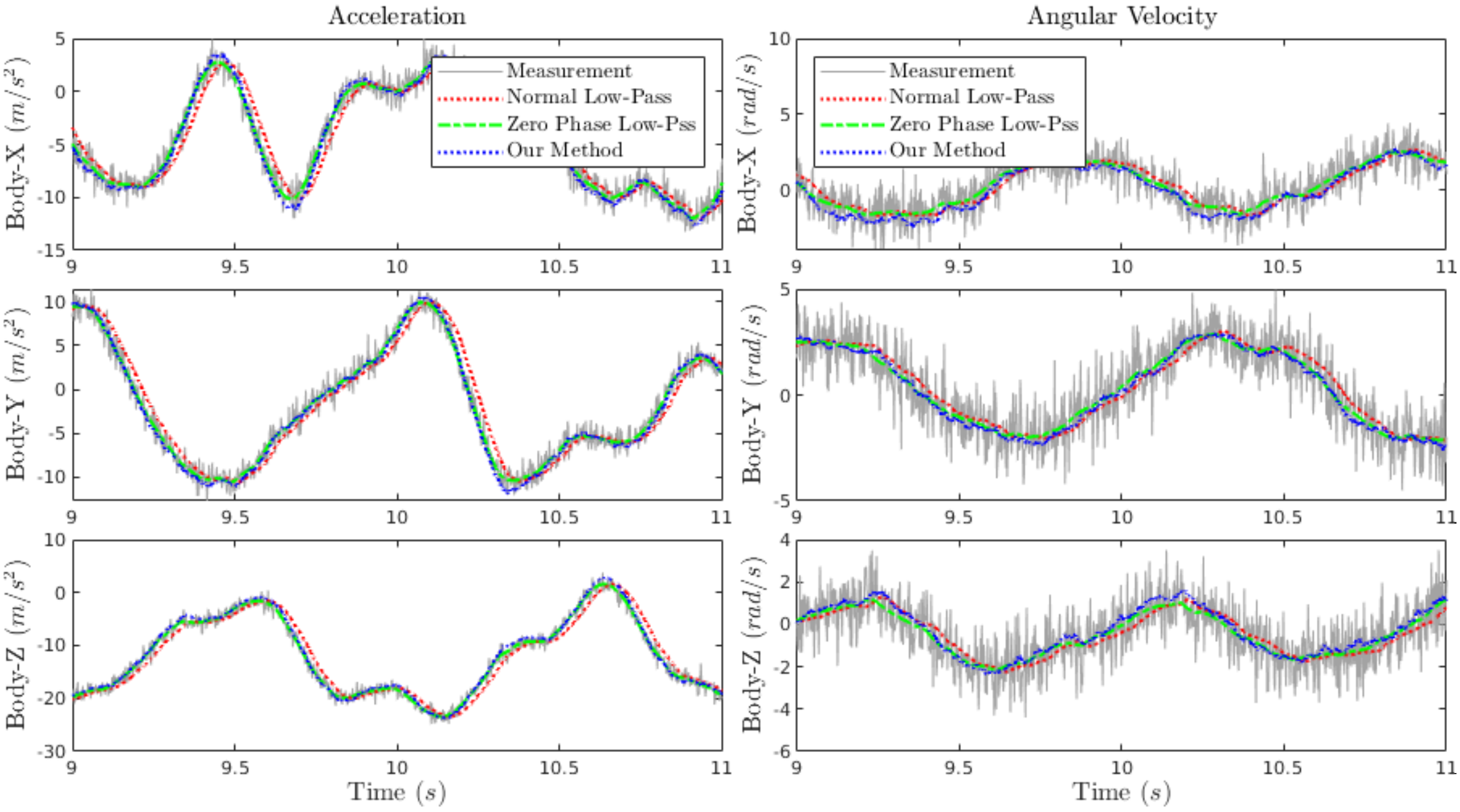}
	\caption{Comparison between our method, normal low-pass filter and zero-phase low-pass filter in the simulation.
		\label{fig:filt_sim}}
	\vspace{-0cm}
\end{figure}

A major advantage of the proposed state formulation is its ability to estimate the acceleration and angular velocity. To evaluate the estimation performance, we compare its estimation results with a first-order Butterworth low-pass filter (referred to as ``Normal Low-Pass") and a non-causal zero-phase low-pass filter (referred to as ``Zero Phase Low-Pass") \cite{gustafsson1996determining} directly applied on the accelerometer and gyro measurements removing biases. The Butterworth low-pass filter is designed following section \ref{case_study} such that it attains a noise attenuation level similar to EKF, and the zero-phase filter runs the Butterworth filter in its forward direction \cite{gustafsson1996determining}. All the results are shown in Fig.~\ref{fig:filt_sim}, where we can see that all the three filters achieve similar noise attenuation w.r.t. the raw measurements. Moreover, it is apparent that the normal low-pass filter has a considerable delay (40 $ms$) while the filtered results of the zero-phase filter and our EKF are very close to the ground truth values. It should be noted that our EKF is causal, being able to run in an online fashion, while the zero-phase filter is not, and that the 4-th order integrator model is not an exact model of the ground true sinusoidal acceleration and angular velocities. In conclusion, a statistical model is robust to capture the UAV dynamic motion and an EKF based on it can estimate the dynamics state with a delay as low as a non-causal zero-phase filter. 

\section{Experimental Validation}
This section supplies experimental results to verify the proposed approach. Two applications are demonstrated: the first one is the same to the simulation in Section VI: estimating both the kinematics state and dynamics states (i.e., acceleration and angular velocity) while simultaneously calibrating the offset between the the position reference point and IMU (see Fig. \ref{fig:example_sim}). The experiment of the first application is conducted on an UAV, an UGV and a handheld platforms (see Fig. \ref{fig:uas}) to show the versatility of the proposed statistical model. The second application is to calibrate the extrinsic parameters (both rotation and translation) between two IMUs. Finally, running time of the proposed method is evaluated.

\subsection{Simultaneous States Estimation and Extrinsic Calibration of POS/IMU Sensors}~\label{sec:case1}

As shown in Fig. \ref{fig:uas}, the system is configured with tracker balls of a motion capture system to provide the position and orientation measurement, and an onboard IMU sensor to provide the acceleration and angular velocity measurements. The two sensors (i.e., motion capture system and IMU) are pre-calibrated to align their reference frames. This system is identical to Fig.~\ref{fig:example_sim} and hence the proposed approach in section \ref{sim} and its implementation in section \ref{simulation} can be directly applied here. The three systems are moved randomly in the 3D space (for UAV and handheld platform) or on the floor (for UGV) to obtain sufficient excitation, their trajectories are shown in Fig. \ref{fig:pos_3d}. In each scenario, we compare the estimated kinematics state with that of the normal EKF based on the input formulation and the estimated acceleration and angular velocity with that of zero-phase low-pass filters.

\begin{figure}[h]
	\centering
	\includegraphics[width=1\columnwidth]{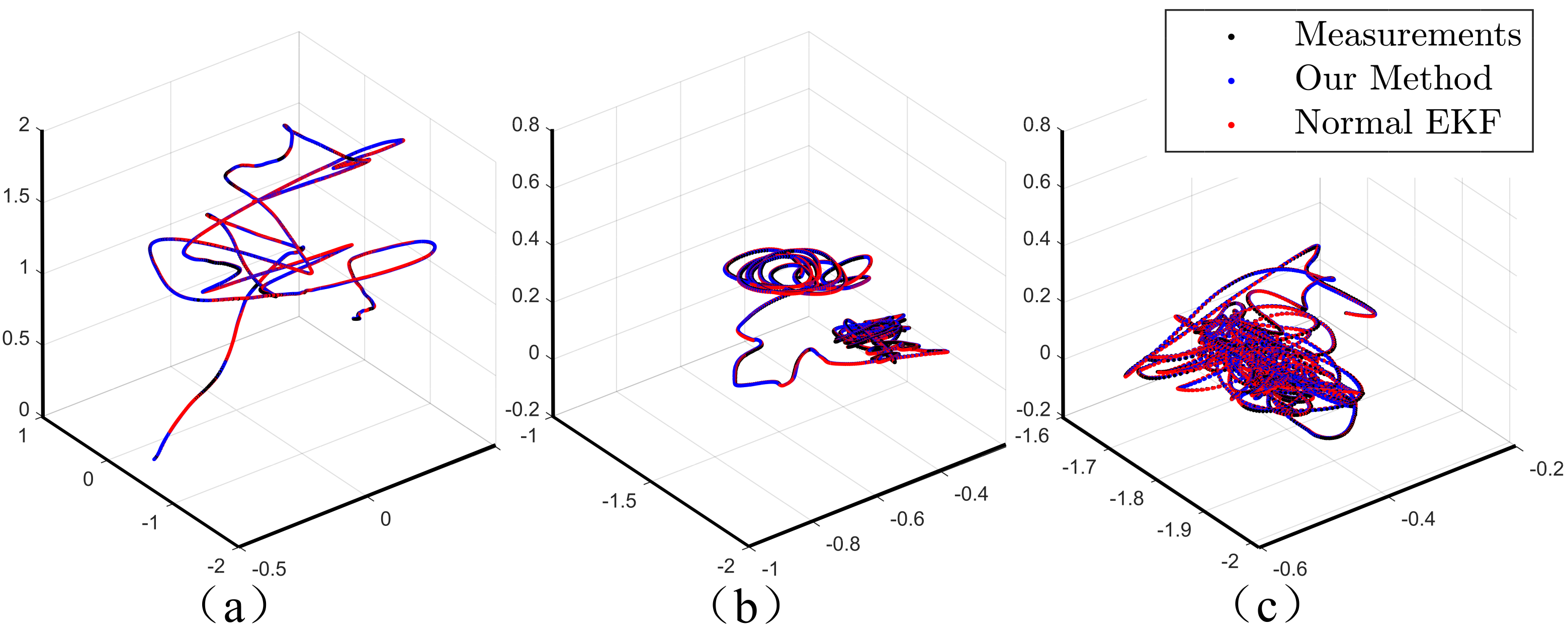}
	\caption{The trajectories in three validation experiments of the first application: (a) the UAV flight experiment; (b) the UGV experiment; (c) the handheld experiment.
		\label{fig:pos_3d}}
	\vspace{-0cm}
\end{figure}

\begin{figure}[t]
		\centering
		\includegraphics[width=1\columnwidth]{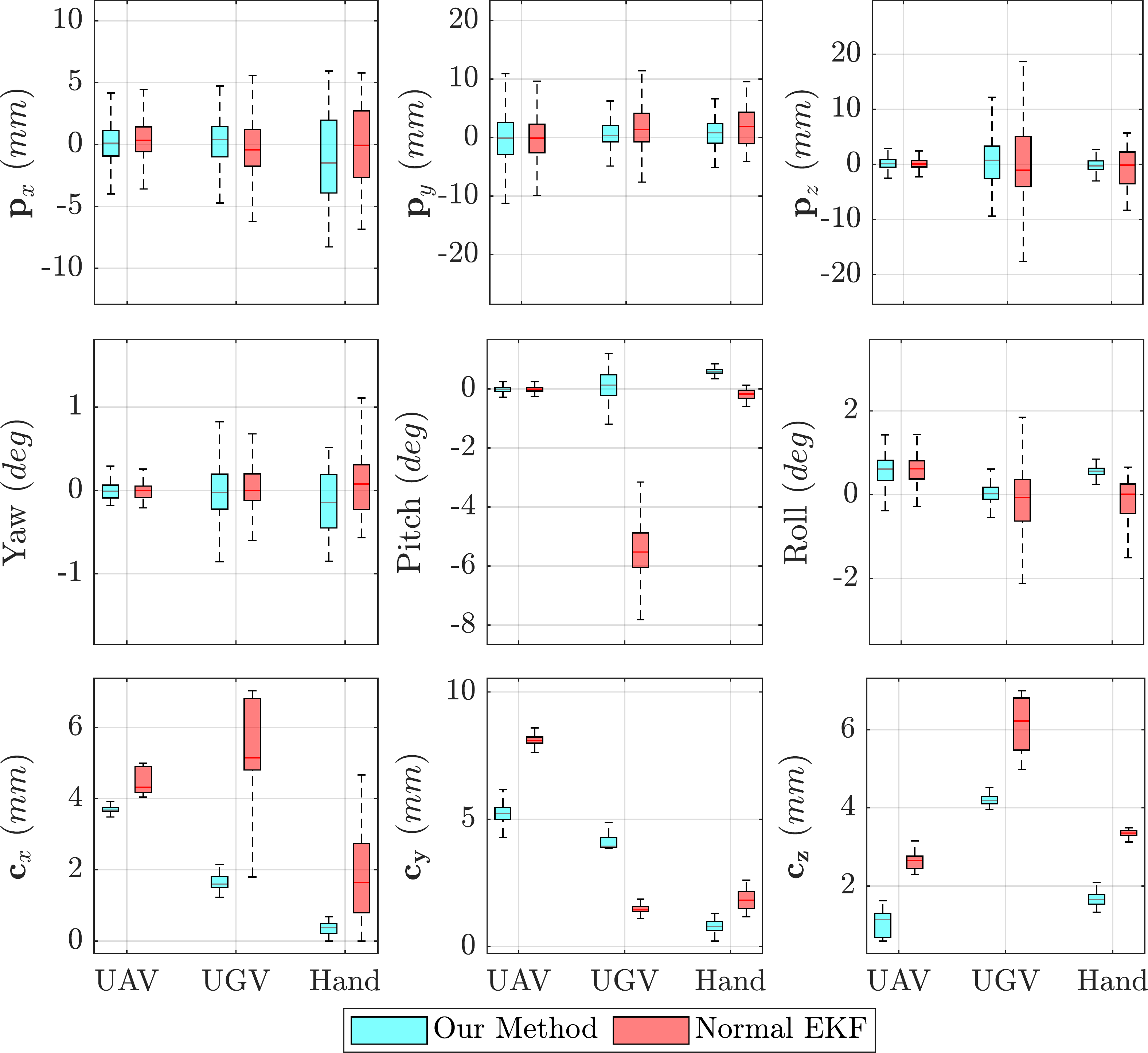}
		\caption{The error distribution of position ($\mathbf p_x$ $\mathbf p_y$ $\mathbf p_z$), attitude (Yaw Pitch Roll) and translational offset ($\mathbf c_x$ $\mathbf c_y$ $\mathbf c_z$) after convergence in UGV, UAV and handheld experiments.
			\label{fig:c1_error_bar}}
		\vspace{-0.cm}
\end{figure}

\begin{figure}[t]
	\centering
	\includegraphics[width=1\columnwidth]{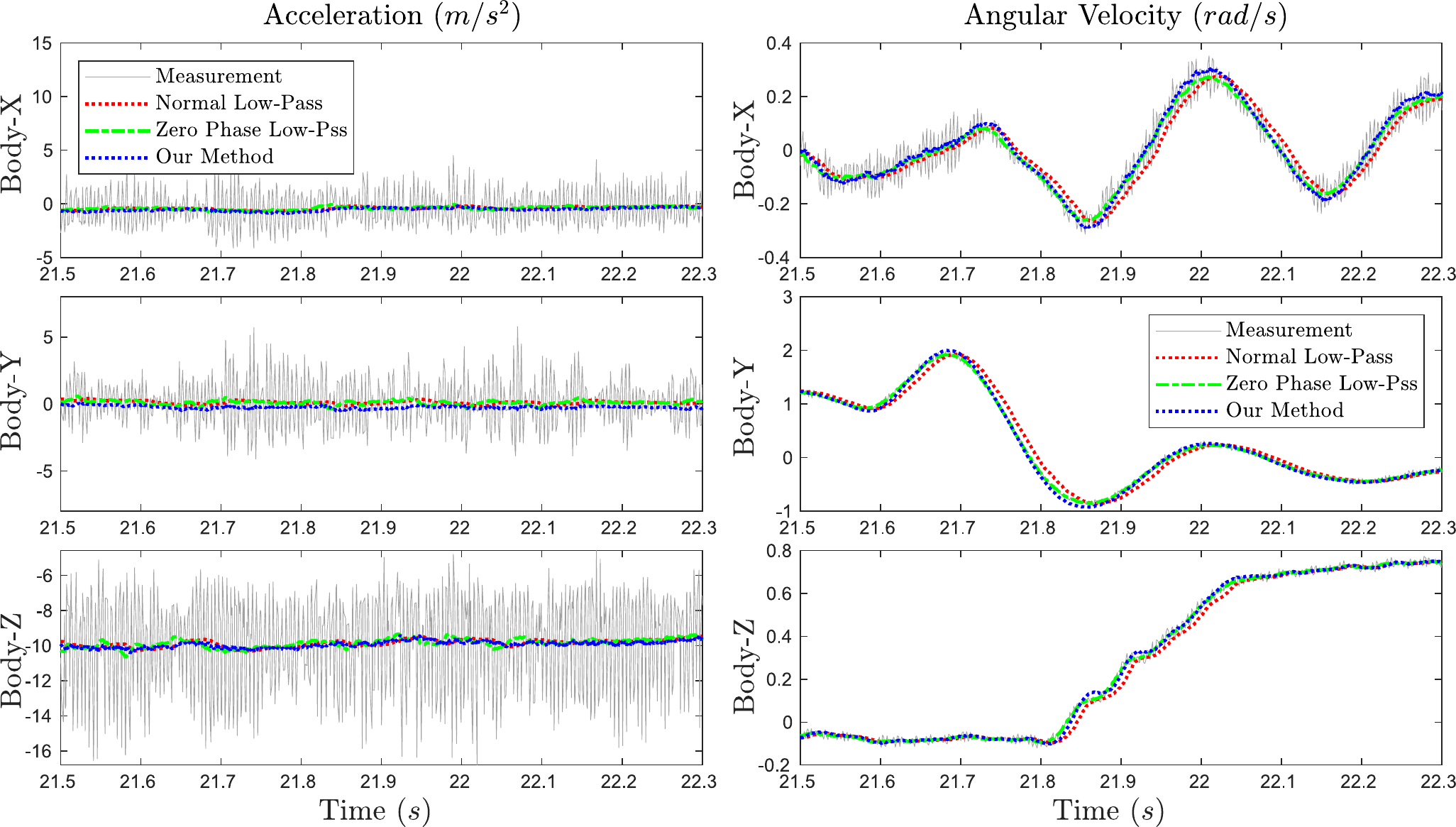}
	\caption{The comparison of acceleration and angular velocity estimations with a normal low-pass filter outputs and a non-causal zero-phase low-pass filter outputs, for the UAV flight experiment.
		\label{fig:c1_filtcompare}}
	\vspace{-0.cm}
\end{figure}

\begin{figure}[t]
	\centering
	\includegraphics[width=1\columnwidth]{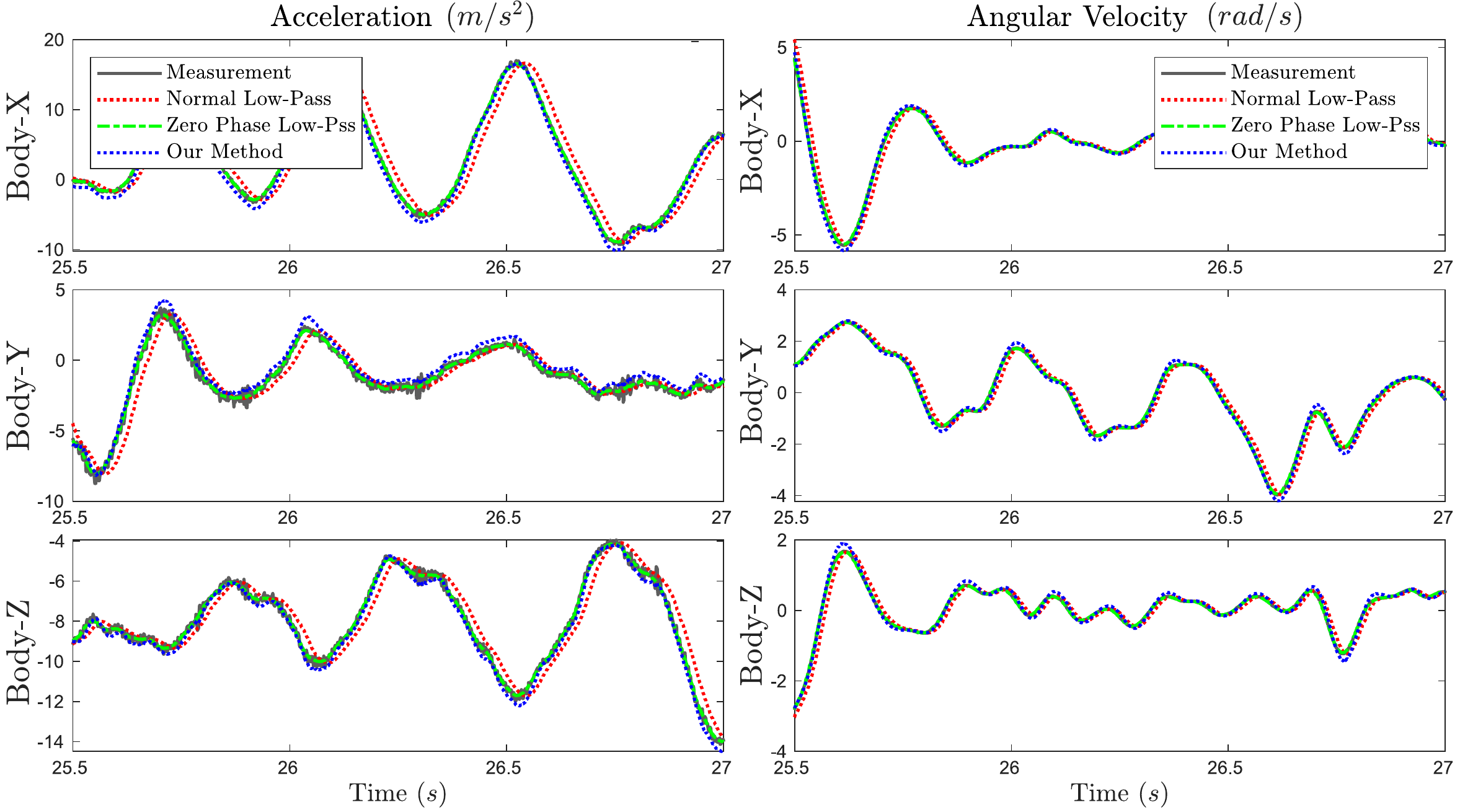}
	\caption{The comparison of acceleration and angular velocity estimations with a normal low-pass filter outputs and a zero-phase low-pass filter outputs, for the UGV experiment.
		\label{fig:c1_filt_comp_r1}}
	\vspace{-0.cm}
\end{figure}

\begin{figure}[h]
	\centering
	\includegraphics[width=1\columnwidth]{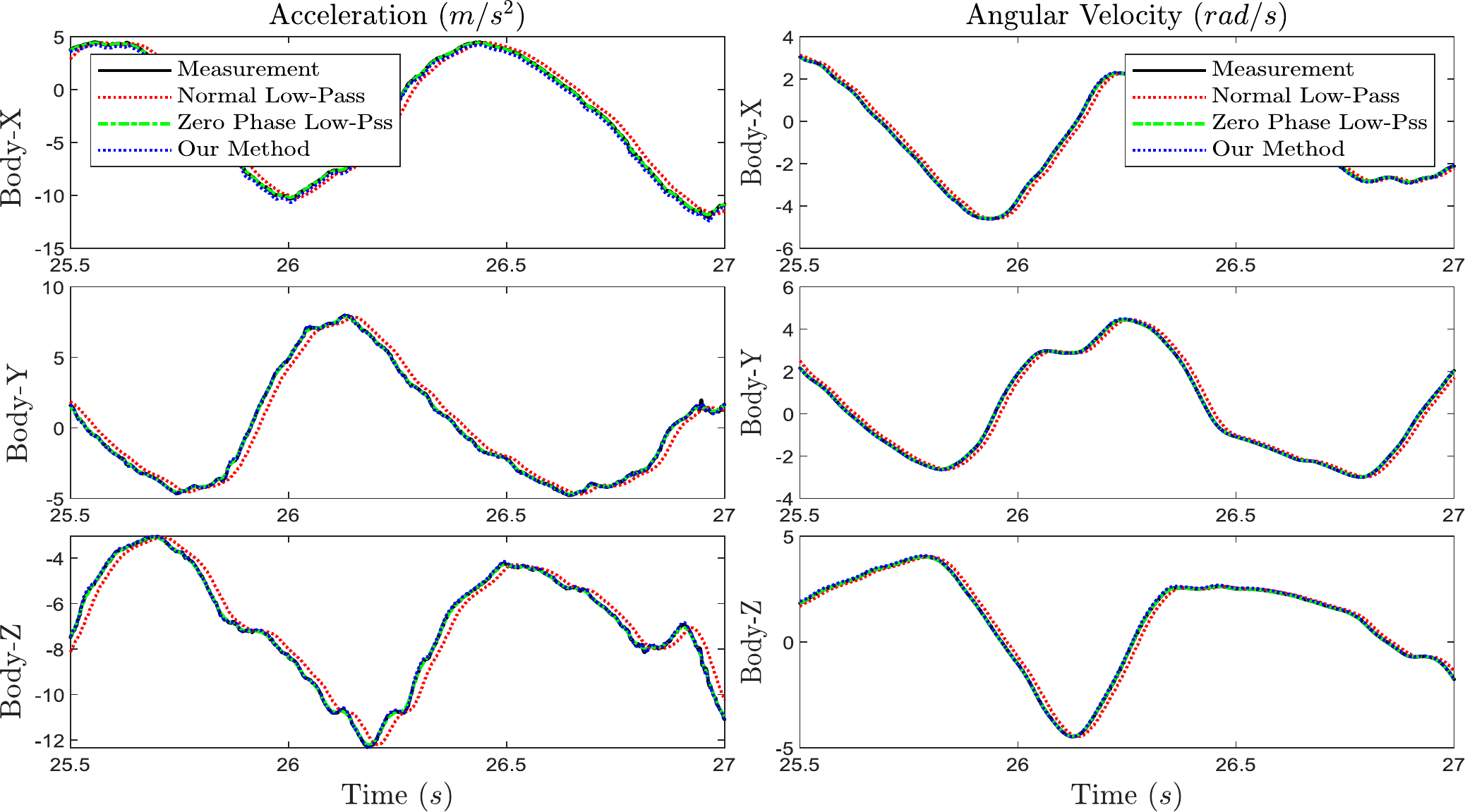}
	\caption{The comparison of acceleration and angular velocity estimations with a normal low-pass filter outputs and a zero-phase low-pass filter outputs, for the handheld experiment.
		\label{fig:c1_filt_comp_r2}}
	\vspace{-0.cm}
\end{figure}

The estimation performance of the kinematics states and extrinsic parameters are shown in Fig. \ref{fig:c1_error_bar}, where the first three figures compare the measured position $\mathbf p_m$ to the position prediction $(\widehat{\mathbf p} + \widehat{\mathbf R} \widehat{\mathbf c} )$ from the two EKFs. It is seen that in all these three platforms, both EKFs have comparable means and variances in position prediction errors. With the pre-calibration between the motion capture system and the onboard IMU, we also extract the full attitude measurement from the motion capture system and use it as the ground truth to evaluate attitude estimation performance. Shown in the middle three subplots of Fig. \ref{fig:c1_error_bar}, overall, the two EKFs have comparable mean and variance. For the UGV case, the pitch error of our method is significantly better than that of the normal EKF. It is also observed some residual errors in attitude estimation, which might be due to the calibration error between mocap and onboard IMU. Fig. \ref{fig:c1_error_bar} further shows the calibration error of the translational offset $\mathbf c$ by comparing the extrinsic estimation from the two EKFs to hand measured values. It can be seen that our method achieves smaller errors overall, the difference between the two EKFs are very small (within 5 $mm$), which is below the precision of hand measurements. These results show that the EKF based on the proposed statistical model achieves an estimation performance comparable to (and occasionally better than) conventional input formulation.  Besides the estimation of kinematics state, Fig. \ref{fig:c1_filtcompare}, Fig. \ref{fig:c1_filt_comp_r1} and \ref{fig:c1_filt_comp_r2} show the filter results of acceleration and angular velocity on the three platforms. It can be seen that our method has much lower filtering delay (e.g., 21 $ms$ less for the angular velocity of UAV) in all cases than the normal low-pass filter while they both achieve similar level of noise attenuation. The filtering performance is comparable to a non-causal zero-phase filter. For the acceleration of the UAV case, the noise suppression is easily seen but not the filtering delay due to the small and relative steady body accelerations. This is because the UAV undergoes low speed motions in a confined indoor area (see Fig.~\ref{fig:pos_3d} (a)), in this case, the rotor drag is very small \cite{martin2010true, faessler2017differential} and the specific force acting on the UAV is only the rotor thrust which is always along the body $Z$ direction. We could further exploit this property to calibrate the offset between the IMU position and the center of gravity: if such offset is present and the UAV undergoes rotation w.r.t. its center of gravity, accelerations along body $X$ and $Y$ will show up. This is however beyond the scope of this paper. For the UGV and handheld platforms, the reduced delay in acceleration estimation of our method relative to the low-pass filter is obvious due to the drastic acceleration change. These three experiments draw consistent conclusions: 1) The two EKFs have comparable performance in estimating kinematics state and extrinsic parameters; 2) Our proposed method can additionally filter the dynamics states such as acceleration and its filtering performance is comparable to a non-causal zero-phase filter. 

\subsection{Inter IMU calibration}\label{case2}
The multi-IMU case exists in many robotic applications, for example many UAV platform has back-up IMUs to obtain redundancies. Fig. \ref{fig:2imu} shows one such system with two IMUs. Extrinsic parameters (both translation and rotation) of the two IMUs must be calibrated to register their data to the same coordinate frame. This is possible with only onboard IMU measurements: since the two IMUs measure the same acceleration vector and angular velocity vector under their respective body frame, the relative rotation could be calibrated, and the relative translation could be calibrated since it relates the linear acceleration of one IMU to the angular acceleration of the other. It can be seen that to calibrate the relative rotation, either angular velocity or linear acceleration of the two IMUs are required, and to calibrate the relative translation, the angular acceleration (hence angular velocity) of one IMU and the linear acceleration of the other are required. This gives us a variety of possible combinations, in this case study, we make use of the angular velocity and the linear acceleration of IMU1 and the linear acceleration of IMU2.

\begin{figure}
	\centering
	\includegraphics[width=0.6\columnwidth]{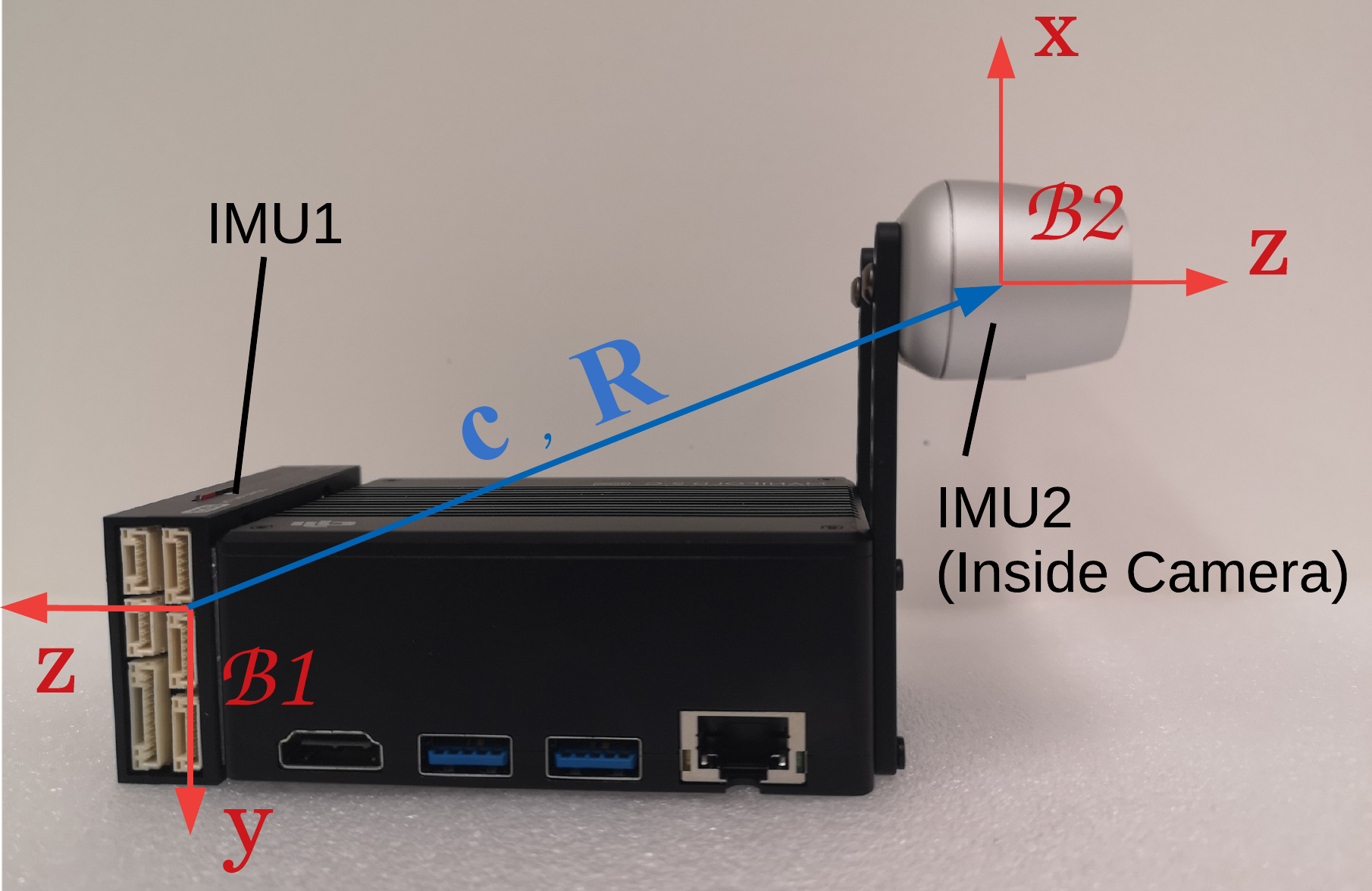}
	\caption{System configuration of inter IMU calibration.
		\label{fig:2imu}}
	\vspace{-0cm}
\end{figure}

To demonstrate how the proposed statistical model can be used to estimate the linear and angular acceleration, thus enabling online calibration of the two IMUs' extrinsic parameters, we first build a state space model. Referring to Fig. \ref{fig:2imu}, we denote $\mathcal{B}1$ and $\mathcal{B}2$ as the body frame of the two IMUs, respectively. The $\mathbf c$ and $\mathbf R$ in Fig.\ref{fig:2imu} denotes the relative translation and rotation between IMU1 and IMU2, respectively. Choose $\mathcal B1$ as the reference frame and model its linear acceleration $\mathbf a_1$ and angular acceleration $\boldsymbol{\tau}_1$ as two independent statistical model (\ref{e:gamma_model}) with state vector $\boldsymbol{\gamma}_{\mathbf a_1}$ and $\boldsymbol{\gamma}_{\boldsymbol{\tau_1}}$, respectively. Then, the state equation is:
\begin{equation}~\label{e:smodel2}
\begin{array}{ll}
&\dot{\bm \gamma}_{\mathbf a_1}=\mathbf A_{\bm \gamma_{\mathbf a_1}}{\bm \gamma}_{\mathbf a_1}+\mathbf B_{\bm \gamma_{\mathbf a_1}}\mathbf w_{\bm \gamma_{\mathbf a_1}}, \quad \mathbf a_1 = \mathbf C_{\bm \gamma_{\mathbf a_1}} {\bm \gamma_{\mathbf a_1}} \\
& \dot{\bm \gamma}_{\bm \tau_1}=\mathbf A_{\bm \gamma_{\bm \tau_1}}{\bm \gamma}_{\bm \tau_1}+\mathbf B_{\bm \gamma_{\bm \tau_1}}\mathbf w_{\bm \gamma_{\bm \tau_1}}, \quad \bm \tau_1 = \mathbf C_{\bm\gamma_{\bm \tau_1}} {\bm \gamma_{\bm \tau_1}} \\
& \dot{\bm \omega}_1 = \bm \tau_1, \quad \dot {\mathbf b}_{\bm \omega_1} = \mathbf w_{\mathbf b_{\bm \omega_1}}, \\
& \dot {\mathbf b}_{\mathbf a_1} = \mathbf w_{\mathbf b_{\mathbf a_1}}, \quad \dot {\mathbf b}_{\mathbf a_2} = \mathbf w_{\mathbf b_{\mathbf a_2}}  \\
&\dot {\mathbf c} = \mathbf 0, \quad \dot{\mathbf R} = \mathbf 0
\end{array}
\end{equation}
where $\mathbf b_{\mathbf a_i}$ and $\mathbf b_{\bm \omega_i}$ are the biases of the $i$-th IMU. It should be noticed that in the above model, we modeled the angular acceleration (instead of angular velocity as in Section VII. A) as a statistical model because it needs to be estimated. The angular acceleration model is then augmented with one additional deterministic integrator to obtain the angular velocity.

The measurement vector is $\mathbf y^T =\left[ \bm \omega_{m_1}^T \ \ \ \mathbf a_{m_1}^T \ \ \ \mathbf a_{m_2}^T \right]$ with components defined below. 
\begin{equation}~\label{e:mmodel2}
\begin{aligned}
\bm \omega_{m_1}\!&=\!\bm \omega_1 \!+ \!\mathbf b_{\bm \omega_1}\!+\!\mathbf n_{\bm \omega_1} \\
\mathbf a_{m_1}\!&=\!\mathbf a_1 + \!\mathbf b_{\mathbf a_1}\!+\mathbf n_{\mathbf a_1} \\
\mathbf a_{m_2} \!&=\! \mathbf R\!\left(\mathbf a_1\! +\!\lfloor{\bm \omega_1}\rfloor^2\mathbf c \!+\!\lfloor{\bm \tau_1}\rfloor \mathbf c\right)\! + \!\mathbf b_{\mathbf a_2}\!+\!\mathbf n_{\mathbf a_2} 
\end{aligned}
\end{equation}

Rewriting the state space model with state equation (\ref{e:smodel2}) and output equation (\ref{e:mmodel2}) into a more compact form as below
\begin{equation}~\label{e:state_compact}
    \begin{aligned}
    \dot{\mathbf x} &= \mathbf f(\mathbf x, \mathbf w) \\
    \mathbf y&= \mathbf h(\mathbf x)
    \end{aligned}
\end{equation}
where the state vector is
\begin{equation}\notag
    \mathbf x^T = \begin{bmatrix}
    \bm \gamma_{\mathbf a_1}^T, \bm \gamma_{\bm \tau_1}^T, \bm \omega_1^T, \mathbf b_{\bm \omega_1}^T, \mathbf b_{\mathbf a_1}^T, \mathbf b_{\mathbf a_2}^T, \mathbf c^T, \mathbf R^T
    \end{bmatrix}
\end{equation}

If define a constant vector $\bar{\mathbf x}$ such that
\begin{equation}~\label{e:minimal_cond}
    \begin{aligned}
        \bar{\mathbf x}^T &= \begin{bmatrix}
    \bar{\bm \gamma}_{\mathbf a_1}^T, \mathbf 0^T, \mathbf 0^T, \mathbf 0^T, \bar{\mathbf b}_{\mathbf a_1}^T, \bar{\mathbf b}_{\mathbf a_2}^T, \mathbf 0^T, \mathbf 0^T
    \end{bmatrix} \\
    s.t. & \quad \mathbf A_{\bm \gamma_{\mathbf a_1}}{\bar{\bm \gamma}}_{\mathbf a_1} = \mathbf 0, \quad \mathbf C_{\bm \gamma_{\mathbf a_1}} \bar{\bm \gamma}_{\mathbf a_1} + \bar{\mathbf b}_{\mathbf a_1} = \mathbf 0 \\
    & \quad \mathbf R \mathbf C_{\boldsymbol{\gamma}_{\mathbf a_1}} \bar{\bm \gamma}_{\mathbf a_1} + \bar{\mathbf b}_{\mathbf a_2} = \mathbf 0 
    \end{aligned}
\end{equation}
then it is easy to show that the system (\ref{e:state_compact}) is invariant to the initial state $\bar{\mathbf x}$. i.e., $\check{\mathbf x} = \mathbf x + \bar{\mathbf x} $ satisfy
\begin{equation}\notag
    \begin{aligned}
    \dot{\check{\mathbf x}} &= \mathbf f(\check{\mathbf x}, \mathbf w) \\
    \mathbf y&= \mathbf h(\check{\mathbf x})
    \end{aligned}
\end{equation}
which implies that $\check{\mathbf x}$ would produce the same output trajectory (ignoring noises) as that of $\mathbf x$ and hence these two states cannot be distinguished. This means that the state space model with state equation (\ref{e:smodel2}) and output equation (\ref{e:mmodel2}) is not a minimal realization.

Further investigating the structure of (\ref{e:minimal_cond}), we find that $\bar{\mathbf x} = \mathbf 0$ if and only if the subsystem
\begin{equation}~\label{e:subsystem}
    \begin{aligned}
    \begin{bmatrix}
    \dot{\bm \gamma}_{\mathbf a_1} \\ \dot{\mathbf b}_{\mathbf a_1} 
    \end{bmatrix} &= 
    \begin{bmatrix}
    \mathbf A_{\bm \gamma_{\mathbf a_1}} & \mathbf 0 \\
    \mathbf 0 & \mathbf 0
    \end{bmatrix}
    \begin{bmatrix}
    \bm \gamma_{\mathbf a_1} \\ \mathbf b_{\mathbf a_1} 
    \end{bmatrix} + 
    \begin{bmatrix}
    \mathbf B_{\bm \gamma_{\mathbf a_1}} & \mathbf 0 \\
    \mathbf 0 & \mathbf I
    \end{bmatrix}
    \underbrace{\begin{bmatrix}
    \mathbf w_{\bm \gamma_{\mathbf a_1}} \\ \mathbf w_{\mathbf b_{\mathbf a_1}} 
    \end{bmatrix}}_{\mathbf w_{\bm \gamma_{\mathbf a}}} \\
    \mathbf a \!&=\!\begin{bmatrix}
    \mathbf C_{\bm \gamma_{\mathbf a_1}} & \mathbf I 
    \end{bmatrix}
    \begin{bmatrix}
    \bm \gamma_{\mathbf a_1} \\ \mathbf b_{\mathbf a_1} 
    \end{bmatrix}
    \end{aligned}
\end{equation}
is observable. One sufficient condition for (\ref{e:subsystem}) to be observable is choosing $\mathbf A_{\bm \gamma_{\mathbf a_1}}$ to be non-singular. However, this is too conservative (i.e., it limits the representation capability of the statistical model) and unnecessary: when (\ref{e:subsystem}) is unobservable, it implies that the statistical model used to capture $\mathbf a$, the biased acceleration, has redundant states that should be eliminated when designing the statistical model in the first place.  Removing the redundant states can be achieved by transforming (\ref{e:subsystem}) into the Kalman canonical form as below \cite{hespanha2018linear}.
\begin{equation}~\label{e:Kalman_can}
    \begin{aligned}
    \begin{bmatrix}
    \dot{\bm \gamma}_{\mathbf a} \\ \bullet 
    \end{bmatrix} &= 
    \begin{bmatrix}
    \mathbf A_{\bm \gamma_{\mathbf a}} & \mathbf 0 \\
    \bullet & \bullet
    \end{bmatrix}
    \begin{bmatrix}
    \bm \gamma_{\mathbf a} \\ \bullet
    \end{bmatrix} + 
    \begin{bmatrix}
    \mathbf B_{\bm \gamma_{\mathbf a}} \\
    \bullet
    \end{bmatrix}
    \mathbf w_{\bm \gamma_{\mathbf a}}  \\
    \mathbf a \!&=\!\begin{bmatrix}
    \mathbf C_{\bm \gamma_{\mathbf a}} & \mathbf 0
    \end{bmatrix}
    \begin{bmatrix}
    \bm \gamma_{\mathbf a} \\ \bullet
    \end{bmatrix}
    \end{aligned}
\end{equation}
where $(\mathbf C_{\bm \gamma_{\mathbf a}}, \mathbf A_{\bm \gamma_{\mathbf a}})$ is observable subsystem and $\bullet$'s are terms belong to the unobservable subspace. Then it is sufficient to keep only the observable state component $\bm \gamma_{\mathbf a}$ in (\ref{e:Kalman_can}) as it fully represents the biased acceleration $\mathbf a$ hence the measurement $\mathbf a_{m_1}$. Furthermore, from the output equation of (\ref{e:subsystem}), we have $\mathbf a_1 = \mathbf C_{\boldsymbol{\gamma}_{\mathbf a_1}} {\bm \gamma_{\mathbf a_1}} = \mathbf a - \mathbf b_{\mathbf a_1} $. Hence, the measurement $\mathbf a_{m_2}$ can be fully represented by the biased acceleration $\mathbf a$ and the relative bias $\mathbf b_{\mathbf a} = \mathbf b_{\mathbf a_2} - \mathbf R \mathbf b_{\mathbf a_1}$. Finally dropping the subscript 1,2 for the sake of notation simplicity, the state equation (\ref{e:smodel2}) can be reduced to the following minimal realization.

\begin{equation}~\notag
\begin{array}{c}
\dot {\mathbf b}_{\mathbf a} = \mathbf w_{\mathbf b_{\mathbf a}}, \ \dot {\mathbf b}_{\bm \omega} = \mathbf w_{\mathbf b_{\bm \omega}}, \ \dot{\bm \omega} = \bm \tau, \ \dot {\mathbf c} = \mathbf 0, \  \dot{\mathbf R} = \mathbf 0 \\
\dot{\bm \gamma}_{\mathbf a}=\mathbf A_{\bm \gamma_{\mathbf a}}{\bm \gamma}_{\mathbf a}+\mathbf B_{\bm \gamma_{\mathbf a}}\mathbf w_{\bm \gamma_{\mathbf a}}, \dot{\bm \gamma}_{\bm \tau}=\mathbf A_{\bm \gamma_\tau}{\bm \gamma}_{\bm \tau}+\mathbf B_{\bm \gamma_{\bm \tau}}\mathbf w_{\bm \gamma_{\bm \tau}}
\end{array}
\end{equation}
where $\mathbf w_{\mathbf b_{\mathbf a}} = \mathbf w_{\mathbf b_{\mathbf a_2}} - \mathbf R \mathbf w_{\mathbf b_{\mathbf a_1}}$ is a relative noise seen in $\mathcal{B}$2 and the output equation is
\begin{equation}~\notag
\begin{aligned}
&\mathbf a \!=\! \mathbf C_{\bm \gamma_{\mathbf a}} \bm \gamma_{\mathbf a}, \quad  \bm \tau = \mathbf C_{\bm \tau_{\bm \tau}} \bm \gamma_{\bm \tau} \\
&\bm \omega_{m_1}\!=\!\bm \omega \!+\! \mathbf b_{\bm \omega} \!+\!\mathbf n_{\bm \omega}, \ \mathbf a_{m_1}\!=\!\mathbf a +\mathbf n_{\mathbf a},\\
&\mathbf a_{m_2} \!=\! \mathbf R\!\left(\mathbf a\! +\!\lfloor{\bm \omega}\rfloor^2\mathbf c \!+\!\lfloor{\bm \tau}\rfloor \mathbf c\right)\! + \!\mathbf b_{\mathbf a}\!+\!\mathbf n_{\mathbf a}
\end{aligned}
\end{equation}
\begin{lemma}~\label{lemma:obs_interIMU}
The nonlinear system with state equation
\begin{equation}~\label{e:smodel2_min_append}
\begin{array}{c}
\dot {\mathbf b}_{\mathbf a} = \mathbf 0, \ \dot {\mathbf b}_{\bm \omega} = \mathbf 0, \ \dot {\mathbf c} = \mathbf 0, \  \dot{\mathbf R} = \mathbf 0 \\
\dot{\bm \omega} = \bm \tau, \ \dot{\bm \gamma}_{\bm \tau}=\mathbf A_{\bm \gamma_{\bm \tau}}{\bm \gamma}_{\bm \tau}, \dot{\bm \gamma}_{\mathbf a}=\mathbf A_{\bm \gamma_{\mathbf a}}{\bm \gamma}_{\mathbf a}
\end{array}
\end{equation}
where $(\mathbf C_{\bm \gamma_{\mathbf a}}, \mathbf A_{\bm \gamma_{\mathbf a}})$ and $(\mathbf C_{\bm \gamma_{\bm \tau }}, \mathbf A_{\bm \gamma_{\bm \tau}})$ are both observable and output equation
\begin{equation}~\label{e:mmodel2_min_append}
\begin{aligned}
&\mathbf h_1 \!=\! \mathbf R\!\left(\mathbf a\! +\!\lfloor{\bm \omega}\rfloor^2\mathbf c \!+\!\lfloor{\bm \tau}\rfloor \mathbf c\right)\! + \!\mathbf b_{\mathbf a} \\
&\mathbf h_2 \!=\!\bm \omega \!+\! \mathbf b_{\bm \omega} , \\
&\mathbf h_3 \!=\!\mathbf a
\end{aligned} \nonumber 
\end{equation}
is locally weakly observable if $\mathbf c \neq \mathbf 0$ and the system initial states $\boldsymbol{\omega}^{(k)}$ and $\mathbf a^{(k)}, k \geq 0$ have sufficient excitation. 
\end{lemma}
{\it Proof.} See Appendix.~\ref{append:interIMU}.\\

Lemma~\ref{lemma:obs_interIMU} implies that the overall system observability is independent from the exact form of statistical models used for $\bm \tau$ and $\mathbf a$ as long as they are observable. This allows us to choose generic higher order statistical models that better characterize the motion without affecting the observability. Furthermore, Lemma~\ref{lemma:obs_interIMU} enables us to design an error state EKF similar to Section \ref{sec:estimation} to estimate the extrinsic parameters $\mathbf c, \mathbf R$ online. 

With the mathematical formulation and theoretical proof on its observability, we now proceed to the experiment verification. The sensor suite containing two IMUs is shaken in all 6 degree of freedom (DOF) to give the system sufficient excitation. To evaluate the effectiveness of the proposed statistical models, the output of the two statistical models $\left(\mathbf C_{{\bm \gamma}_{\mathbf a}}, \mathbf A_{\bm \gamma_{\mathbf a }} \right)$ and $\left(\mathbf C_{{\bm \gamma}_{\bm \tau}}, \mathbf A_{\bm \gamma_{\bm \tau }} \right)$, which are $\mathbf a$ and $\bm \tau$, respectively, are compared to the ``ground truth". Since $\mathbf a$ is the acceleration with bias of IMU1, we compare it to the ground truth obtained by filtering the accelerometer measurements of IMU1 via a non-causal low-pass filter aimed to attenuate the measurement noise. Similar, $\bm \tau$ means the angular acceleration of IMU1, we compare it to the ground truth obtained by numerically differentiating the gyro measurements of IMU1 and then filtering it via a non-causal low-pass filter. Fig \ref{fig:dimu_domeg} shows the comparison results. It is seen that our estimation closely track the results obtained by non-causal filter without a noticeable delay. There is a slight overshoot when the acceleration suddenly changes its direction. This is due to the non-exact model (i.e., 4-th order integrator) being used for the handheld motion. This slight overshoot is also in agreement with the frequency domain analysis in Section. \ref{case_study}. Fig. \ref{fig:c2_offs} shows the relative translation $\mathbf c$ and rotation $\mathbf R$ between the two IMUs. The ground truth is obtained by hand measurements for translation and CAD design for rotation. It can be seen that the estimated offset (both translation and rotation) quickly converges. The final converged error is 5 $mm$ in translation and 1 degree in rotation, which are below the precision of hand measurements. 
\begin{figure}[t]
	\centering
	\includegraphics[width=1\columnwidth]{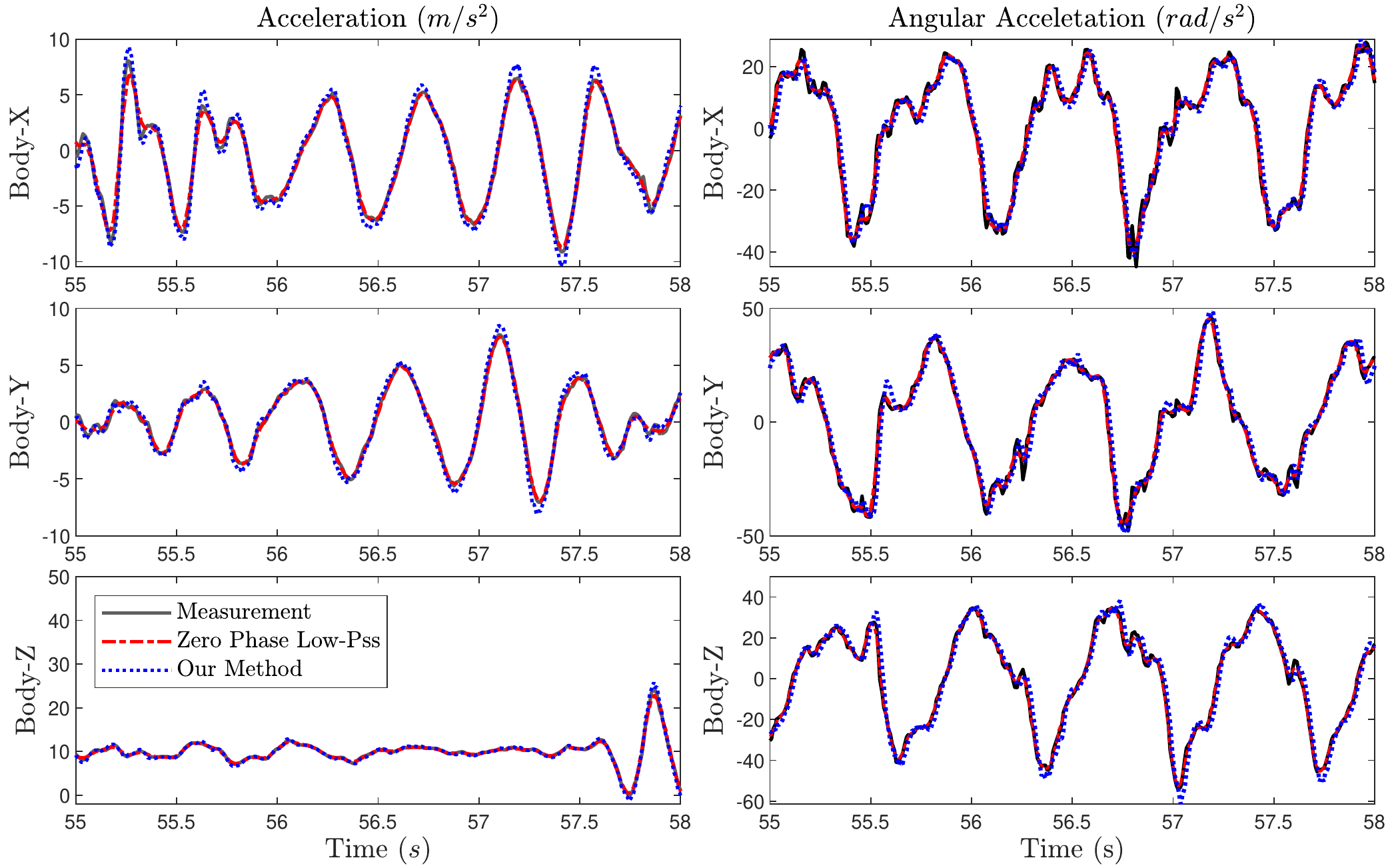}
	\caption{The estimated acceleration with bias and the estimated angular acceleration of IMU1.
		\label{fig:dimu_domeg}}
	\vspace{-0.3cm}
\end{figure}

\begin{figure}[t]
	\centering
	\includegraphics[width=1\columnwidth]{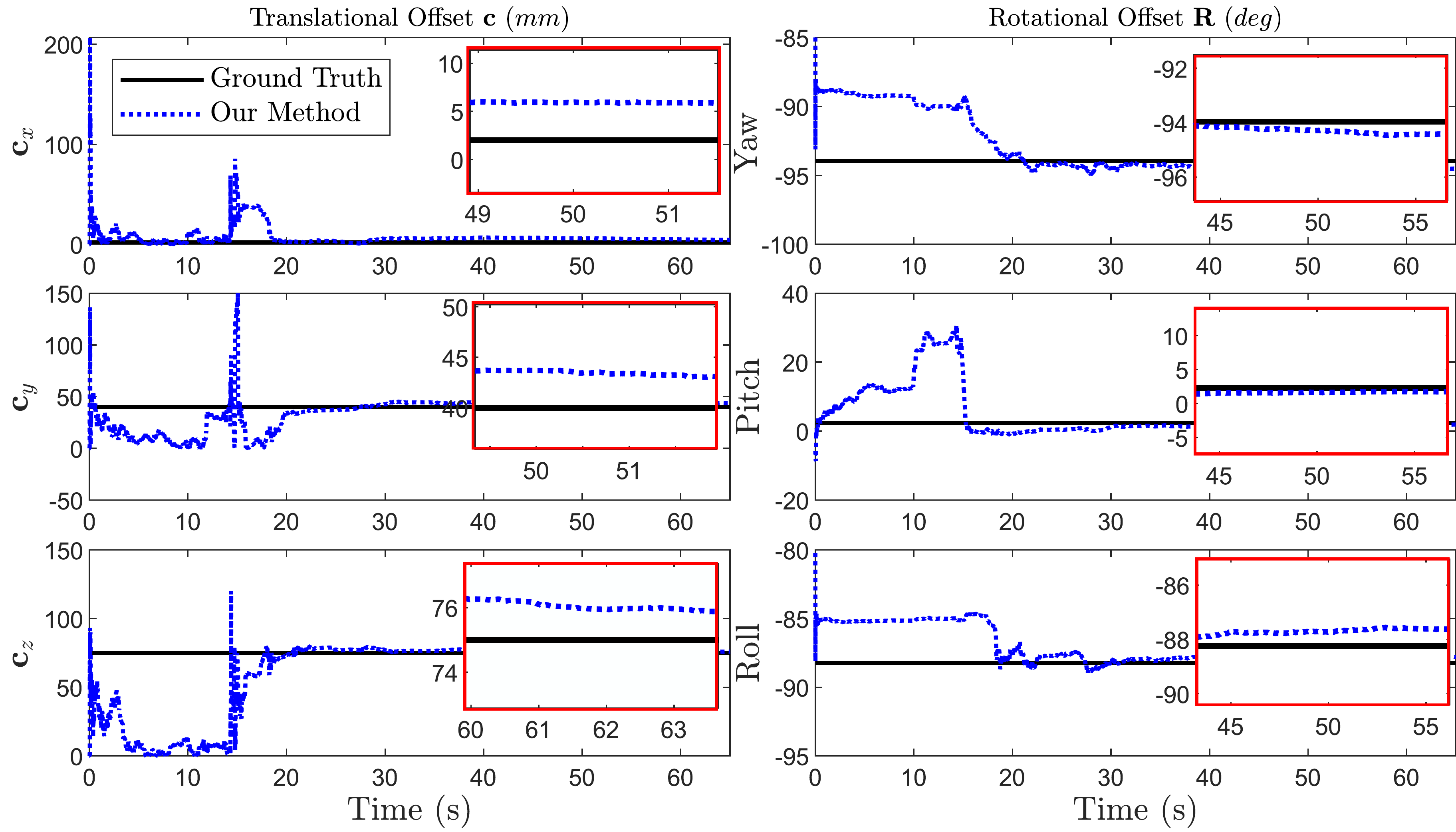}
	\caption{The estimation of translational offset and rotational offset.
		\label{fig:c2_offs}}
	\vspace{-0.3 cm}
\end{figure}

\subsection{Running Time}

An apparent drawback of the statistical model-based state estimation approach is that the system order is increased by the dimension of the statistical model. This will increase computation time of the EKF. Taking the application in Section \ref{sec:case1} as an example, if a 4-th order integrator (\ref{e:n_int_model}) is used for each channel of acceleration and angular velocity, the system dimension will increase by 24 orders. To verify the real-time running performance of the proposed approach, we implement the error-state EKF in C++. As a comparison, the running time of the normal EKF based on input formulation is also tested. The experiments run on a personal computer with processor of Intel Core i7-8700. As shown in Table. \ref{tab::c_runtime}, the state formulation EKF takes more than 3 times computation time than the input formulation, which is expected because of the increased system dimension. Nevertheless, the total running time is below $0.04 \ ms$, which is only a fraction of other modules (e.g., planning) and is sufficient for real-time running for most robotic applications. This can be further reduced by using the statistical model only when necessary. For example, when linear acceleration is required to be estimated while the angular velocity is not, then a hybrid input and state formulation could be used where the gyro measurements are formulated as input and the accelerometer measurements are formulated as outputs with an augmented statistical model. 

\begin{table}[htbp]
\footnotesize
\centering
\caption{Comparison of the Average Run Time}
\label{tab::c_runtime}
\tabcolsep 0.08in
\begin{tabular}{ccccc}
\hline
Methods   &Predict   &Update   & \tabincell{c}{Total}\\ \hline
\tabincell{c}{State formulation}   &0.0174 $ms$ & 0.0200 $ms$ & 0.0374 $ms$\\ \hline
\tabincell{c}{Input formulation} & 0.0050 $ms$ & 0.0056 $ms$ & 0.0106 $ms$\\ \hline
\end{tabular}
\end{table}

\section{Conclusion}
This paper proposed to modeling a robot's dynamics by a statistical model. Such a statistical model is robust and versatile, being applicable to different robotic systems without accurately modeling the robot's specific first principle dynamics. By augmenting the statistical model to the robot's kinematic equation, it enables to simultaneously estimate the dynamics state with low latency in addition to the kinematics states and extrinsic parameters. Extensive simulations and experiments on different robots platforms for different applications were conducted and the results supports the proposed methods. The augmented system requires more computations due to the increased state dimension, but the overall computation time is still sufficiently small for real-time implementations. The proposed method is suitable for applications where dynamic states (e.g., translational or angular accelerations) are needed in real-time. 

Besides the practical contribution, this paper also proposed new theoretical tools to ease the observability analysis of robotic systems operating on manifolds and in high dimension. A new paradigm which calculates Lie derivatives and partial differentiation directly on the respective manifolds is developed. This new paradigm is simpler, more natural and scalable to systems of high dimensions. Furthermore, a novel {\it thin} set concept was introduced to characterize the unobservable subspace of the system states. These lay theoretical foundations for observability analysis for robotic systems of high dimension.

\appendix
\begin{theorem}~\label{def:thin}
Let $p(\mathbf x)$ be a multivariate polynomial of $\mathbf x = \begin{bmatrix}
x_1, \cdots, x_n
\end{bmatrix}^T \in \mathbb{R}^n$. If $p(\mathbf x)$ is not constant zero, then
$$U = \{ \mathbf x \in \mathbb{R}^n | p(\mathbf x) = 0 \}$$
is a thin set in $\mathbb R^n$ by Definition \ref{thin}.
\end{theorem}

\begin{proof}
Assume $U$ is not thin, then there must exist an open  subset $V \subset \mathbb R^n$ contained in $U$ such that
$$p(\mathbf x) = 0, \forall \mathbf x \in V.$$

Since this equation holds on the entire open subset $V$ on which $p(\mathbf x)$ is infinitely differentiable, we have 
$$\frac{\partial^m p(\mathbf x)}{\partial t_m \cdots \partial t_1}  = 0;\  m \geq 0,  \ t_i \in \{x_1, \cdots, x_n \}, \ \forall {\mathbf x} \in V.$$

Since $p(\mathbf x)$ is a multivariate polynomial, it can be expressed as
$$p(\mathbf x) = \sum_{i=1}^{m} a_i \left( \prod_{j=1}^{n} x_j^{k_{ij}} \right); \ k_{ij} \in \{0, 1, \cdots  \} $$

Substituting $p(\mathbf x)$ into the above partial differentiation and letting $m_i = \sum_{j=1}^{n} k_{ij}$ leads to
$$\frac{\partial^{m_i} p(\mathbf x)}{\partial x_n^{k_{in}} \cdots \partial x_1^{k_{i1}}} = a_i = 0; \ i = 1, \cdots, m$$
and $p(\mathbf x) \equiv 0$, which contradicts with $p(\mathbf x)$ not being constant zero. 
\end{proof}

\begin{lemma}~\label{lemma_Omega_dot}
Let $\bm \Omega=\lfloor {\bm \omega}\rfloor^2 +\lfloor{\bm \omega^{(1)}}\rfloor$. If $i \geq 2$ is the first order of derivative such that $\bm \omega^{(i)} \neq \mathbf 0$ and $j > i$ is the first order of derivative such that $\bm \omega^{(j)} \nparallel \bm \omega^{(i)}$, then the matrix
$$\bm \Theta = \begin{bmatrix}
\bm \Omega^{(i-1)} \\
\bm \Omega^{(j-1)}
\end{bmatrix}$$ has full column rank. 
\end{lemma}

\begin{proof}
It is shown that $\bm \Omega^{(m)} = \sum_{k=0}^{m}\text C_{k}^{m}\lfloor \bm \omega^{(k)}\rfloor\lfloor \bm \omega^{(m-k)}\rfloor + \lfloor \bm \omega^{(m+1)}\rfloor$. And as indicated in the description of this Lemma, 
\begin{equation}\notag
\begin{array}{lc}
\left\{\begin{array}{ll}
\bm \omega^{(k)} = \mathbf 0& k \in \{0,1,\cdots,i-1\}\\
\bm \omega^{(i)} \neq \mathbf 0&\\
\bm \omega^{(k)} = \mu_k\bm \omega^{(i)}& k\in\{i+1, \cdots, j-1\}\\
\bm \omega^{(j)} \nparallel \bm \omega^{(i)}
\end{array}\right.
\end{array}
\end{equation}
Then
\begin{equation}\notag
\begin{aligned}
\bm \Omega^{(i-1)} &= \lfloor \bm \omega^{(i)}\rfloor\\
\bm \Omega^{(j-1)} &= \left\{\begin{array}{ll}
\Lambda^{i}_{j-1}\lfloor \bm \omega^{(i)}\rfloor^2+\lfloor \bm \omega^{(j)}\rfloor&{2i \leq j-1}\\

 \lfloor \bm \omega^{(j)}\rfloor&\text{otherwise}
\end{array}\right.
\end{aligned}
\end{equation}
where $\Lambda^{i}_{j-1}\!\!=\!\! \sum_{k=i}^{j-1-i}\text C_{k}^{j-1}\mu_k\mu_{j-1-k}$.

Assuming $\bm \Theta$ does not have full column rank, then $\exists\; \mathbf v \neq \mathbf 0$, s.t. $\bm \Theta\cdot \mathbf v = \mathbf 0$. Then we have $\mathbf v \parallel \bm \omega^{(i)}$ from $\bm \Omega^{(i-1)} \mathbf v = \mathbf0$ and $\mathbf v \parallel \bm \omega^{(j)}$ from $\bm \Omega^{(j-1)} \mathbf v = \mathbf0$. This is, however, contradictory to the condition that $\bm \omega^{(j)}\nparallel\bm \omega^{(i)}$. Thus $\bm \Theta$ has full column rank.
\end{proof}

\subsection{Proof of Proposition 1}\label{appix:proposition1}

\begin{proof}
Let $\mathbf b = \mathbf a+\bm \Omega \mathbf c$, the acceleration at the position reference point and seen in the body frame, and $\mathbf g_m = \mathbf R^T \mathbf h_1^{(m+2)}$, the $(m+2)$-th time derivative of the position reference point and seen in the body frame, we can derive, by induction, that

\begin{equation}\notag
    \begin{aligned}
    \mathbf g_0 &= \mathbf b \\
    \mathbf g_m &= \mathbf b^{(m)} + \sum_{k=0}^{m-1} \bm \alpha_k^m \mathbf g_k; \ m \geq 1
    \end{aligned}
\end{equation}
where $\bm\alpha_k^{m}$ is defined recursively as below:
\begin{equation}\notag
\begin{aligned}
\bm\alpha_0^1 &= \lfloor \bm \omega \rfloor \\
\bm\alpha_k^{m+1}\!&=\! \left\{\!\begin{array}{ll}
\dot{\bm\alpha}_k^{m}\! -\! \bm\alpha_k^{m}\lfloor \bm \omega\rfloor\quad  & {k=0}\\
\dot{\bm\alpha}_k^{m}\! -\! \bm\alpha_k^{m}\lfloor \bm \omega\rfloor
\!+\! \bm\alpha_{k-1}^{m} &{1 \leq k \leq m-1} \\
\lfloor \bm \omega\rfloor\! +\! \bm\alpha_{k-1}^{m}& k = m
\end{array}\right.
\end{aligned}
\end{equation}

It is seen that $\bm \alpha^m_k$ are polynomials of $\bm \omega^{(k)}$ and $\mathbf g_m$ are polynomials of $\bm \omega^{(k)}$, $\mathbf a^{(k)}$, and $\mathbf c$. Hence,
\begin{equation}\notag
    \begin{aligned}
    \nabla_{\bm \theta} \mathbf g_m = \mathbf 0; \nabla_{\mathbf c} \mathbf g_m = \bm \Omega^{(m)} + \sum_{k=0}^{m-1} \bm \alpha_k^m( \nabla_{\mathbf c} \mathbf g_k) \\
    \nabla_{\mathbf a} \mathbf g_m = 
    \! \left\{\!\begin{array}{ll}
\mathbf I_3  & {m=0}\\
\sum_{k=0}^{m-1} \bm \alpha_k^m (\nabla_{\mathbf a} \mathbf g_k) &{ m \geq 1}
\end{array}\right.
    \end{aligned}
\end{equation}

By definition,
\begin{equation}\notag
\begin{aligned}
    & \left( \mathcal L_{\mathbf f_0}^1\mathbf h_1 \right) ({\mathbf x}) =  \dot{\mathbf{h}}_1(\mathbf{x}), \text{subject to} \ \dot{\mathbf{x}}= \mathbf f_0(\mathbf{x}) 
\end{aligned}
\end{equation}
where $\mathbf f_0 (\mathbf x)$ abbreviates (\ref{e:state_formula_model_detail}), we have
\begin{equation}\notag
\begin{aligned}
    \left( \mathcal L^{m+2}_{\mathbf f_0 \cdots \mathbf f_0}\mathbf h_1 \right) ({\mathbf x}) &=  {\mathbf{h}}^{(m+2)}_1(\mathbf{x}), \text{subject to} \ \dot{\mathbf{x}}= \mathbf f_0(\mathbf{x})  \\
    &= \mathbf R \mathbf g_m
\end{aligned}
\end{equation}

Then, $\mathcal{O}_{S_1}$ in (\ref{e:O_state_E}) is detailed as:
\begin{equation}\notag
\begin{aligned}
\mathcal{O}_{S_1}&=\begin{bmatrix}
\lfloor \bm \beta\rfloor & \mathbf 0 & \mathbf 0\\
{}_{n_1}^2 \mathcal G_{\bm \theta}^{\mathbf h_1} &{}_{n_1}^2 \mathcal G_{\mathbf c}^{\mathbf h_1} & {}_{n_1}^2\mathcal G_{\mathbf a}^{\mathbf h_1}
   \end{bmatrix} \\
{}_{n_1}^2 \mathcal G_{\bm \theta}^{\mathbf h_1}   & = 
\begin{bmatrix}
   \vdots \\
   \nabla_{\bm \theta} \left( \mathcal L^{m+2}_{\mathbf f_0 \cdots \mathbf f_0}\mathbf h_1 \right) \\
   \vdots
   \end{bmatrix} = 
\begin{bmatrix}
   \vdots \\
   - \mathbf R \lfloor \mathbf g_m \rfloor \\
   \vdots
   \end{bmatrix} \\
 {}_{n_1}^2 \mathcal G_{\mathbf c}^{\mathbf h_1}   & = 
\begin{bmatrix}
   \vdots \\
   \nabla_{\mathbf c} \left( \mathcal L^{m+2}_{\mathbf f_0 \cdots \mathbf f_0}\mathbf h_1 \right) \\
   \vdots
   \end{bmatrix} = 
\begin{bmatrix}
   \vdots \\
   \mathbf R \nabla_{\mathbf c} \mathbf g_m  \\
   \vdots
   \end{bmatrix} \\
    {}_{n_1}^2 \mathcal G_{\mathbf a}^{\mathbf h_1}   & = 
\begin{bmatrix}
   \vdots \\
   \nabla_{\mathbf a} \left( \mathcal L^{m+2}_{\mathbf f_0 \cdots \mathbf f_0}\mathbf h_1 \right) \\
   \vdots
   \end{bmatrix} = 
\begin{bmatrix}
   \vdots \\
   \mathbf R \nabla_{\mathbf a} \mathbf g_m  \\
   \vdots
   \end{bmatrix}
\end{aligned}
\end{equation}
where $m = 0,1, \cdots, n_1 -2$. Since the rotation matrix $\mathbf R$ does not affect the rank, we have
\begin{equation}\notag
\begin{aligned}
\mathcal{O}_{S_1}&\sim \begin{bmatrix}
   \lfloor \bm \beta\rfloor & \mathbf 0 & \mathbf 0\\
   -\lfloor{\mathbf g}_{0}\rfloor &\bm \Omega^{(0)} &\mathbf I_3\\
   \vdots &\vdots &\vdots\\
    -\lfloor{\mathbf g}_{m}\rfloor &\bm \Omega^{(m)} +\bm \Pi_{\mathbf c}^{m}  &\bm \Pi_{\mathbf a}^{m}\\
   \vdots &\vdots &\vdots
   \end{bmatrix}
\end{aligned}
\end{equation}
where $m = 1, 2, \cdots, n_1 -2$, $\bm\Pi_{\mathbf c}^{m} = \sum_{k=0}^{m-1}\bm \alpha_{k}^{m}\left(\nabla_{\mathbf c}{ \mathbf g_k}\right)$ and $\bm\Pi_{\mathbf a}^{m} = \sum_{k=0}^{m-1}\bm \alpha_{k}^{m}\left(\nabla_{\mathbf a}{ \mathbf g_k}\right)$ are linearly dependent on preceding rows and can therefore be eliminated by Gaussian elimination. Hence,
\begin{equation}\notag
\begin{aligned}
\mathcal{O}_{S_1}
   &\sim \begin{bmatrix}
   \lfloor \bm \beta\rfloor & \mathbf 0 & \mathbf 0\\
   -\lfloor{\mathbf g}_{0}\rfloor &\bm \Omega^{(0)} & \mathbf I_3\\
   \vdots &\vdots &\vdots\\
    -\lfloor{\mathbf g}_{m}\rfloor+\sum_{k=0}^{m-1}\bm\alpha_k^m\lfloor{\mathbf g}_{k}\rfloor &\bm \Omega^{(m)}&\mathbf 0\\
   \vdots &\vdots &\vdots
   \end{bmatrix}
\end{aligned}
\end{equation}

Since column 3 only contains an identity matrix, the other terms in row 2 can be eliminated to zero by column transformation. Therefore, the $\mathcal O_{S_1}$ has full column rank if and only if the following sub-matrix $\mathcal O^E_{S_1}$ have full column rank.

\begin{equation}\notag
\begin{aligned}
\mathcal O^E_{S_1}
\!\!=\!\! \begin{bmatrix}
   \lfloor \bm \beta\rfloor & \mathbf 0 & \mathbf 0\\
   \vdots &\vdots &\vdots\\
    -\lfloor{\mathbf g}_{m}\rfloor+\sum_{k=0}^{m-1}\bm\alpha_k^m\lfloor{\mathbf g}_{k}\rfloor &\bm \Omega^{(m)}&\mathbf 0\\
   \vdots &\vdots &\vdots
   \end{bmatrix}; \ m \geq 1
\end{aligned}
\end{equation}

Notice that elements of $\mathcal O^E_{S_1}$ are polynomial functions of $\bm \omega^{(k)}, \mathbf a^{(k)}$, and $\mathbf c$ for a given $\bm \beta = \mathbf R^T\mathbf e_1$, and that $\bm \omega^{(k)}$ and $\mathbf a^{(k)}$ are further linearly dependent on $\bm \gamma_{ \bm \omega}$ and $\bm \gamma_{\mathbf a}$, respectively, thus elements of $\mathcal O^E_{S_1}$ are polynomial functions of the following subset of states:
$$\mathbf x_{\text{sub}} = \begin{bmatrix}{}
\bm \gamma_{\bm \omega} \\
\bm \gamma_{\mathbf a} \\
\mathbf c
\end{bmatrix} \in \mathbb{R}^{n_{\omega} +n_{ a} +  3}$$

Denote
$$\text{Sub}(\mathcal O^E_{S_1}) $$
the set of all square matrix whose rows are drawn from $\mathcal O^E_{S_1}$. Since elements of $\mathcal O^E_{S_1}$ are polynomials of $\mathbf x_{\text{sub}}$, the determinant $\det(\mathbf S), \forall \mathbf S \in \text{Sub}(\mathcal O^E_{S_1})$ is also a polynomial function of $\mathbf x_{\text{sub}}$. Then if $\mathcal O^E_{S_1}$ is rank-deficient, there must have
$$\det(\mathbf S)(\mathbf x_{\text{sub}}) = 0, \forall \mathbf S \in \text{Sub}(\mathcal O^E_{S_1}).$$
which forms infinite number of equality constraints. For each constraint, $\det(\mathbf S)(\mathbf x_{\text{sub}})$ is a polynomial of $\mathbf x_{\text{sub}}$ and its roots are either the whole space of $\mathbb{R}^{n_{\omega} +n_{ a} + 3}$ if $\det(\mathbf S)(\mathbf x_{\text{sub}})\equiv 0$ or otherwise lie on a {\it thin} subset (see Theorem \ref{def:thin}). 

Now we show that $\mathcal O^E_{S_1}$ is not constantly rank-deficient by giving a particular sufficient condition as follows.
\begin{enumerate}[label=(\roman*)]
    \item $i \geq 3$ is the first order of derivative such that $\bm \omega^{(i)} \neq \mathbf 0$ and $j > i $ is the first order of derivative such that $\bm \omega^{(j)} \nparallel \bm \omega^{(i)}$  
    \item $\exists k \in \{ 1, \cdots, i - 2\}, \ s.t. \ \mathbf a^{(k)}\nparallel \bm \beta$
\end{enumerate}

From Lemma~\ref{lemma_Omega_dot}, condition (i) ensures the matrix 
$$\bm \Theta = \begin{bmatrix}
\bm \Omega^{(i-1)} \\
\bm \Omega^{(j-1)}
\end{bmatrix}$$
has full column rank and that $\bm \Omega^{(k)} = \mathbf 0$. Then the matrix $\mathcal{O}_{S_1}^E$ can be reduced to
\begin{equation}
    \mathcal{O}_{S_1}^E\!\sim \!\! \left[\!\! \begin{array}{c; {1pt/1pt}c; {1pt/1pt}c}
   \bm \Psi & \mathbf 0\\
    \vdots & \vdots \\\!
    \bullet \! & \bm \Theta \\
    \vdots & \vdots  
    \end{array}\!\! \right]  \nonumber
\end{equation}
where $\bm \Psi\!\! =\!\! \begin{bmatrix}
\lfloor \bm \beta \rfloor\\
-\lfloor\mathbf a^{(k)}\rfloor
\end{bmatrix}$.

Then, condition (ii) ensures $\bm \Psi$ has full column rank and hence the matrix $\mathcal O^E_{S_1}$ has full column rank. This means $\exists \mathbf S \in \text{Sub}(\mathcal O^E_{S_1})$ and $\bar{\mathbf x}_{\text{sub}} \in \mathbb{R}^{n_{\omega} +n_{ a} + 3}$ such that $\det(\mathbf S)(\bar{\mathbf x}_{\text{sub}}) \neq 0$. Moreover, since $\det(\mathbf S)(\mathbf x_{\text{sub}})$ is a continuous function of $\mathbf x_{\text{sub}}$, so if $\det(\mathbf S)({\mathbf x}_{\text{sub}}) \neq 0$ at point $\bar{\mathbf x}_{\text{sub}}$, it must be so in a neighbour of that point, meaning that there are infinite number of submatrix $\mathbf S \in \mathcal O_{\text{sub}}$ whose determinant is not constant zero and hence the solution of each $\det(\mathbf S) = 0$ lies on a {\it thin} subset of $\mathbb{R}^{n_{\omega} +n_{ a} + 3}$. As a result, the states $\mathbf x_{\text{sub}}$ satisfying $\mathcal O_{\text{sub}}$ being rank-deficient is the joint of the infinite number of {\it thin} subset and the resultant solution space is even thinner, i.e., a {\it thin} subset in $\mathbb{R}^{n_{\omega} +n_{ a} + 3}$. 
\end{proof}

\subsection{Proof of Proposition 2}\label{appix:proposition2}

\begin{proof}
Recall that $\mathbf h_2 = \mathbf R^T \mathbf e_1 \triangleq \bm \beta$, the matrix $\mathcal O_{S_2} = {}_{n_2}^1\mathcal G_{\bm \omega}\mathbf h_2=\begin{bmatrix}\nabla_{\bm \omega}\bm \beta^{(1)} & \hdots & \nabla_{\bm \omega}\bm \beta^{(n_2)} \end{bmatrix}^T$. Since $\dot{\bm \beta}=\lfloor \bm \omega\rfloor^T \bm \beta=\lfloor \bm \beta\rfloor \bm \omega$, each term $\nabla_{\bm \omega}\bm \beta^{(m)}, m\in \{1, \cdots, n_2 \}$ can be computed as:
\begin{equation}\nonumber
\begin{aligned}
\!\!\nabla_{\bm \omega}\bm \beta^{(m)}\!&=\!\frac{\partial}{\partial \bm \omega}\frac{d^{m}\bm \beta}{dt^{m}} 
=\!\frac{\partial}{\partial \bm \omega}\frac{d^{m-1}(\dot{\bm \beta})}{dt^{m-1}} =\!\frac{\partial}{\partial \bm \omega}\frac{d^{m-1}(\lfloor \bm \omega \rfloor^T \bm \beta)}{dt^{m-1}}
\\
&=\!\frac{\partial}{\partial \bm \omega}\!\sum _{i=0}^{m-1}\text C_{i}^{m-1} \lfloor \bm \omega^{(i)} \rfloor^T\! \bm \beta^{(m-1 - i)}
\\
&=\!\sum _{i=1}^{m}\text C_{i-1}^{m-1} \lfloor \bm \beta^{(m-i)}\rfloor \frac{\partial}{\partial \bm \omega}\!\bm \omega^{(i-1)} \\&\quad\quad + \!\sum _{i=1}^{m}\text C_{i-1}^{m-1} \! \lfloor \bm \omega^{(i-1)} \rfloor^T\!\frac{\partial}{\partial \bm \omega}\!\bm \beta^{(m-i)}
\\
&=\!\lfloor \bm \beta^{(m-1)} \rfloor\!+\!\sum _{i=1}^{m}\!\underbrace{\text C_{i-1}^{m-1} \lfloor \bm \omega^{(i-1)} \rfloor^T\! \nabla_{\bm \omega}\bm \beta^{(m-i)}}_{\mathbf K_i}
\end{aligned}
\end{equation}

Notice that terms $\mathbf K_i, 0 \leq i \leq m -1$ can be eliminated by preceding rows using the Gaussian elimination and $\mathbf K_m = \mathbf 0$, the only left term is $\lfloor \bm \beta^{(m-1)} \rfloor$. Therefore
\begin{equation}\nonumber
    \mathcal O_{S_2} \sim \begin{bmatrix} \lfloor \bm \beta\rfloor \\ \hdots \\ \lfloor \bm \beta^{(n_2-1)}\rfloor \end{bmatrix}
\end{equation}

Since $\mathcal O_{S_2}$ is rank-deficient, then $\exists \mathbf v \neq \mathbf 0$ such that $\mathcal O_{S_2} \mathbf v=\mathbf 0$, which implies $\lfloor \bm \beta^{(m)} \rfloor \mathbf v=\mathbf 0, \forall m \geq 0 $. In particular, let $m=0$, then $\bm \beta \parallel \mathbf v$; let $m=1, \lfloor \bm \beta^{(1)} \rfloor \mathbf v = \mathbf 0 \implies \lfloor \lfloor \bm \omega \rfloor^T \bm \beta \rfloor \bm \beta = \mathbf 0 \implies \bm \omega \parallel \bm \beta$. Now assume that $\bm \omega^{(m)} \parallel \bm \beta$ is true for $0 \leq m \leq k - 1$, then
\begin{equation}\nonumber
\begin{aligned}
    &\lfloor\bm \beta^{(k+1)}\rfloor\bm \beta = \mathbf 0 \Rightarrow \lfloor \frac{d^k \dot{\bm \beta}}{dt^k}  \rfloor \bm \beta = \mathbf 0 \\
    &\Rightarrow\sum_{i=0}^{k}\lfloor\lfloor\bm \omega^{(i)}\rfloor^T\bm \beta^{(k-i)}\rfloor\bm \beta = \mathbf 0\\
    &\Rightarrow \sum_{i=0}^{k-1}\left(\lfloor\lfloor\bm \omega^{(i)}\rfloor^T\bm \beta^{(k-i)}\rfloor + \lfloor\lfloor\bm \omega^{(k)}\rfloor^T\bm \beta\rfloor\right)\bm \beta =\mathbf 0\\
    &\Rightarrow\!\!-\!\! \sum_{i=0}^{k-1} \underbrace{\!\left(\!\lfloor\!\bm \omega^{(i)}\!\rfloor\!\lfloor\bm\beta^{(k-i)}\rfloor\!\! -\!\! \lfloor\bm \beta^{(k-i)}\rfloor\!\lfloor\!\bm\omega^{(i)}\!\rfloor\!\right)\!\!\bm \beta\!\!}_{= \mathbf 0} +\! \lfloor\lfloor\bm \omega^{(k)}\rfloor^T\bm \beta\rfloor\bm\beta\! \!=\!\!\mathbf 0\\
    &\Rightarrow \lfloor\lfloor\bm \omega^{(k)}\rfloor^T\bm \beta\rfloor\bm\beta =\mathbf 0\\
    &\Rightarrow \bm \omega^{(k)} \parallel \bm \beta
    \end{aligned}
\end{equation}
Therefore, if $\mathcal O_{S_2}$ is rank-deficient, then $\bm \omega^{(k)} \parallel \bm \beta, \ \forall k \geq 0$. 
\end{proof}
\subsection{Proof of Lemma \ref{lemma:obs_interIMU}}~\label{append:interIMU}
\begin{proof}
Define the state vector $\mathbf x$
\begin{equation}
    \mathbf x^T = \begin{bmatrix}
    \mathbf b_{\mathbf a}^T, \mathbf b_{\bm \omega}^T,  \mathbf c^T, \mathbf R^T, \bm \omega^T, \bm \gamma_{\bm \tau}^T, \bm \gamma_{\mathbf a}^T
    \end{bmatrix} \nonumber
\end{equation}
and rewrite the state equation (\ref{e:smodel2_min_append}) as $\dot{\mathbf x}= {\mathbf f}_0\left(\mathbf x\right)$, then the observability matrix is computed as
\begin{equation}
\begin{aligned}
    \mathcal{O}\!=\!\left[\!\begin{array}{c} \!\multirow{2}*{$^0_{n}\mathcal G^{\mathbf h_1}_{\mathbf x}$}\! \\
    \\ [1mm]\hdashline[1pt/1pt] \\
    [-1.5mm]\!{\!^0_{n}\mathcal G^{\mathbf h_2}_{\mathbf x}\!}\!\\ 
    [-1.5mm] \\
    [1mm]\hdashline[1pt/1pt]\\
    [-3mm] \!^0_{n}\mathcal G^{\mathbf h_3}_{\mathbf x}\!\end{array}\!\right]\!\!=\!\!\left[\!\begin{array}{ccccccc}
    \mathbf I_{3}\! & \mathbf 0 & \bullet & \bullet & \bullet &\bullet & \mathbf 0\\
    \mathbf 0 & \mathbf 0 & \!\!^1_{n}\mathcal G^{\mathbf h_1}_{\mathbf c}\!\! & \!\!^1_{n}\mathcal G^{\mathbf h_1}_{\bm \theta}\!\! & \!\!^1_{n}\mathcal G^{\mathbf h_1}_{\bm \omega}\!\! & \bullet & \bullet\\
    [1mm]\hdashline[1pt/1pt]  &&&&\\
    [-3mm] \mathbf 0 & \mathbf I_{3}\!\! & \mathbf 0 & \mathbf 0 & \mathbf I_3 & {\mathbf 0} & \mathbf 0 \\
    \mathbf 0 & \mathbf 0 & \mathbf 0 & \mathbf 0 & \mathbf 0 & \!\mathcal O_{\bm \gamma_{\bm \tau}}\!\! & \mathbf 0 \\
    [1mm]\hdashline[1pt/1pt]&&&& \\
    [-3mm] {\mathbf 0} & {\mathbf 0} & {\mathbf 0} & \mathbf 0 & \mathbf 0 & \mathbf 0 &\!\! \mathcal O_{\bm \gamma_{\mathbf a}}\!\!
        \end{array}\!\right]
   \end{aligned} \nonumber
\end{equation}
which has full column rank if and only if the following submatrix
\begin{equation}
    \mathcal{O}_{\text{sub}} = \left[ \begin{array}{c; {1pt/1pt}c; {1pt/1pt}c}
    ^1_{n}\mathcal G^{\mathbf h_1}_{\mathbf c} & ^1_{n}\mathcal G^{\mathbf h_1}_{\bm \theta} & ^1_{n}\mathcal G^{\mathbf h_1}_{\bm \omega}
    \end{array} \right] \nonumber
\end{equation}
has full column rank.

Let $\bm \Omega=\lfloor {\bm \omega}\rfloor^2 +\lfloor{\bm \omega^{(1)}}\rfloor$, then

\begin{equation}
    \mathcal{O}_{\text{sub}}\!\! \sim \!\! \left[\!\! \begin{array}{c; {1pt/1pt}c; {1pt/1pt}c}
    \bm \Omega^{(1)} & -\lfloor \mathbf a^{(1)} \!+\! \bm \Omega^{(1)} \mathbf c \rfloor\! & \!\!- \lfloor \lfloor \bm \omega^{(1)}\rfloor \mathbf c \rfloor \!-\! \lfloor \bm \omega^{(1)}\rfloor \lfloor \mathbf c \rfloor\!\! \\
    \vdots & \vdots & \vdots \\
    \bm \Omega^{(n)}\! & \!-\lfloor \mathbf a^{(n)} + \bm \Omega^{(n)} \mathbf c \rfloor\!\! & \!- \lfloor \lfloor \bm \omega^{(n)}\rfloor \mathbf c \rfloor \!- \!\lfloor \bm \omega^{(n)}\rfloor \lfloor \mathbf c \rfloor \!\!
    \end{array}\!\! \right] \nonumber
\end{equation}
where the sign $\sim$ indicates equivalent rank. Notice that 1) the rank of $\mathcal O_{\text{sub}}$ depends on $\bm \omega^{(k)},\mathbf a^{(k)}$ (hence the extended states $\bm \gamma_{\bm \tau}, \bm \gamma_{\mathbf a}$, and state $\bm\omega$), and $\mathbf c$ only and is independent from the bias $\mathbf b_{\mathbf a}, \mathbf b_{\bm \omega}$, and relative rotation $\mathbf R$; and moreover 2) elements of $\mathcal O_{\text{sub}}$ are polynomial functions of $\bm \omega^{(k)}, k \geq 0, \mathbf a^{(k)}, k \geq 1$ and $\mathbf c$. Since $\bm \omega^{(k)}, k\geq 1$ are further linearly dependent on $\bm \gamma_{\bm \tau}$ and $ \mathbf a^{(k)}$ are further linearly dependent on $\bm \gamma_{\mathbf a}$,  elements of $\mathcal O_{\text{sub}}$ are polynomial functions of the following subset of states
$$\mathbf x_{\text{sub}} = \begin{bmatrix}{}
\bm \gamma_{\bm \tau} \\
\bm \gamma_{\mathbf a} \\
\bm \omega\\
\mathbf c
\end{bmatrix} \in \mathbb{R}^{n_{\tau} +n_{ a} +  6}$$

Denote
$$\text{Sub}(\mathcal O_{\text{sub}}) $$
the set of all square matrix whose rows are drawn from $\mathcal O_{\text{sub}}$. Since elements of $\mathcal O_{\text{sub}}$ are polynomials of $\mathbf x_{\text{sub}}$, the determinant $\det(\mathbf S), \forall \mathbf S \in \text{Sub}(\mathcal O_{\text{sub}})$ is also a polynomial function of $\mathbf x_{\text{sub}}$. Then if $\mathcal O_{\text{sub}}$ is rank-deficient, there must have
$$\det(\mathbf S)(\mathbf x_{\text{sub}}) = 0, \forall \mathbf S \in \text{Sub}(\mathcal O_{\text{sub}}).$$
which forms infinite number of equality constraints for $\mathbf x_{\text{sub}}$. For each constraint, $\det(\mathbf S)(\mathbf x_{\text{sub}})$ is a polynomial of $\mathbf x_{\text{sub}}$ and its roots are either the whole space of $\mathbb{R}^{n_{\tau} +n_{ a} + 6}$ if $\det(\mathbf S)(\mathbf x_{\text{sub}})\equiv 0$ or otherwise lie on a {\it thin} subset (see Theorem \ref{def:thin}). 

Now we show that $\mathcal O_{\text{sub}}$ is not constant rank-deficient if $\mathbf c \neq \mathbf 0$ by giving a particular sufficient condition as below.

\begin{enumerate}[label=(\roman*)]
    \item $i \geq 2$ is the first order of derivative such that $\bm \omega^{(i)} \neq \mathbf 0$ and $j > i + 1$ is the first order of derivative such that $\bm \omega^{(j)} \nparallel \bm \omega^{(i)}$, $\bm \omega^{(j-1)} = \mathbf 0$, and $\mathbf a^{(i-1)} \!\!+\! \bm \Omega^{(i-1)} \mathbf c = \mathbf a^{(j-1)} \!\!+\! \bm \Omega^{(j-1)} \mathbf c = \mathbf 0$. 
    \item $\exists k, l \geq 1 \ s.t.,  \ \bm \omega^{(k)} = \bm \omega^{(l)} = \mathbf 0$, $\mathbf a^{(k)} \!\!+\! \bm \Omega^{(k)} \mathbf c$ and $\mathbf a^{(l)} \!\!+\! \bm \Omega^{(l)} \mathbf c$ are not collinear, and $i-1, j-1, k, l$ are different from each other.
    \item $\exists m \geq j \ s.t.\ \bm \omega^{(m)} \nparallel \mathbf c$ and $m$ is different from $k, l, i - 1$ and $j-1$.
\end{enumerate}

From Lemma~\ref{lemma_Omega_dot} , condition (i) ensures the matrix
$$\bm \Theta = \begin{bmatrix}
\bm \Omega^{(i-1)} \\
\bm \Omega^{(j-1)}
\end{bmatrix}$$
has full column rank. Then, the matrix $\mathcal{O}_{\text{sub}}$ can be reduced to
\begin{equation}
    \mathcal{O}_{\text{sub}}\!\sim \!\! \left[\!\! \begin{array}{c; {1pt/1pt}c; {1pt/1pt}c}
    \vdots & \vdots & \vdots \\\!
    \bm \Theta \! & \mathbf 0 \! & \mathbf 0 \\
    \bullet \! & \bm \Phi \!\! & \mathbf 0 \\
    \bullet \! & \bullet \! & \!\! \bm \Psi \!\!  \\
    \vdots & \vdots & \vdots 
    \end{array}\!\! \right]  \nonumber
\end{equation}
where 
$$
\begin{aligned}
&\bm \Phi = \begin{bmatrix}
\!-\lfloor \mathbf a^{(k)} \!\!+\! \bm \Omega^{(k)} \mathbf c \rfloor \\
\!-\lfloor \mathbf a^{(l)} \!\!+\! \bm \Omega^{(l)} \mathbf c \rfloor 
\end{bmatrix} \\
&\bm \Psi = - \lfloor \lfloor \bm \omega^{(m)}\rfloor \mathbf c \rfloor \!\!-\!\!\lfloor \bm \omega^{(m)}\rfloor \lfloor \mathbf c \rfloor
\end{aligned}$$

Then, condition (ii) ensures $\bm \Phi$ has full column rank and condition (iii) ensures $\bm \Psi$ has full rank and hence the matrix $\mathcal O_{\text{sub}}$ has full column rank. This means $\exists \mathbf S \in \text{Sub}(\mathcal O_{\text{sub}})$ and $\bar{\mathbf x}_{\text{sub}} \in \mathbb{R}^{n_{\tau} +n_{ a} + 6}$ such that $\det(\mathbf S)(\bar{\mathbf x}_{\text{sub}}) \neq 0$. Moreover, since 1) $\det(\mathbf S)(\mathbf x_{\text{sub}})$ is a continuous function of $\mathbf x_{\text{sub}}$, so if $\det(\mathbf S)({\mathbf x}_{\text{sub}}) \neq 0$ at point $\bar{\mathbf x}_{\text{sub}}$, it must be so in a neighbour of that point, and 2) the order of $i,j,k,l$ is arbitrary, so there are infinite number of submatrix $\mathbf S \in \mathcal O_{\text{sub}}$ whose determinant is not constant zero and for each the solution of $\det(\mathbf S) = 0$ lies on a {\it thin} subset of $\mathbb{R}^{n_{ \tau} +n_{ a} + 6}$. As a result, the states $\mathbf x_{\text{sub}}$ satisfying $\mathcal O_{\text{sub}}$ being rank-deficient is the joint of the infinite number of {\it thin} subset and the resultant solution space is even thinner. Such a thin subset is generally not satisfied when the system initial states $\boldsymbol{\omega}^{(k)}$ and $\mathbf a^{(k)}$ have sufficient excitation.
\end{proof}

\bibliography{EKF_V4} 
	
\end{document}